\documentclass[11pt,letterpaper]{article}

\usepackage{amsmath,amssymb,amsfonts,amsthm,bbm}
\usepackage{mathtools}   
\usepackage{stmaryrd} 

\usepackage{enumitem}   

\usepackage{authblk} 

\usepackage[titletoc,title]{appendix}  

\usepackage[normalem]{ulem}

\usepackage[T1]{fontenc}

\usepackage[table,xcdraw,dvipsnames]{xcolor}   
\usepackage{hyperref}
\newcommand\myshade{85}
\colorlet{mylinkcolor}{YellowOrange}
\colorlet{mycitecolor}{Aquamarine}
\colorlet{myurlcolor}{violet}

\hypersetup{
  linkcolor  = mylinkcolor!\myshade!black,
  citecolor  = mycitecolor!\myshade!black,
  urlcolor   = myurlcolor!\myshade!black,
  colorlinks = true,
}

\usepackage{graphicx}
\usepackage{caption}
\usepackage{subcaption}
\usepackage{float}

\usepackage{booktabs}
\usepackage{makecell}
\usepackage{array}
\usepackage{multirow}

\usepackage[linesnumbered,ruled,vlined]{algorithm2e}

\SetCommentSty{mycommfont}

\SetKwInput{KwInput}{Input}                
\SetKwInput{KwOutput}{Output}              

\usepackage{listings}             
\usepackage{stmaryrd} 

\renewcommand{\hat}{\widehat}
\renewcommand{\tilde}{\widetilde}
\renewcommand{\bar}{\widebar}

\newcommand{\bfm}[1]{\ensuremath{\boldsymbol{#1}}} 

\def\bbone{\mathbbm{1}} 

\def\ba{\bfm a}   \def\bA{\bfm A}  
\def\bb{\bfm b}   \def\bB{\bfm B}  
     
   \def\bD{\bfm D}  
\def\be{\bfm e}   \def\bE{\bfm E}  \def\EE{\mathbb{E}}
\def\bff{\bfm f}    
\def\bg{\bfm g}     
   \def\bH{\bfm H}  
\def\bi{\bfm i}   \def\bI{\bfm I}  
\def\bj{\bfm j}   \def\bJ{\bfm J}

   \def\bM{\bfm M}  
     
     \def\OO{\mathbb{O}}
     \def\PP{\mathbb{P}}
   \def\bQ{\bfm Q}  \def\QQ{\mathbb{Q}}
     \def\RR{\mathbb{R}}

\def\bu{\bfm u}   \def\bU{\bfm U}  
\def\bv{\bfm v}   \def\bV{\bfm V}  
     
   \def\bX{\bfm X}  
     
\def\bz{\bfm z}   \def\bZ{\bfm Z}

\def\calA{{\cal  A}} 
\def\calB{{\cal  B}} \def\cB{{\cal  B}}
 
\def\calD{{\cal  D}} \def\cD{{\cal  D}}
\def\calE{{\cal  E}} \def\cE{{\cal  E}}
\def\calF{{\cal  F}} \def\cF{{\cal  F}}
 
\def\calH{{\cal  H}} \def\cH{{\cal  H}}
\def\calI{{\cal  I}} 
 
 \def\cK{{\cal  K}}
\def\calL{{\cal  L}} \def\cL{{\cal  L}}
\def\calM{{\cal  M}} \def\cM{{\cal  M}}
 \def\cN{{\cal  N}}

\def\calR{{\cal  R}} \def\cR{{\cal  R}}
 \def\cS{{\cal  S}}
 \def\cT{{\cal  T}}
 
 \def\cV{{\cal  V}}
 
\def\calX{{\cal  X}} \def\cX{{\cal  X}}
\def\calY{{\cal  Y}} \def\cY{{\cal  Y}}
\def\calZ{{\cal  Z}} \def\cZ{{\cal  Z}}

\newcommand{\bfsym}[1]{\ensuremath{\boldsymbol{#1}}}

            \def\bDelta {\bfsym {\Delta}}

 \def\btheta{\bfsym {\theta}}

              \def\bSigma{\bfsym \Sigma}

\providecommand{\abs}[1]{\left\lvert#1\right\rvert}
\providecommand{\norm}[1]{\left\lVert#1\right\rVert}

\providecommand{\bbrackets}[1]{\llbracket #1 \rrbracket}

\providecommand{\defeq}{:=}

\usepackage{mathtools}
\DeclarePairedDelimiterX{\infdivx}[2]{(}{)}{%
  #1 \; \delimsize\| \; #2%
}

\DeclareMathOperator{\diag}{diag}
\newcommand{\E}[1]{{\mathbb{E}} \left[ #1 \right]}

\DeclareMathOperator{\Var}{Var}

\DeclareMathOperator{\vect}{vec}

\newtheorem{definition}{Definition}[section]
\newtheorem{assumption}[definition]{Assumption}
\newtheorem{lemma}[definition]{Lemma}
\newtheorem{proposition}[definition]{Proposition}
\newtheorem{theorem}[definition]{Theorem}

\theoremstyle{definition}
\newtheorem{remark}{Remark}

\definecolor{royalpurple}{rgb}{0.47, 0.32, 0.66}
\definecolor{greenfresh}{HTML}{00897B}
\definecolor{bluefresh}{HTML}{1E88E5}
\definecolor{redfresh}{HTML}{E53935}

\definecolor{royalpurple}{rgb}{0.47, 0.32, 0.66}

\def\beq{\begin{equation}}
\def\eeq{\end{equation}}

\def\bet{\begin{theorem}}
\def\eet{\end{theorem}}

\def\bel{\begin{lemma}}
\def\eel{\end{lemma}}

\def\tr{\mbox{tr}}

\def\eps{\varepsilon}
\def\lam {\lambda}

\def\cond{\;|\;}

\usepackage{setspace}
\usepackage{etoolbox}
\BeforeBeginEnvironment{equation*}{\begin{singlespace}\vspace*{-\baselineskip}}
	\AfterEndEnvironment{equation*}{\end{singlespace}\noindent\ignorespaces}
\BeforeBeginEnvironment{equation}{\begin{singlespace}\vspace*{-\baselineskip}}
	\AfterEndEnvironment{equation}{\end{singlespace}\noindent\ignorespaces}
\BeforeBeginEnvironment{align}{\begin{singlespace}\vspace*{-\baselineskip}}
	\AfterEndEnvironment{align}{\end{singlespace}\noindent\ignorespaces}
\BeforeBeginEnvironment{align*}{\begin{singlespace}\vspace*{-\baselineskip}}
	\AfterEndEnvironment{align*}{\end{singlespace}\noindent\ignorespaces}
\BeforeBeginEnvironment{eqnarray}{\begin{singlespace}\vspace*{-\baselineskip}}
	\AfterEndEnvironment{eqnarray}{\end{singlespace}\noindent\ignorespaces}
\BeforeBeginEnvironment{eqnarray*}{\begin{singlespace}\vspace*{-\baselineskip}}
	\AfterEndEnvironment{eqnarray*}{\end{singlespace}\noindent\ignorespaces}

\usepackage[framemethod=TikZ]{mdframed} 

\usepackage{fancybox}

\usepackage[top=1in, bottom=1in, left=1in, right=1in]{geometry}

\usepackage{setspace}
\usepackage[authoryear]{natbib}  

\newcommand{\mybib}{main,nn-lda}

\usepackage{graphicx}

\usepackage{listings}
\usepackage{comment}

\DeclareMathOperator*{\mat1}{\text{mat}_1}
\DeclareMathOperator*{\vec1}{\text{vec}}
\DeclareMathOperator*{\matk}{{\text{mat}}_m}
\DeclareMathOperator*{\mats}{{\text{mat}}_S}

\def\ideal{(\text{\footnotesize ideal})}

\renewcommand{\hat}{\widehat}
\renewcommand{\tilde}{\widetilde}
\newcommand{\widebar}{\overline}

\def\spacingset#1{\renewcommand{\baselinestretch}%
{#1}\small\normalsize} \spacingset{1}

\def\TITLE{High-Dimensional Tensor Discriminant Analysis: Low-Rank Discriminant Structure, Representation Synergy, and Theoretical Guarantees}

\usepackage{titlesec}

\titleformat{\section}
  {\normalfont\Large\bfseries}{\thesection}{1em}{}
\titlespacing*{\section}
  {0pt}{0.05\baselineskip}{0.05\baselineskip}

\newcommand{\blind}{0}
\begin{document}
\title{\bf \TITLE}

\if0\blind
{
  \author{
    Elynn Chen$^\dag$ \hspace{8ex}
    Yuefeng Han$^\sharp$ \hspace{8ex}
    Jiayu Li$^\flat$ \hspace{8ex} \\ \normalsize
    \medskip
    $^{\dag,\flat}$New York University \hspace{8ex}
    $^{\sharp}$ University of Notre Dame
    }
    \date{}
    \maketitle
} \fi

\if1\blind
{
  \author{}
  \date{}
  \maketitle
  \vspace{8ex}
} \fi

\bigskip
\begin{abstract}
\spacingset{1.38}
High-dimensional tensor-valued predictors arise in modern applications, increasingly as learned representations from neural networks. Existing tensor classification methods rely on sparsity or Tucker structures and often lack theoretical guarantees. Motivated by empirical evidence that discriminative signals concentrate along a few multilinear components, we introduce CP low-rank structure for the discriminant tensor, a modeling perspective not previously explored. Under a Tensor Gaussian Mixture Model, we propose high-dimensional CP low-rank Tensor Discriminant Analysis (CP-TDA) with Randomized Composite PCA (\textsc{rc-PCA}) initialization, that is essential for handling dependent and anisotropic noise under weaker signal strength and incoherence conditions, followed by iterative refinement algorithm. We establish global convergence and minimax-optimal misclassification rates.

To handle tensor data deviating from tensor normality, we develop the first semiparametric tensor discriminant model, in which learned tensor representations are mapped via deep generative models into a latent space tailored for CP-TDA. Misclassification risk decomposes into representation, approximation, and estimation errors. Numerical studies and real data analysis on graph classification demonstrate substantial gains over existing tensor classifiers and state-of-the-art graph neural networks, particularly in high-dimensional, small-sample regimes.
\end{abstract}

\noindent%
\vfill


\newpage
\spacingset{1.9} 
\addtolength{\textheight}{.1in}%


\section{Introduction}  \label{sec:intro}

High-dimensional tensor-valued data pervade modern scientific applications, from neuroimaging \citep{zhou2013tensor} and economics \citep{chen2021statistical,chen2020constrained,mo2025act} to recommendation systems \citep{bi2018multilayer} and climate analysis \citep{chen2024semi}. Beyond naturally occurring tensors, tensor predictors increasingly arise as learned representations from neural networks or outputs of multimodal integration pipelines \citep{hao2013linear,huang2025seeing}. These settings share two key challenges: (i) tensor dimensionality grows multiplicatively across modes, and (ii) discriminant structure governing class differences is often concentrated in a low-dimensional multilinear subspace \citep{wen2024tensorview,wu2025tensor}. Existing classification approaches, including sparse tensor linear discriminant analysis (LDA) \citep{pan2019covariate}, tensor regressions \citep{zhou2013tensor,XiaZhangZhou2022}, and deep neural networks \citep{Chien2018,jahromi2024variational}, address parts of this problem, but do not explicitly impose or exploit low-rank structure in the \emph{discriminant tensor} itself, nor do they provide comprehensive high-dimensional theoretical guarantees for misclassification risk.

This paper develops a new \emph{tensor discriminant analysis} (TDA) framework addressing both issues. Motivated by empirical evidence that discriminant signals in tensor predictors are often well-approximated by a few multilinear directions, especially for learned tensor features \citep{wen2024tensorview,wu2025tensor}, we introduce \emph{CP low-rank discriminant tensors}. To the best of our knowledge, the low-rank structure of the discriminant tensor has not been systematically studied, despite the prevalence of low multilinear rank in tensor representations. 
Under a Tensor Gaussian Mixture Model (TGMM) with common mode-wise covariances, we develop {\it the first tensor LDA with a CP low-rank discriminant tensor} (CP-TDA), which preserves multiway structure, reduces effective dimensionality, and yields interpretable mode-specific discriminant directions. 
The discriminant tensor is estimated via a new Randomized Composite PCA (\textsc{rc-PCA}) initialization that lies within the basin of attraction of the global optimum, followed by iterative refinement.

Modern tensor data, however, especially those arising from deep learning representations or multimodal integration, rarely follow tensor normal distributions. To address this broader regime, we propose Semiparametric Tensor Discriminant Networks (STDN), a semiparametric extension combining tensor discriminant analysis with deep learning. Unlike existing semiparametric LDA based on restrictive elementwise monotone transformations, we introduce a tensor Gaussianizing flow that transforms tensor representations, from any neural network encoder, into a latent space that approximately satisfies a TGMM. In this space, CP-TDA is the Bayes-optimal classifier under the surrogate model. This offers a principled alternative to dense neural network classifier heads: it preserves the interpretability of discriminant analysis while leveraging the representational power of deep neural networks. To the best of our knowledge, this is {\it the first use of deep learning to construct semiparametric tensor discriminant analysis with flexible nonlinear transformations}.

Beyond methodological contributions, we also provide new theoretical guarantees for CP-TDA and its semiparametric extension. For CP-TDA, we develop perturbation analysis for \textsc{rc-PCA} under dependent, heteroskedastic noise induced by sample discriminant tensors, relaxing the stringent eigen-gap and incoherence conditions used in prior CP decomposition work \citep{anandkumar2014guaranteed,sun2017provable,han2023cp}. This makes \textsc{rc-PCA} valuable beyond tensor LDA, including CP decomposition and regression where existing initializations fail. We establish non-asymptotic global convergence rates for the high-dimensional CP low-rank discriminant tensor and derive the first minimax-optimal misclassification rates for CP-TDA. For the semiparametric framework, we obtain excess misclassification risk bounds that separate three sources of error: CP low-rank approximation error, statistical error under CP-TDA, and Gaussianizing flow representation error, thereby providing guarantees absent in current deep-learning-based tensor classifiers.

Finally, we illustrate the practical benefits of our framework on two bioinformatics graph datasets, using deep learning encoders to produce tensor representations. The proposed semiparametric tensor LDA consistently outperforms both existing tensor classifiers on tensor features and state-of-the-art graph neural networks. These results underscore the central message of this work: carefully designed statistical structure, here, CP low-rank discriminant modeling under TGMM, can be effectively combined with modern deep learning to achieve both interpretability and strong empirical performance, not as a competitor to deep learning but as a statistically principled complement.

\subsection{Related Work and Our Distinction}

\textbf{Tensor decomposition and regression.}
Tensor methods have become essential for multiway data analysis, with advances in tensor decomposition \citep{sun2017provable,zhang2018tensor,chen2024distributed}, tensor regression \citep{zhou2013tensor,li2017parsimonious,XiaZhangZhou2022,chen2024factor}, and tensor clustering \citep{SunLi2019,HanLuoWangZhang2022,LuoZhang2022,hu2023multiway}. Most work exploits low-rank structure, with Tucker-type models widely used for parameter reduction \citep{XiaZhangZhou2022,HanWillettZhang2022}. In contrast, CP-based low-rank modeling is far less developed: \cite{zhou2013tensor} employ a block relaxation algorithm for a highly non-convex CP tensor generalized linear model without principled initialization, leading to unreliable convergence in high dimensions (see Appendix \ref{sec:zhoulimit}).

\noindent\textbf{High-dimensional vector discriminant analysis.}
Vector-based discriminant analysis extends Fisher's LDA to high dimensions through sparse and direct LDA methods \citep{mai2012direct,cai2019high}, feature screening for multiclass LDA \citep{Pan2016}, and high-dimensional quadratic discriminant analysis \citep{cai2021convex}. However, these approaches operate on vectorized data, losing multiway structure and mode-specific interactions inherent in tensor data.

\noindent\textbf{Tensor classification.}
Recent tensor classifiers include support tensor machines \citep{hao2013linear,guoxian2016}, tensor logistic regression \citep{Liuregression2020,SONG2023430}, and tensor neural networks \citep{Chien2018,jahromi2024variational,wen2024tensorview}. Tensor multilinear discriminant analysis \citep{lai2013sparse,Franck2023} extracts multilinear features but relies on subsequent vector classifiers. While empirically effective, these methods typically lack theoretical guarantees and ignore mode-wise covariance structure. On the theoretical side, \cite{HanWillettZhang2022} analyze tensor logistic regression under i.i.d. entries for each tensor predictor, while \cite{pan2019covariate} study tensor LDA with sparsity in a TGMM, and \cite{li2022tucker} and \cite{wang2024parsimonious} develop Tucker low-rank and envelope-based discriminant structures under group sparsity. Against this backdrop, we develop CP low-rank tensor LDA under a TGMM with common mode-wise covariances and extend it to semiparametric tensor LDA with a learned Gaussianizing flow, providing non-asymptotic misclassification bounds that account for both low-rank structure and latent distributional mismatch.

\subsection{Notation and Paper Organization}\label{sec:notation}

Let $[M]=\{1,\ldots,M\}$ and $\bbone(\cdot)$ denote the indicator function. For matrix $\bA = (a_{ij})\in \RR^{m\times n}$, the Frobenius and spectral norms are $\|\bA\|_{\rm F} = (\sum_{ij} a_{ij}^2)^{1/2}$ and $\|\bA\|_{2}$, respectively. For squared matrix $\bA = (a_{ij})\in \RR^{m\times m}$, denote its smallest and largest eigenvalues as $\lambda_{\min}(\bA)$ and $\lambda_{\max}(\bA)$, respectively. For sequences $\{a_n\}$ and $\{b_n\}$, we write $a_n=O(b_n)$ if $|a_n|\leq C |b_n|$, $a_n\asymp b_n$ if $1/C \leq a_n/b_n\leq C$, and $a_n=o(b_n)$ if $\lim_{n\to\infty} a_n/b_n =0$, for some constant $C$ and sufficiently large $n$. Similarly, $a_n\lesssim b_n$ (resp. $a_n\gtrsim b_n$) means $a_n\le Cb_n$ (resp. $a_n\ge Cb_n$). 

An $M$-th order tensor is a multi-dimensional array $\calX \in \RR^{d_1 \times \cdots \times d_M}$, with index $\calI = (i_1, \dots, i_M)$. The inner product $\langle \mathcal{X},\mathcal{Y} \rangle = \sum_{\calI} \calX_{\calI} \calY_{\calI}$ induces the Frobenius norm $\| \mathcal{X} \|_{\rm F} = \langle \mathcal{X},\mathcal{X} \rangle^{1/2}$. For $\mathnormal{A} \in \mathbb{R}^{\tilde{d_m} \times d_m}$, the mode-$m$ product $\mathcal{X} \times_m \mathnormal{A}$ yields an $M$-th order tensor of size $d_1 \times \cdots \times d_{m-1} \times \tilde{d_m} \times d_{m+1} \times \cdots \times d_M$, where $(\mathcal{X} \times_m A)_{i_1, \ldots, i_{m-1}, j, i_{m+1}, \ldots, i_M} = \sum_{i_m=1}^{d_m} \mathcal{X}_{i_1, \ldots, i_M} A_{j, i_m}$. The mode-$m$ matricization $\text{mat}_m(\mathcal{X})$ reshapes $\cX$ into a $d_m \times \prod_{k\neq m} d_k$ matrix, where element $\mathcal{X}_{i_1\dots i_M}$ maps to position $(i_m, j)$, with $j=1+\sum\nolimits_{k \neq m} (i_k -1) \prod_{l<k,l \neq m}d_l$. For $S \subseteq \{1, \dots, M\}$, the multi-mode matricization ${\rm mat}_S(\calX)$ yields a $\prod_{m \in S} d_m$-by-$\prod_{m \notin S} d_m$ matrix with row index
$i = 1 + \sum_{m \in S} (i_m - 1) \prod_{\ell \in S, \ell < m} d_\ell$ and column index
$j = 1 + \sum_{m \notin S} (i_m - 1) \prod_{\ell \notin S, \ell < m} d_\ell$. Vectorization is denoted $\text{vec}(\mathcal{X}) \in \mathbb{R}^d$ with $d=\prod_{m=1}^M d_m$. 

The Tensor Normal distribution generalizes matrix normal to higher-orders \citep{hoff2011TN}. For \(\calZ \in \RR^{d_1 \times \cdots \times d_M}\) i.i.d. \(N(0,1)\) elements, mean tensor \(\mathcal{M} \in \mathbb{R}^{d_1 \times \cdots \times d_M}\) and mode covariance matrices \(\Sigma_m \in \mathbb{R}^{d_m \times d_m}\) for \(m\in [M]\), define \(\calX=\calM + \calZ \times_{m=1}^M \Sigma_m^{1/2}\). 
Then, \(\calX\) follows a tensor normal distribution \(\calX \sim \cT\cN(\cM; \bSigma)\) with \(\bSigma := [\Sigma_m]_{m=1}^M\), which is equivalent to \(\vect(\calX) \sim \cN(\vect(\cM); \Sigma_M\otimes\cdots\otimes\Sigma_1)\), where \(\otimes\) denotes the Kronecker product. 

\noindent\textbf{Organization.}
The rest of the paper is organized as follows. Section~\ref{sec:model} introduces tensor discriminant analysis with a CP low-rank discriminant tensor, including \textsc{rc-PCA} initialization and an iterative refinement procedure. Section~\ref{sec:theorems} establishes theoretical guarantees with estimation error bounds and minimax-optimal misclassification risk rates. Section~\ref{sec:Tensor LDA-TNN} proposes semiparametric CP-TDA with a learned Gaussianizing flow based on deep neural networks. Section~\ref{sec:simu} reports simulation studies, and Section~\ref{sec:appl} illustrates empirical performance on real world data. Section~\ref{sec:conclude} concludes. All proofs, lemmas and additional experiments are deferred to the supplementary material.

\section{CP Tensor Discriminant Analysis: The Parametric Model}
\label{sec:model}

In this section, we develop the parametric foundation of our approach by studying tensor discriminant analysis under the Tensor Gaussian Mixture Model. We introduce a CP low-rank structure for the discriminant tensor and formulate the resulting CP-structured Tensor LDA (CP-TDA) estimator together with its high-dimensional estimation procedure. This parametric framework serves as the theoretical backbone of our methodology and provides key structural insights (low-rank discriminants, initialization requirements, and contraction behavior) that enable the semiparametric representation-learning extension with neural network encoders and flows in Section \ref{sec:Tensor LDA-TNN}.

\subsection{Tensor Gaussian Mixture Model and Oracle Tensor LDA}
\label{sec:oracle lda}
We consider a high-dimensional tensor predictor $\calX\in\RR^{d_1 \times d_2 \cdots \times d_M}$ and class label $Y \in [K]$ under the Tensor Gaussian Mixture Model (TGMM) \citep{mai2021doubly},  
\begin{equation} \label{eqn:tgmm2}
(\calX \cond Y = k) \sim \cT\cN(\cM_k; \bSigma_k),
\quad \pi_k \defeq \PP(Y = k), 
\quad \sum\nolimits_{k=1}^K \pi_k = 1, 
\quad \text{for any } k \in [K],
\end{equation}
where $\calM_k$ is the class mean tensor, $\bSigma_k= [\Sigma_m]_{m=1}^M$ represents within-class covariance matrices, and $0<\pi_k<1$ is the prior probability for class $k$. 
We focus on binary classification ($K = 2$) with common covariance $\bSigma_1=\bSigma_2=\bSigma$. 
Given independent samples from the two tensor-normal distributions, the goal is to classify future observations.

When parameters $\btheta := (\pi_1, \pi_2, \cM_1, \cM_2, \bSigma)$ are known, the Bayes-optimal tensor LDA rule is:
\begin{equation}
\label{eqn:lda-rule}
\Upsilon(\cX) 
= \bbone\big\{ \langle \cX - \cM, \; \calB \rangle +\log (\pi_2/\pi_1 ) \ge 0 \big\},
\end{equation}
where $\bbone(\cdot)$ is the indicator function, $\cM = (\cM_1 + \cM_2)/2$, $\cD=(\cM_2-\cM_1)$, and the discriminant tensor $\calB = \cD \times_{m=1}^M \Sigma_m^{-1}$. 
This generalizes classical vector LDA to the tensor setting and is optimal under TGMM \citep{mai2021doubly}. The optimal misclassification error is 
$R_{\rm opt}=\pi_1\phi(\Delta^{-1}\log(\pi_2/\pi_1)-\Delta/2)+\pi_2(1-\phi(\Delta^{-1}\log(\pi_2/\pi_1)+\Delta/2)),$
where $\phi$ is the standard normal CDF and $\Delta=\sqrt{\langle \calB, \; \cD \rangle}$ is the signal-to-noise ratio (SNR). 

Classical vector LDA can be recovered by vectorizing $\calX$ but incurs severe computational and statistical burdens. 
Existing high-dimensional LDA methods impose sparsity on a vectorized discriminant, ignoring multiway structure and becoming infeasible even for moderate tensor dimensions \citep{mai2012direct,cai2019high}; for instance, the adaptive procedure in \cite{cai2019high} involves linear programs with 27,000 constraints for a $30 \times 30 \times 30$ tensor.
Existing tensor discriminant analysis methods impose sparsity on the discriminant tensor \citep{pan2019covariate} or Tucker low-rank tensor envelope \citep{wang2024parsimonious}, but none model a CP low-rank discriminant tensor.

Empirical evidence shows that discriminative signals in tensor data often concentrate along a few multilinear directions \cite[see][and our empirical results]{wen2024tensorview}. When class separation is driven by specific discriminative directions in each mode, their multilinear interaction naturally forms a rank-1 outer product structure: ($w \cdot \ba_1 \circ \ba_2 \circ \cdots \circ \ba_M$), where $\ba_m$ identifies the key direction in mode $m$ and $\circ$ denotes the tensor outer product. When multiple discriminative patterns exist, this motivates a CP low-rank structure with rank $R\ll d_m$ for the discriminant tensor:
\begin{equation}\label{eqn:lda-cp}
    \calB =\sum\nolimits_{r=1}^R w_r \cdot (\ba_{r1}\circ\ba_{r2}\circ\cdots\circ \ba_{rM}), 
\end{equation}
where each rank-1 component captures one multilinear pattern, $w_r>0$ represents its signal strength, and $\ba_{rm}$ are unit vectors in $\RR^{d_m}$. 
Beyond empirical motivation, low-rank structure is essential for consistent high-dimensional LDA. As shown by \cite{cai2021convex}, no data-driven method can achieve optimal misclassification error $R_{\rm opt}$ in high dimensions ($d\gtrsim n$) with unknown means, even given identity covariances. Under the CP low-rank assumption, however, our CP-TDA rule with carefully designed estimation of the discriminant tensor achieves minimax optimal misclassification rates, as established in Theorems \ref{thm:class-upp-bound} and \ref{thm:class-lower-bound}.


\begin{remark}
While Tucker decomposition $\calB =\cF\times_1 \bU_1\times_2\cdots\times_M \bU_M$ with core tensor $\calF \in \RR^{r_1 \times r_2 \times \cdots \times r_M}$ and orthogonal loading matrices $\bU_i \in \RR^{d_i \times r_i}$ offers an alternative low-rank structure, CP provides distinct advantages for discriminant analysis. First, when class separation is driven by specific discriminative directions in each mode, these directions naturally interact through outer products to form rank-1 CP components, making CP the natural decomposition choice. Second, CP offers methodological advantages \citep{han2023cp,erichson2020randomized}: it is uniquely identified up to sign changes and permutations, avoiding Tucker's ambiguity from invertible transformations of the core tensor and loading matrices. Although CP can be viewed as Tucker with a superdiagonal core, we allow non-orthogonal CP bases $\{ \ba_{rm}, r\in [R] \}$, offering greater flexibility than Tucker's typical orthonormal representation. Moreover, CP's signal strength sequence is more parsimonious than Tucker's core tensor, enhancing both interpretability and estimation efficiency. 
\end{remark}

\begin{remark}
Our framework can readily incorporate covariates in the same way as \cite{pan2019covariate}: first regress out covariate effects from the tensor observations, then apply CP-TDA to the residual tensors. The theoretical guarantees in our paper extend naturally to this preprocessing step.    
\end{remark}

\subsection{Estimation of CP Low-Rank Discriminant Tensor} \label{sec:method}


Given the oracle Tensor LDA rule, tensor classification reduces to estimating the discriminant tensor $\cB$ from data and plugging the estimate into the rule. The natural sample analogue, obtained from sample means $\widebar\calX^{(k)}$ and estimated mode-wise covariances $\widehat\Sigma_m$, is
\begin{equation}\label{eqn:lda-discrim-tensor}
\widehat\calB = (\widebar\calX^{(2)}-\widebar\calX^{(1)}) \times_{m=1}^{M} \widehat\Sigma_m^{-1},
\end{equation}
where observations $\calX_1^{(k)}, \ldots, \calX_{n_k}^{(k)}$ are i.i.d. from class $k=1,2$, $\widehat\Sigma_m = (n_1+n_2)^{-1} d_{-m}^{-1}\sum_{k=1}^2\sum_{i=1}^{n_k}$ $ \matk(\calX_{i}^{(k)}-\widebar\calX^{(k)}){\matk}^{\top}(\calX_{i}^{(k)}-\widebar\calX^{(k)})$ with $d_{-m}=d/d_m$ for $m\in[M]$, and for identifiability we rescale $\widehat\Sigma_M $ by a factor $ \hat C_{\sigma}^{-1}$ with $\hat C_{\sigma} = \prod_{m=1}^M \hat\Sigma_{m,11}/\hat{\Var}(\cX_{1,1\cdots1})$. 
With $n=\min(n_1, n_2)$, following \cite{drton2021existence}, $\hat\Sigma_m$ is positive definite and \eqref{eqn:lda-discrim-tensor} is well-defined with probability 1 if $n_{}> d_m/d_{-m}$. This condition is mild for comparable mode dimensions. 
The sample discriminant tensor $\hat\calB$ serves as a perturbed version of the true $\calB$ that can be substantially improved under CP low-rankness \eqref{eqn:lda-cp}.

Unlike standard tensor decomposition settings, the perturbation error $\widehat{\calB}-\calB$ exhibits complex dependence and anisotropy, inherited from cross-sectional dependence among tensor normal distributions and mode-wise covariance inversions. This leads to noisier perturbations and additional analytical difficulty. Moreover, as discussed in the previous subsection, the sample discriminant tensor does not leverage low-rankness and thus cannot achieve optimal misclassification risk.


The next subsections introduce our solution to denoise $\hat\calB$: a new Randomized Composite PCA (\textsc{rc-PCA}) initialization superior to existing CP initialization, followed by an iterative refinement procedure. Together they yield a reliable estimator of the high-dimensional CP low-rank discriminant tensor $\calB$.

\subsubsection{\textsc{rc-PCA}: CP Initialization under Weaker Incoherence}

Estimating a CP low-rank discriminant tensor requires an iterative refinement procedure beginning with the sample estimator $\widehat{\mathcal{B}}$ in \eqref{eqn:lda-discrim-tensor}; details appear in Section \ref{sec:iterative-refinement}.
Because denoising $\hat\cB$ is a highly non-convex problem with exponentially many stationary points \citep{arous2019landscape}, initialization becomes the central bottleneck: without a suitable starting point within a narrow basin of attraction, iterative refinement fails to reach the true discriminant tensor at the optimal statistical rate. The difficulty is heightened here because $\widehat{\mathcal{B}}-\mathcal{B}$ exhibits complex dependent noise.

To address this, we introduce Randomized Composite PCA (\textsc{rc-PCA}), a new initialization method (Algorithm \ref{alg:initialize-cp}). By combining randomized projections with composite PCA, \textsc{rc-PCA} relaxes the signal-to-noise ratio and incoherence requirements among CP bases, yielding initial CP basis estimates lying within the contraction region of the refinement algorithm. With polynomially many random attempts, \textsc{rc-PCA} reaches the attraction ball of the global optimum with high probability (Theorem \ref{thm:cp-initilization}).

Our strategy handles two scenarios: significant gaps between signal strengths $\{w_r, r \in [R]\}$, and comparable signal strengths. For the multi-mode matricization,
\[ {\rm mat}_S(\calB) = \sum\limits_{r=1}^R w_r \vect(\circ_{m\in S} \ba_{rm}) \vect(\circ_{m\in S^c} \ba_{rm})^\top, \]
where $S, S^c$ partition $[M]$. For orthogonal tensor components, signal strengths correspond directly to singular values. 
For non-orthogonal cases, Proposition \ref{prop:singular_value_gap} shows that under mild non-orthogonality (incoherence condition), the difference between signal strengths and singular values remains small, allowing us to use singular value gaps as proxies for signal strength differences.

\begin{proposition}
\label{prop:singular_value_gap}
Let $\bA_S= (\ba_{1,S},\ldots, \ba_{R,S})\in \RR^{d_S\times R}$ and $\bA_{S^c}= (\ba_{1,S^c},\ldots, \ba_{R,S^c})\in \RR^{d_{S^c}\times R}$, where $\ba_{r,S} = \vect(\circ_{m\in S} \ba_{rm})$ and $\ba_{r,S^c} = \vect(\circ_{m\in S^c} \ba_{rm})$. Then
$ \mathrm{mat}_S(\calB) = \bA_S W \bA_{S^c}^\top$ with \(W = \mathrm{diag}(w_1, w_2, \ldots, w_R)\). For $\delta = \| \bA_S^\top \bA_S - I_R \|_2 \vee \| \bA_{S^c}^\top \bA_{S^c} - I_R \|_2$ and $w_1 \geq \cdots \geq w_R$, the $r$-th largest singular values \(\lambda_r\) of \(\mathrm{mat}_S(\calB)\) and the signal strengths \(w_r\) satisfy: $|\lambda_r - w_r| \leq \sqrt{2}\delta w_1.$
\end{proposition}

Algorithm \ref{alg:initialize-cp} first applies Composite PCA (CPCA) for large-gap scenarios. When CPCA fails due to insufficient spectral gaps, it employs Procedure \ref{alg:initialize-random}, which (1) performs random projections to enlarge the gap of the leading singular value from the rest, enabling CPCA to extract candidate CP bases associated with the dominant component, and (2) clusters these candidates to obtain representative CP bases. As shown in Appendix \ref{append:proof:initia}, polynomially many random projections ensure coverage of all desirable initializations with high probability and provable bounds.

\begin{algorithm}[h!]
    \SetKwInOut{Input}{Input}
    \SetKwInOut{Output}{Output}
    \Input{Initial tensor $\hat\calB$, CP rank $R$, $S \subset [M]$, small constant $0<c_0<1$.}

    If $S=\emptyset$, pick $S \subset [M]$ to maximize $\min(d_S,d/d_S)$ with $d_S = \prod_{m\in S}d_m$ and $d = \prod_{m=1}^M d_m$.

    Unfold $\hat\calB$ to be a $d_S\times (d/d_S)$ matrix ${\rm mat}_S(\hat\calB)$.

    Compute  $\widehat\lambda_r,\widehat \bu_r, \widehat \bv_r$, $1\le r\le R$ as the top $R$ components in the SVD ${\rm mat}_S(\hat\calB) = \sum_{r}\widehat\lambda_r \widehat \bu_r \widehat \bv_r^\top$. Set $\widehat \lambda_0=\infty$ and $\widehat\lambda_{R+1}=0$.

\If{$\min\{|\widehat\lambda_r-\widehat\lambda_{r-1}|,|\widehat\lambda_r-\widehat\lambda_{r+1}| \} > c_0 \widehat\lambda_R $  }{
 Compute $\widehat \ba_{rm}^{\rm rcpca}$ as the top left singular vector of $\matk (\widehat \bu_r)\in\RR^{d_m\times (d_S/d_m)}$ for $m \in S$, or $\matk (\widehat \bv_r)\in\RR^{d_m\times (d_{S^c}/d_m)}$ for $m \in S^c$.
}
\Else{
Form disjoint index sets $I_1,...,I_N$ from all contiguous indices $1\le r\le R$ that do not satisfy the above criteria of the eigengap.

For each $I_j$, form $d_{S}\times (d/d_{S})$ matrix $\Xi_{j}=\sum_{\ell\in I_j} \widehat \lambda_{\ell} \widehat \bu_{\ell}\widehat \bv_{\ell}^\top$, reshape it into a tensor $\Xi_j\in \RR^{d_1\times\cdots\times d_M}$. Then run Procedure \ref{alg:initialize-random} on $\Xi_j$ to obtain $\widehat \ba_{rm}^{\rm rcpca}$ for all $r\in I_j, m\in [M]$.
}

    \Output{Warm initialization $\hat \ba_{rm}^{\rm rcpca}, 1\le r\le R, 1\le m\le M$.}
    \caption{Randomized Composite PCA (\textsc{rc-PCA})}
    \label{alg:initialize-cp}
\end{algorithm}
\vspace{-10pt}

\SetAlgorithmName{Procedure}{procedure}{List of Procedures}
\begin{algorithm}[h!]
\caption{Randomized Projection}\label{alg:initialize-random}
    \SetKwInOut{Input}{Input}
    \SetKwInOut{Output}{Output}
    \Input{Noisy tensor $\Xi\in\RR^{d_1\times\cdots\times d_M}$, rank $s$, $S_1 \subset [M]\backslash\{1\}$, number of random projections $L$, tuning parameter $\nu$.}

    If $S_1=\emptyset$, pick $S_1$ to maximize $\min(d_{S_1},d_{-1}/d_{S_1})$ with $d_{S_1} = \prod_{m\in S_1}d_m$ and $d_{-1} = \prod_{m=2}^M d_m$. Let $S_1\vee S_1^c = [M]\backslash\{1\}$.
    
   \For{$\ell = 1$ to $L$}{
        Randomly draw a standard Gaussian vector $\theta\sim\cN(0, I_{d_1})$. 
        
        Compute $\Xi\times_{1}\theta$, unfold it to be $d_{S_1}\times (d_{-1}/d_{S_1})$ matrix, and compute its leading singular value $\eta_\ell$ and singular vector $\widetilde \bu_{\ell}, \widetilde \bv_{\ell}$. 
        
        Compute $\widetilde \ba_{\ell m}$ as the top left singular vector of $\matk(\widetilde  \bu_{\ell})\in\RR^{d_m\times (d_{S_1}/d_m)}$ for $m \in S_1$, or $\matk (\widetilde \bv_{\ell})\in\RR^{d_m\times (d_{S_1^c}/d_m)}$ for $m \in S_1^c$. 
        
        Compute $\widetilde \ba_{\ell 1}=\Xi\times_{m=2}^M \widetilde  \ba_{\ell m}/\norm{\Xi\times_{m=2}^M \widetilde  \ba_{\ell m}}_2$.
    }
    \For{$r = 1$ to $s$}{
        Among $\ell \in [L]$, choose $\ell^*$ with tuple $(\widetilde \ba_{\ell^* m},1\le m\le M)$ that correspond to the largest $\|\Xi\times_{m=1}^{M} \widetilde  \ba_{\ell m}\|_2$. Set it to be $\widehat \ba_{rm}^{\rm rcpca}=\widetilde \ba_{\ell^* m}$.

        Remove all $\ell \in [L]$ with $\max_{1\le m\le M} |\widetilde  \ba_{\ell m}^\top \widehat \ba_{rm}^{\rm rcpca}|>\nu$.
    }

    \Output{Warm initialization $\hat \ba_{rm}^{\rm rcpca}, 1\le r\le s, 1\le m\le M$.}
\end{algorithm}
\SetAlgorithmName{Algorithm}{algorithm}{List of Algorithms}

\subsubsection{Iterative Refinement of the CP Low-Rank Discriminant Tensor} \label{sec:iterative-refinement}

Our refinement procedure employs iterative tensor projection (Algorithm \ref{alg:tensorlda-cp}). Each iteration performs simultaneous orthogonalized projections across tensor modes to refine the CP components $\ba_{rm},\; r \in [R]$, followed by signal strength estimation $w_r, \; r \in [R]$. This approach preserves the signal strength of $\calB$ while reducing the perturbation error $\hat\calB - \calB$, thereby improving the signal-to-noise ratio and estimation accuracy.

To illustrate, consider estimating mode-$m$ CP bases $\{\ba_{r m}\}_{r\in[R]}$ given true mode-$\ell$ CP bases $\{\ba_{r\ell}\}_{r\in[R]}$ for all $\ell\neq m$. Define $\bA_{\ell}=[\ba_{1\ell},\ldots,\ba_{R\ell}]$ and its right inverse $\bB_{\ell} = \bA_{\ell}(\bA_{\ell}^{\top} \bA_{\ell})^{-1} = [\bb_{1\ell},...,\bb_{R\ell}]$, where $\ba_{k\ell}^\top\bb_{r\ell }=\bbone\{k=r\}$. Projecting the noisy $\hat\calB$ with $\bb_{r \ell}$ on all modes $\ell \ne m$ isolates the signal component containing only $\ba_{rm}$:
\begin{equation} \label{eq:cp-ideal}
\bz_{rm} 
= \hat\calB \times_{\ell=1, \ell\ne m}^{M} \bb_{r\ell}^\top
= \underbrace{w_{r} \ba_{r m}}_\text{signal}
+ \underbrace{ (\hat\calB - \calB) \times_{\ell=1, \ell\ne m}^{M} \bb_{r\ell}^\top}_\text{noise}.
\end{equation}
The signal retains strength $w_r$ along $\ba_{rm}$, while noise is filtered through projections. When the true $\{\bB_{\ell},\; \ell \neq m\}$ are known, the noise term is small relative to the signal, enabling signal recovery with theoretical guarantees via perturbation theory \citep{wedin1972perturbation}. In practice, with unknown $\{\bB_{\ell},\; \ell \neq m\}$, we iteratively update CP bases $\{\ba_{rm},\; r \in [R], m \in [M]\}$ at iteration $t$:
\begin{align*}
\bz_{rm}^{(t)} &= \hat\calB \times_1 \hat\bb_{r1}^{(t)\top} \times_2 \cdots \times_{m-1} \hat\bb_{r,m-1}^{(t)\top} \times_{m+1} \hat\bb_{r, m+1}^{(t-1)\top} \times_{m+2} \cdots \times_M \hat\bb_{rM}^{(t-1)\top},\quad 
\hat \ba_{rm}^{(t)} = \bz_{rm}^{(t)}/\| \bz_{rm}^{(t)} \|_2.
\end{align*}
The projection error becomes 
$
\bz_{rm} - \bz_{rm}^{(t)} = w_r\ba_{rm} - \sum_{i=1}^R \widetilde w_{i,r} \ba_{im} + \be_{rm} - \tilde \be_{rm}, 
$
where $\widetilde w_{i,r} = w_i \prod_{\ell=1}^{m-1} [ \ba_{i,\ell}^{\top} \hat \bb_{r\ell}^{(t)} ] \prod_{\ell=m+1}^{M} [ \ba_{i,\ell}^{\top} \hat \bb_{r\ell}^{(t-1)} ]$, $\be_{rm} = (\hat\calB - \calB)\times_{\ell=1}^{M} \bb_{r,\ell}^\top$, and $\tilde \be_{rm} = (\hat\calB - \calB)\times_{\ell=1}^{m-1} \widehat \bb_{r,\ell}^{(t)\top} \times_{\ell=m+1}^{M} \widehat \bb_{r,\ell}^{(t-1)\top}$.
The multiplicative measure $\widetilde w_{i,r}$ decays rapidly as $\ba_{i,\ell}^{\top} \hat \bb_{r\ell}^{(t)}$ approaches zero for $i \neq r$ with increasing iterations, since $\widetilde w_{i,r}$ contains $M-1$ multiplicative terms. The remaining projected noise is controlled by the signal-to-noise ratio, allowing us to establish upper bounds on estimation error and ensure accuracy of the recovered CP bases.

\begin{algorithm}[htpb!]
    \SetKwInOut{Input}{Input}
    \SetKwInOut{Output}{Output}
    \Input{Initial tensor $\hat\calB$, 
    CP rank $R$, warm-start $\hat \ba_{rm}^{(0)}, 1\le r\le R, 1\le m\le M$, tolerance parameter $\epsilon>0$, maximum number of iterations $T$}
    
    Let $t=0$, and compute $(\hat \bb_{1m}^{(0)},\ldots, \hat \bb_{rm}^{(0)})$ as the right inverse of $(\hat \ba_{1m}^{(0)}, \ldots, \hat \ba_{rm}^{(0)})^\top$.
    
    \Repeat{$t = T$ {\bf or} $\max_{r,m}\| \hat \ba_{rm}^{(t)}  \hat \ba_{rm}^{(t)\top} - \hat \ba_{rm}^{(t-1)} \hat \ba_{rm}^{(t-1)\top}\|_{2}\le \epsilon$}{
        Set $t=t+1$. 
        
        \For{$m = 1$ to $M$}{
            \For{$r = 1$ to $R$}{
                Compute $\bz_{rm}^{(t)} = \hat\calB \times_1 \hat\bb_{r1}^{(t)\top} \times_2 \cdots \times_{m-1} \hat\bb_{r,m-1}^{(t)\top} \times_{m+1} \hat\bb_{r, m+1}^{(t-1)\top} \times_{m+2} \cdots \times_M \hat\bb_{rM}^{(t-1)\top}$
                
                
                Compute $\hat \ba_{rm}^{(t)} = \bz_{rm}^{(t)}/\| \bz_{rm}^{(t)} \|_2$
            }
            Compute $(\hat \bb_{1m}^{(t)},\ldots, \hat \bb_{rm}^{(t)})$ as the right inverse of $(\hat \ba_{1m}^{(t)}, \ldots, \hat \ba_{rm}^{(t)})^\top$
            
            Set $(\hat \bb_{1m}^{(t+1)},..., \hat \bb_{rm}^{(t+1)})=(\hat \bb_{1m}^{(t)},..., \hat \bb_{rm}^{(t)})$ 
        }
        Compute $\hat w_{r}^{(t)} = \big|\hat\calB\times_{m=1}^M (\hat\bb_{rm}^{(t)})^\top\big|, 1\le r\le R$
        }

    \Output{$\hat  \ba_{rm}=\hat \ba_{rm}^{(t)}$, 
        $\hat w_{r} = \big|\hat\calB\times_{m=1}^M (\hat  \bb_{rm}^{(t)})^\top\big|$, 
        $\hat\calB^{\rm cp}=\sum_{r=1}^R \hat w_r \circ_{m=1}^M \hat\ba_{rm}$, $1\le r\le R,  m\le M$}
    \caption{Discriminant Tensor Iterative Projection for CP low-rank (DISTIP-CP) }
    \label{alg:tensorlda-cp}
\end{algorithm}

\begin{remark}[CP-TDA Rule]
Finally, given the estimated discriminant tensor $\hat\calB^{\rm cp}$, the high-dimensional CP-TDA rule is given by
\begin{equation}
\hat\Upsilon_{\rm cp}(\cX^*) = \bbone\left\{ \langle \cX^* - (\widebar\calX^{(1)}+\widebar\calX^{(2)})/2, \; \hat\calB^{\rm cp} \rangle + \log(\hat\pi_2/\hat\pi_1) \ge 0 \right\}.   \label{eqn:lda-rule-cp}
\end{equation}
This assigns $\cX^*$ to class $1$ when $\hat\Upsilon_{\rm cp}(\cX^*)=0$ or $2$ when $\hat\Upsilon_{\rm cp}(\cX^*)=1$.
\end{remark}

\section{Theoretical Analysis of CP-TDA}
\label{sec:theorems}


In this section, we establish the statistical properties of our algorithms, providing consistency guarantees and statistical error rates for the estimated discriminant tensor and misclassification error under regularity conditions. We define $d=\prod_{m=1}^M d_m$, $d_{-m}=d/d_m$, $d_{\min}=\min\{d_1,\cdots,d_M\}$. We measure the distance between estimated and true CP bases using $\|\hat\ba_{rm}{\hat\ba_{rm}}^{\top} - \ba_{rm}\ba_{rm}^{\top}\|_2=\sqrt{1-(\hat\ba_{rm}^\top\ba_{rm})^2}$. Assume $w_1 \ge w_2 \ge \cdots \ge w_R$. As $\| \ba_{im}\|_2^2=1$, we quantify the correlation among columns of $\bA_m$ as
\vspace{-2em}
\begin{align}\label{corr-k}
\delta_m = \| \bA_m^\top \bA_m - I_{R}\|_{2}, \quad \text{with} \quad \delta_{\max} =\max\{\delta_1,\cdots, \delta_M \}. 
\end{align}

\noindent
\textbf{Global convergence.} We first analyze the global convergence of Algorithm \ref{alg:tensorlda-cp}, given a suitable warm initialization.
Define the following quantities in terms of the initialization error $\psi_0$: 
\begin{align}
\alpha&=\sqrt{(1-\delta_{\max})(1-1/(4R))}-(R^{1/2}+1)\psi_0, \label{eq:alpha} \\  
\rho &= 2\alpha^{1-M}\sqrt{R-1}(w_1/w_R) \psi_{0}^{M-2}      \label{eq:rho} , \\
\psi^{\rm ideal}_r &= \frac{ \sqrt{ \sum_{k=1}^M d_k}}{\sqrt{n}  w_r}   +   \frac{w_1 }{ w_r} \max_{1\le k\le M} \sqrt{\frac{d_k}{n  d_{-k}}} , \quad \psi^{\rm ideal} =\psi^{\rm ideal}_R .
\end{align}
The following theorem provides the convergence rates for the CP low-rank discriminant tensor. 

\begin{theorem}\label{thm:cp-converge}
Suppose $C_0^{-1} \leq \lambda_{\min}(\otimes_{m=1}^M \Sigma_m) \leq \lambda_{\max}(\otimes_{m=1}^M \Sigma_m) \leq C_0$ for some constant $C_0 > 0$. Assume $n_{1} \asymp n_{2}$ and $n=\min\{n_{1},n_{2}\}$.
Let $\Omega_0 = \{\underset{r\in[R],m\in[M]}{\max}\|\hat\ba_{rm}^{(0)}\hat\ba_{rm}^{(0)\top} - \ba_{rm}\ba_{rm}^{\top})\|_2 \le \psi_0 \}$ for any initial estimates $\hat\ba_{rm}^{(0)}$. Suppose $\alpha>0$, $\rho<1$, with the quantities defined in \eqref{eq:alpha} and \eqref{eq:rho}.
After at most $T=O(\log \log (\psi_0/\psi^{\rm ideal}))$ iterations, Algorithm \ref{alg:tensorlda-cp} produces final estimates satisfying 
\begin{align} \label{eqn:a-bound}
\|\hat{\ba}_{rm}\hat{\ba}_{rm}^\top - \ba_{rm}\ba_{rm}^{\top}\|_2 &\le C\psi^{\rm ideal}_r, \quad 
\abs{\hat w_r -w_r} \le C w_r \psi^{\rm ideal}_r,  \\
\norm{\hat \cB^{\rm cp} - \cB}_{\rm F} &\le C\frac{ \sqrt{\sum_{k=1}^M d_k R}}{\sqrt{n} }   +   C w_1\sqrt{R}  \max_{1\le k\le M} \sqrt{\frac{d_k}{n  d_{-k}}}, \label{eqn:b-bound}
\end{align}
with probability at least $\PP(\Omega_0) - n^{-c} -\sum_{m=1}^M\exp(-c d_m)$. 
\end{theorem}

The statistical convergence rates in \eqref{eqn:a-bound}, \eqref{eqn:b-bound} comprise two components: the first matches conventional tensor CP decomposition rates, while the second reflects the estimation accuracy of the mode-$k$ sample precision matrices $\hat\Sigma_k^{-1}$. By updating each CP basis vector $\ba_{rm}$ separately, our algorithm removes bias from non-orthogonal loading vectors.
Our theoretical analysis employs large deviation inequalities and spectrum perturbation theory \citep{cai2018rate,han2020iterative} to analyze the projected discriminant tensor. Notably, we accommodate correlated noise tensor entries, unlike most existing tensor decomposition work assuming i.i.d. entries \citep{zhang2018tensor,wang2020learning}. The proof appears in Section \ref{sec:proof-cp-converge}.

\begin{remark}
In Theorem \ref{thm:cp-converge}, the initialization accuracy $\psi_0$ must be sufficiently small to ensure $\alpha>0$, $\rho<1$, which impose sharper quantitative requirements on the warm start. The condition $\alpha>0$ requires $R^{1/2}\psi_0$ to be small, with the extra $R^{1/2}$ factor arising from inverting the estimated $A_m^\top A_m$ in constructing mode-$m$ projections. The condition $\rho<1$ ensures that the update map is contractive in a neighborhood of the true CP bases, so the iterates stay within this basin of attraction and the error decays geometrically down to $\psi^{\text{ideal}}$.    
\end{remark}


\noindent
\textbf{Convergence of initialization.}
Theorem \ref{thm:cp-initilization} establishes performance bounds for \textsc{rc-PCA} initialization (Algorithm \ref{alg:initialize-cp}), depending on CP basis incoherence (the degree of non-orthogonality). We consider two complementary regimes: (i) large spectral separation and (ii) small spectral separation, where the large-separation analysis no longer applies.
Recall $ \bA_m= (\ba_{1m},\ldots,\ba_{Rm})$. We measure correlation via
\vspace{-2em}
\begin{align}\label{corr-all}
\delta = \| \bA_S^\top \bA_S - I_{R}\|_{2} \vee \| \bA_{S^c}^\top \bA_{S^c} - I_{R}\|_{2},
\end{align}
where $\bA_S= (\ba_{1,S},\ldots, \ba_{R,S})\in \RR^{d_S\times R}$ and $\bA_{S^c}= (\ba_{1,S^c},\ldots, \ba_{R,S^c})\in \RR^{d_{S^c}\times R}$, with $S$ maximizing $\min(d_S,d_{S^c})$, $d_S=\prod_{m\in S} d_m, d_{S^c}=\prod_{m\in S^c} d_m$, and $\ba_{r,S} = \vect(\circ_{m\in S} \ba_{rm}), \ba_{r,S^c} = \vect(\circ_{m\in S^c} \ba_{rm})$.

\begin{theorem}\label{thm:cp-initilization}
Suppose $C_0^{-1} \leq \lambda_{\min}(\otimes_{m=1}^M \Sigma_m) \leq \lambda_{\max}(\otimes_{m=1}^M \Sigma_m) \leq C_0$ for some constant $C_0 > 0$. Assume $n_{1} \asymp n_{2}$ and $n=\min\{n_{1},n_{2}\}$.

\noindent (i) Large spectral separation. If $\min\{w_r-w_{r+1},w_{r}-w_{r-1}\} \ge c_0 w_R$ for all $1\le r\le R$, with $w_0=\infty, w_{R+1}=0$, and $c_0$ is sufficiently small constant, then with probability at least $1-n^{-c}-\sum_{m=1}^M \exp(-c d_m)$, \textsc{rc-PCA} (Algorithm \ref{alg:initialize-cp}) yields the following error bounds, 
\begin{align}\label{thm:initial:eq1}
\|\widehat \ba_{rm}^{\rm rcpca}\widehat \ba_{rm}^{\rm rcpca\top}  - \ba_{rm} \ba_{rm}^\top \|_{2} &\le \big(1+2\sqrt{2}(w_1/w_R)\big)\delta+C \phi_{1},
\end{align}
for all $1\le r\le R$, $1\le m\le M$, where $C$ is some positive constants, and
\begin{align}\label{thm:initial:eq2}
\phi_{1} &= \frac{\sqrt{d_S}+\sqrt{d_{S^c}}}{\sqrt{n} w_R} +  \frac{w_1}{w_R} \max_{1\le k\le M} \sqrt{\frac{d_k}{n  d_{-k}}}  .    
\end{align}

\noindent (ii) Small spectral separation. When condition (i) fails, assume $w_1\asymp w_R$ and the number of random projections $L\ge C_0 d_1^2 \vee C_0 d_1 R^{2(w_1/w_R)^2}$. Then with probability at least $1-n^{-c}-d_1^{-c}-\sum_{m=2}^M e^{-c d_m}$, \textsc{rc-PCA} (Algorithm \ref{alg:initialize-cp}) yields the following error bounds,
\begin{align}\label{thm:initial:eq1*}
\|\widehat \ba_{rm}^{\rm rcpca}\widehat \ba_{rm}^{\rm rcpca\top}  - \ba_{rm} \ba_{rm}^\top \|_{2} &\le C \sqrt{ \delta_{\max} }+ C \sqrt{\phi_{2} }.
\end{align}
where $C$ is some positive constants, and
\vspace{-1em}
\begin{align}\label{thm:initial:eq3}
\phi_{2} &= \frac{\sqrt{d_1 d_{S_1}}+\sqrt{d_1 d_{S_1^c}} }{\sqrt{n} w_R} +  \frac{w_1}{w_R} \max_{1\le k\le M} \sqrt{\frac{d_k}{n  d_{-k}}}  .    
\end{align}
\end{theorem}


The bounds in \eqref{thm:initial:eq1} and \eqref{thm:initial:eq1*} comprise bias from non-orthogonal loading vectors (first term) and stochastic error from noise and precision matrices estimations (second term). When the eigengap condition is not met, randomized projection yields the slower rate in \eqref{thm:initial:eq1*} compared to \eqref{thm:initial:eq1}. For case (ii), a generalized result for arbitrary eigen ratios $w_1/w_R$ appears in Section \ref{append:proof:initia}.

Using Algorithm \ref{alg:initialize-cp} to initialize Algorithm \ref{alg:tensorlda-cp}, with $\psi_0$ as the maximum of the right-hand sides of \eqref{thm:initial:eq1} or \eqref{thm:initial:eq1*}, the condition in Theorem \ref{thm:cp-converge} satisfies with $\PP(\Omega_0)\ge 1-n^{-c}-d_1^{-c}-\sum_{m=2}^M e^{-c d_m}$.

\begin{remark} 
Our initialization advances existing tensor CP methods in several ways. First, unlike \cite{han2023tensor,han2023cp}, we handle repeated singular values in tensor CP decomposition. Second, \textsc{rc-PCA} requires weaker incoherence conditions for the non-orthogonality of the CP basis vectors under small CP rank. For example, for 3-way tensors, \cite{anandkumar2014guaranteed} requires incoherence condition $\max_{i\neq j}\max_{m} \ba_{im}^\top \ba_{jm}  \le {\rm polylog}(d_{\min})/\sqrt{d_{\min}}$, while our condition $\psi_0\lesssim 1$ yields $\max_{i\neq j}\max_{m} \ba_{im}^\top \ba_{jm}  \lesssim R^{-2}$, which is weaker when $R\lesssim d_{\min}^{1/4}$.

\end{remark}


\noindent
\textbf{Misclassification error.}
We now establish the misclassification error rates for high-dimensional tensor LDA with CP low-rank structure. We derive both upper and lower bounds for the excess misclassification risk, demonstrating that Algorithm \ref{alg:tensorlda-cp} achieves optimal rates.

Under the TGMM and CP-TDA rule defined in \eqref{eqn:lda-rule-cp}, the classifier's performance is measured by its misclassification error
\vspace{-2em}
\begin{align*}
\cR_{\btheta}(\hat\Upsilon_{\rm cp}) = \PP_{\btheta}\big({\rm label}(\cX^*) \neq \hat\Upsilon_{\rm cp}(\cX^*)\big),    
\end{align*}
where $\PP_{\btheta}$ denotes the probability with respect to $\cX^* \sim \pi_1 \cT\cN(\cM_1; \bSigma) + \pi_2 \cT\cN(\cM_2; \bSigma)$ and ${\rm label}(\cX^*)$ denotes the true class of $\cX^*$. In this paper, we use the excess misclassification risk relative to the optimal misclassification error, $\cR_{\btheta}(\hat\Upsilon_{\rm cp}) -\cR_{\rm opt}(\btheta)$, to measure the performance of the classifier $\hat\Upsilon_{\rm cp}$.

\begin{theorem}[Upper bound of misclassification rate]
\label{thm:class-upp-bound}
Assume the conditions of Theorem \ref{thm:cp-converge} hold, and $ w_1\sqrt{\max_m d_m^2 R/(nd)}+\sqrt{\sum_{m=1}^M d_m R/n} = o(\Delta)$ as $n \rightarrow \infty$, where the SNR is $\Delta=\sqrt{\langle \calB, \; \cD \rangle}$. 

\noindent (i) For bounded $\Delta\le c_0$ with some $c_0>0$, then with probability at least $1-n^{-c}-d_1^{-c}-\sum_{m=2}^M e^{-cd_m }$, the misclassification rate of classifier $\hat\Upsilon_{\rm cp}$ satisfies
\begin{equation}\label{eqn:lda-mis-cp1}
\cR_{\btheta}(\hat\Upsilon_{\rm cp}) -\cR_{\rm opt}(\btheta) \le C\left(\frac{\sum_{m=1}^M d_m R }{n} \right), 
\end{equation}
for some constants $C,c>0$.\\
(ii) For $\Delta\to\infty$ as $n\to \infty$, then there exists $\vartheta_n=o(1)$, with probability at least $1-n^{-c}-d_1^{-c}-\sum_{m=2}^M e^{-cd_m }$, the misclassification rate of classifier $\hat\Upsilon_{\rm cp}$ satisfies
\begin{equation}\label{eqn:lda-mis-cp2}
\cR_{\btheta}(\hat\Upsilon_{\rm cp}) -\cR_{\rm opt}(\btheta) \le C \exp\left\{-\left(\frac18+\vartheta_n\right)\Delta^2 \right\} \left(\frac{\sum_{m=1}^M d_m R }{n}  \right), 
\end{equation}
for some constants $C,c>0$.
\end{theorem}

Compared with Theorem \ref{thm:cp-converge}, the excess misclassification rates in \eqref{eqn:lda-mis-cp1} and \eqref{eqn:lda-mis-cp2} only contain one part. The error induced by the estimation accuracy of the mode-$m$ precision matrix becomes negligible in the excess misclassification rates.

To understand the difficulty of the tensor classification problem, it is essential to obtain the minimax lower bounds for the excess misclassification risk. We especially consider the following parameter space of CP low-rank discriminant tensors,
\begin{align*}
\calH= \Big\{\;& \btheta=(\cM_1, \cM_2, \bSigma): \cM_1,\cM_2 \in \RR^{d_1 \times \cdots \times d_M}, \; \bSigma = [\Sigma_m]_{m=1}^M, \; \Sigma_m \in \RR^{d_m \times d_m}, \; \text{for some} \; C_0>0, \\
&C_0^{-1} \leq \lambda_{\min}(\otimes_{m=1}^M \Sigma_m) \leq \lambda_{\max}(\otimes_{m=1}^M \Sigma_m) \leq C_0, \; \calB =\sum_{r=1}^R w_r \circ_{m=1}^M \ba_{rm}\; \text{with}\; \|\ba_{rm}\|_2=1 \Big\}   . 
\end{align*}
The following lower bound holds over $\cH$.

\begin{theorem}[Lower bound of misclassification rate] \label{thm:class-lower-bound}
Under the TGMM, the minimax risk of excess misclassification error over the parameter space $\calH$ satisfies the following conditions.

\noindent (i) For $c_1<\Delta \le c_2$ with constants $c_1,c_2>0$, for any $\gamma>0$, there exists constant $C_{\gamma}>0$ such that
\begin{equation}\label{eqn:lda-lbd-cp1}
\inf_{\hat \Upsilon_{\rm cp}} \sup_{\btheta \in \calH} \PP\left( \cR_{\btheta}(\hat\Upsilon_{\rm cp}) -\cR_{\rm opt}(\btheta) \ge C_{\gamma}  \frac{\sum_{m=1}^M d_m R}{n}  \right) \ge 1-\gamma   .   
\end{equation}
(ii) For $\Delta\to\infty$ as $n\to \infty$, for any $\gamma>0$, there exists constant $C_{\gamma}>0$ and $\vartheta_n=o(1)$ such that
\begin{equation}\label{eqn:lda-lbd-cp2}
\inf_{\hat \Upsilon_{\rm cp}} \sup_{\btheta \in \calH} \PP\left(\cR_{\btheta}(\hat\Upsilon_{\rm cp}) -\cR_{\rm opt}(\btheta) \ge C_{\gamma} \exp\left\{-\left(\frac18+\vartheta_n\right)\Delta^2 \right\} \frac{\sum_{m=1}^M d_m R }{n}  \right) \ge 1-\gamma   .   
\end{equation}
\end{theorem}

Comparing the upper bounds of the excess misclassification risk \eqref{eqn:lda-mis-cp1}--\eqref{eqn:lda-mis-cp2} with lower bounds \eqref{eqn:lda-lbd-cp1}--\eqref{eqn:lda-lbd-cp2}, our convergence rates are minimax optimal.


\section{Semiparametric CP-TDA via Representation Learning}
\label{sec:Tensor LDA-TNN}

Although TGMM provides a clean parametric setting where CP-TDA is optimal, real data such as images, graphs, and multimodal observations do not naturally follow tensor-normal structure. Existing semiparametric mixture models based on coordinatewise monotone transformations \citep{liu2012high, han2013coda} are too restrictive for modern high-dimensional data. To extend CP-TDA to general settings, we propose a new semiparametric approach where a neural network encoder first maps inputs into tensor-valued features, and a tensor Gaussianizing flow then transforms these features into a latent space that closely approximates TGMM. CP-TDA is applied to this latent representation, combining the expressive power of deep learning with the robustness, parsimony, and interpretability of CP-TDA, and enabling improved empirical performance while preserving theoretical guarantees of the discriminant model. We call this Semiparametric CP-TDA framework Semiparametric Tensor Discriminant Networks (STDN).

\noindent
\textbf{Architecture Overview.} STDN augments the parametric CP-TDA framework with three components. 
First, a tensor encoder $h_\beta$ maps raw observations $\calZ$ with class labels $Y \in [K]$ (e.g., images, graphs, text, or vectors) into tensor-valued features $\cX_\beta(\cZ) \in \mathbb{R}^{d_1 \times \cdots \times d_M}$ using architecture-agnostic modules (e.g., convolutional, graph-based or transformer-style) whose role is purely representational, producing multiway features that preserve the structural richness of the input domain. 

Second, a new tensor Gaussianizing flow $g_\varphi$ (FlowTGMM) transforms encoder outputs into latent tensor $\cX^{(L)} = g_\varphi(\cX_\beta(\cZ)) \in \mathbb{R}^{d_1 \times \cdots \times d_M}$ whose class-conditional distributions approximate TGMM. Built on the tensor extension of RealNVP flow \citep{dinh2017density}, FlowTGMM learns flexible, mode-aware, class-conditioned transformations mapping encoder features to a latent space where each class approximately follows a tensor normal distribution with common mode-wise covariances. 

Finally, the CP-TDA head applies CP low-rank discriminant analysis from Section \ref{sec:model} to the latent tensors. Unlike fully connected softmax classifiers, it imposes parsimonious multilinear structure with $O(R\sum_m d_m)$ parameters, yielding mode-wise interpretability and sample efficiency. This creates a semiparametric model where deep neural networks learn features while the classifier maintains statistically optimal CP low-rank discriminant structure rather than a generic softmax layer.

To the best of our knowledge, this is the first tensor classifier combining deep representation learning and flow-based deep generative models with a statistically optimal low-rank tensor discriminant analysis, achieving both neural network flexibility and rigorous theoretical guarantees.

\subsection{Tensor Gaussianizing Flow (FlowTGMM)}
\label{subsec:FlowTGMM}
Encoder outputs $\mathcal{X}_\beta(\mathcal{Z})$ rarely follow TGMM naturally. In practice, they exhibit heterogeneous feature scales, non-Gaussian tails, and complex local geometry that violate the tensor-normal assumption required for CP-TDA's optimality. To bridge this gap, we propose FlowTGMM, a learnable transformation that maps encoder outputs into a latent tensor space $\mathcal{X}^{(L)} = g_\varphi(\mathcal{X}_\beta(\mathcal{Z})) \in \mathbb{R}^{d_1 \times \cdots \times d_M}$ whose class-conditional distributions closely approximate TGMM. This alignment enables the CP-TDA head to operate in precisely the setting where it achieves statistical efficiency, interpretability, and minimax-optimal misclassification rates.

Specifically, FlowTGMM adopts Tensor RealNVP (Definition~\ref{def:realnvp}), a tensor-structured extension of RealNVP \citep{dinh2017density} that uses mode-wise invertible linear mixing and affine coupling layers to preserve multiway structure. 
Unlike traditional semiparametric mixture models that rely on restrictive coordinatewise monotone transformations, FlowTGMM learns rich, nonlinear, mode-aware transformations that adapt to the complexity of high-dimensional tensor data while maintaining computational tractability and theoretical rigor.

\begin{definition}[Tensor RealNVP]
\label{def:realnvp}
Consider a pre-specified binary mask tensor $\cK^{(\ell)} \in \{0,1\}^{d_1 \times \cdots \times d_M}$ with supports $\calA_1^{(\ell)} = \{\bi=(i_1,\dots,i_M) : \cK^{(\ell)}_{\bi} = 1 \}, \; \calA_0^{(\ell)} = \{\bi=(i_1,\dots,i_M) : \cK^{(\ell)}_{\bi} = 0\}$. Each mapping $g^{(\ell)} : \RR^{d_1 \times \cdots \times d_M} \rightarrow \RR^{d_1 \times \cdots \times d_M}$ operates on $\cX^{(\ell-1)}$ in two steps: (i) Mode-wise linear mixer $\widetilde{\cX}^{(\ell)} := \cX^{(\ell-1)} \times_{m=1}^M \bH_m^{(\ell)}$ with pre-specified orthogonal matrices $\bH_m^{(\ell)} \in \RR^{d_m \times d_m}$. (ii) Tensor affine coupling: partition $\widetilde{\cX}^{(\ell)}$ into $\widetilde{\cX}_{\calA_1^{(\ell)}}^{(\ell)}, \widetilde{\cX}_{\calA_0^{(\ell)}}^{(\ell)} \in \RR^{d_1 \times \cdots \times d_M}$ according to the mask. Let $\upsilon^{(\ell)}, t^{(\ell)}$ be tensor-valued conditioners mapping $\widetilde{\cX}_{\calA_1^{(\ell)}}^{(\ell)}$ to tensors supported on $\calA_0^{(\ell)}$ of the same shape. Define
\begin{align*}
\calY_{\calA_1^{(\ell)}}^{(\ell)}:= \widetilde{\cX}_{\calA_1^{(\ell)}}^{(\ell)}, \quad
\calY_{\calA_0^{(\ell)}}^{(\ell)}:= \widetilde{\cX}_{\calA_0^{(\ell)}}^{(\ell)} \odot \exp\left(\upsilon^{(\ell)}\left(\widetilde{\cX}_{\calA_1^{(\ell)}}^{(\ell)}\right)\right) + t^{(\ell)}\left(\widetilde{\cX}_{\calA_1^{(\ell)}}^{(\ell)}\right),
\end{align*}
where $\odot$ and $\exp(\cdot)$ denote elementwise multiplication and exponentiation. The output is $\cX^{(\ell)} := \calY_{\calA_1^{(\ell)}}^{(\ell)} + \calY_{\calA_0^{(\ell)}}^{(\ell)}$. Composing $L$ such mappings defines the full Tensor RealNVP flow:
\[
g_\varphi := g^{(L)} \circ \cdots \circ g^{(1)}, \quad \cX^{(0)} := \cX_\beta(\cZ), \quad \cX^{(\ell)} := g^{(\ell)}(\cX^{(\ell-1)}), \ \ell = 1, \ldots, L,
\]
with parameters $\varphi = \{\upsilon^{(\ell)}, t^{(\ell)} : \ell = 1, \ldots, L \}$. The overall flow output is $\cX^{(L)} = g_\varphi(\cX_\beta(\cZ))$, which preserves the tensor order and dimensions $(d_1, \ldots, d_M)$.
\end{definition}

We propose a class-aware flow training scheme that maximizes the TGMM likelihood of latent tensors given class information. Specifically, we train $g_{\varphi}$ so that transformed tensor features $\cX^{(L)} = g_\varphi(\cX_\beta(\cZ))$ approximately follow, conditional on class, a TGMM with common mode-wise covariances $[\Sigma_m]_{m=1}^M$. This yields an objective that combines class-conditional tensor-normal log-likelihood with flow-induced Jacobian correction, encouraging latent representations of each class to concentrate around their respective means with common mode-wise covariances, precisely the setting in which CP-TDA is optimal.


\begin{proposition}[FlowTGMM]
\label{prop:log-likelihood}
Let $g_\varphi = g^{(L)} \circ \cdots \circ g^{(1)}$ be the Tensor RealNVP from Definition \ref{def:realnvp} with $\cX_i^{(0)}=\cX_\beta(\cZ_i),\; \cX_i^{(\ell)} = g^{(\ell)}(\cX_i^{(\ell-1)}), \; \ell = 1, \ldots, L$. Let $f(\cdot \mid Y = k; \calM_k, \bSigma)$ denote the tensor normal density with class means $\calM_k$ and mode-wise covariances $\bSigma=[\Sigma_m]_{m=1}^M$, as defined in Section \ref{sec:notation}. Given a labeled dataset $\{(\calZ_i, Y_i)\}_{i=1}^{n_0}$, the maximum likelihood estimator $\widehat{\varphi}$ minimizes
\begin{equation}
\label{eqn:objective function}
\calL_{n_0,\text{Flow}}(\varphi) = -\sum_{i=1}^{n_0} \left[ \log f(\cX_i^{(L)} \mid Y = Y_i; \hat \calM_{Y_i}, \hat \bSigma) + \sum_{\ell=1}^L \log \left| \det \bJ_{g^{(\ell)}}(\cX_i^{(\ell-1)}) \right| \right],
\end{equation}
where $\hat \calM_{Y_i}$ and $\hat \bSigma$ are computed from latent features $\cX_i^{(L)}$. Here, the log-Jacobian for each mapping $g^{(\ell)}$ with orthogonal mode-mixer $\bH_m^{(\ell)}$ and affine coupling conditioners $\upsilon^{(\ell)}, t^{(\ell)}$, is defined as
\begin{equation}
\label{eqn:Jacobian}
\log \left| \det \bJ_{g^{(\ell)}}(\cX_i^{(\ell-1)}) \right| = \sum_{\bj \in \calA_0^{(\ell)}} \upsilon^{(\ell)}(\tilde \cX_{i,\mathcal{A}_1^{(\ell)}}^{(\ell)})_{\bj}. 
\end{equation}
\end{proposition}

\begin{remark}
The sample log-likelihood \eqref{eqn:objective function} decomposes into two components: (i) the first term $\log f(\cX_i^{(L)} \mid Y = Y_i; \hat \calM_{Y_i}, \hat \bSigma)$ measures how well the latent tensor $\cX_i^{(L)}$ fits the class-conditional tensor-normal model, encouraging TGMM alignment; (ii) the second term $\sum_{\ell=1}^{L} \log |\det \bJ_{g^{(\ell)}}(\cX_i^{(\ell-1)})|$ accounts for the flow-induced change of variables, ensuring proper density normalization. Each log-Jacobian in \eqref{eqn:Jacobian} consists of affine coupling contributions from $v^{(\ell)}$; the mode-wise mixer contribution vanishes since the mixers are orthogonal, and the term $t^{(\ell)}$ disappear because it only affects off-diagonal entries in the triangular Jacobian matrix. See proofs in Appendix \ref{app:proof1}.
\end{remark}

\subsection{CP-TDA Cross-Entropy Head}
The CP-TDA head operates on latent tensors $\{\cX_i^{(L)}\}_{i=1}^{n_0}$ produced by the encoder-flow pipeline, where $\cX_i^{(L)} = g_\varphi(h_\beta(\cZ_i))$.
It first computes statistics from the latent tensors as in 
Section \ref{sec:method}, including class priors $\widehat{\pi}_k = n_k/(n_1 + n_2)$, class means $\widehat{\cM}_k^{(L)}$, overall mean $\widehat{\cM}^{(L)} = (\widehat{\cM}_1^{(L)} + \widehat{\cM}_2^{(L)})/2$, mean difference $\widehat{\cD}^{(L)} = \widehat{\cM}_2^{(L)} - \widehat{\cM}_1^{(L)}$, mode-wise covariances $\{\widehat{\Sigma}_m^{(L)}\}$, and the sample discriminant tensor $\widehat{\cB}^{(L)} = \widehat{\cD}^{(L)} \times_{m=1}^M (\widehat{\Sigma}_m^{(L)})^{-1}$. It then refines $\widehat{\cB}^{(L)}$ via \textsc{rc-PCA} initialization followed by DISTIP-CP (Algorithms \ref{alg:initialize-cp}-\ref{alg:tensorlda-cp}), yielding a CP rank-$R$ discriminant tensor
$\widehat\calB^{{\rm cp},(L)} = \sum_{r=1}^R \widehat w_r \left( \widehat{\ba}_{r1} \circ \cdots \circ \widehat{\ba}_{rM}\right)$. For input $\cZ^*$, the latent representation is $\cX^{*(L)} = g_\varphi(h_\beta(\cZ^*))$, and the classifier score is
\[
s_{\beta,\varphi}(\cZ^*) = \left\langle \widehat\calB^{{\rm cp},(L)}, \cX^{*(L)} - \calM^{(L)} \right\rangle + \log (\widehat{\pi}_2/\widehat{\pi}_1),
\]
which follows the decision rule in \eqref{eqn:lda-rule-cp}.
During training, we construct logits $\{-\frac{1}{2} s_{\beta,\varphi}(\cZ^*), \frac{1}{2} s_{\beta,\varphi}(\cZ^*)\}$ with cross-entropy loss. By constraining the discriminant tensor to CP rank $R$ and using common mode-wise covariances, the CP-TDA head achieves substantial parameter efficiency, reducing complexity from $O(\prod_m d_m)$ to $O(R \sum_m d_m)$, while providing mode-wise interpretability and maintaining the minimax-optimal guarantees from Section~\ref{sec:theorems}. The complete semiparametric CP-TDA (STDN) is summarized in Algorithm~\ref{alg:snapshot-training} with additional training details provided in Appendix \ref{app:training}.

\subsection{Theoretical Guarantees of Semiparametric CP-TDA (STDN)}

Having obtained network parameters $\hat{\beta}$ and $\hat{\varphi}$ from Algorithm \ref{alg:snapshot-training}, we now analyze the theoretical performance of the resulting classifier. Let $\PP_{\hat\beta,\hat\varphi,k}$ denote the true class-conditional distribution of latent tensor $\cX^{(L)} \mid Y = k$ with priors $\pi_k$, and write $\PP_{\hat\beta,\hat\varphi} = (\PP_{\hat\beta,\hat\varphi,1}, \PP_{\hat\beta,\hat\varphi,2}, \pi_1, \pi_2)$. To enable theoretical analysis, we introduce two population benchmarks. First, since the flow may not produce perfect TGMM distributions, we construct a moment-matched TGMM surrogate $\QQ_{\hat\beta,\hat\varphi}$ with class means $\calM_k = \mathbb{E}[\cX^{(L)} \mid Y = k]$, common mode-wise covariances $[\Sigma_m]_{m=1}^M$, and priors $\pi_k$, yielding oracle LDA with $\calD = \calM_2 - \calM_1$, discriminant tensor $\calB = \calD \times_{m=1}^M \Sigma_m^{-1}$, and SNR $\Delta = \sqrt{\langle \calB, \calD \rangle}$, as in Section \ref{sec:oracle lda}. Second, since $\mathcal{B}$ may not be CP low-rank while our method constrains to rank $R$ for efficiency, we define the oracle CP rank-$R$ discriminant tensor as the best rank-$R$ approximation:
\[
\mathcal{B}_R \in \arg\max_{\widetilde{\mathcal{B}}_R \in \mathcal{C}_R} \langle \mathcal{D}, \widetilde{\mathcal{B}}_R \rangle \quad \text{s.t.} \quad \| \widetilde{\mathcal{B}}_R \times_{m=1}^M \Sigma_m^{1/2} \|_F = \| \mathcal{B} \times_m \Sigma_m^{1/2} \|_F,
\]
where $\mathcal{C}_R$ denotes the cone of CP rank-$R$ tensors, yielding SNR $\Delta_R = \sqrt{\langle \mathcal{D}, \mathcal{B}_R \rangle}$ and efficiency loss $\delta_{\rm cp} := \Delta - \Delta_R \geq 0$.
The semiparametric classifier $\widehat{\Upsilon}_{\hat\beta,\hat\varphi}$ uses the estimated CP rank-$R$ discriminant tensor $\widehat{\calB}_{R}^{(L)}$ from empirical estimates $\{\widehat{\pi}_k, \widehat{\calM}_k^{(L)}, [\widehat{\Sigma}_m^{(L)}]\}$. The following theorem bounds the excess misclassification risk relative to the optimal error $\calR_{\PP_{\hat\beta,\hat\varphi}}^*$.

\begin{theorem}[Misclassification error rate]
\label{thm:semi-risk}
Assume the conditions of Theorem \ref{thm:class-upp-bound} and
Assumption~\ref{assump:flow_regularity} in the appendix hold. (i) For bounded $\Delta_R\le\Delta\le c_0$ with some $c_0>0$, with probability at least $1-n^{-c}-d_1^{-c}-\sum_{m=2}^M e^{-cd_m }$, the misclassification rate of classifier $\widehat{\Upsilon}_{\hat\beta,\hat\varphi}$ satisfies
\begin{align}\label{eqn:semilda-mis1}
\calR_{\PP_{\hat\beta,\hat\varphi}}(\widehat{\Upsilon}_{\hat\beta,\hat\varphi}) - \calR_{\PP_{\hat\beta,\hat\varphi}}^* \le   C \Big(\delta_{\rm cp} +  \frac{\sum_m d_m R}{n} + \Big(\frac{\log n}{n}\Big)^{1/4} \Big), 
\end{align}
where $n=\min\{n_1,n_2\}$ and $C,c>0$ are constants.\\
(ii) For $\Delta \ge \Delta_R \to\infty$ as $n\to \infty$, then there exists $\vartheta_n=o(1)$, with probability at least $1-n^{-c}-d_1^{-c}-\sum_{m=2}^M e^{-cd_m }$, the misclassification rate of classifier $\widehat{\Upsilon}_{\hat\beta,\hat\varphi}$ satisfies
\begin{align}\label{eqn:semilda-mis2}
\calR_{\PP_{\hat\beta,\hat\varphi}}(\widehat{\Upsilon}_{\hat\beta,\hat\varphi}) - \calR_{\PP_{\hat\beta,\hat\varphi}}^* \le C\Big( \delta_{\rm cp}\  e^{- \big(\frac{(\Delta_R/\Delta)^4}{8}+\vartheta_n \big) \Delta^2} + \frac{\sum_{m=1}^M d_m R }{n} e^{- (\frac18+\vartheta_n )\Delta_R^2 }  + \Big(\frac{\log n}{n}\Big)^{1/4} \Big), 
\end{align}
where $n=\min\{n_1,n_2\}$ and $C,c>0$ are constants.

\end{theorem}

The bounds in \eqref{eqn:semilda-mis1} and \eqref{eqn:semilda-mis2} decompose into three terms isolating distinct error sources: (i) population level approximation bias $\delta_{\rm cp}$ from restricting the discriminant tensor $\cB$ to a CP rank-$R$ discriminant tensor $\cB_R$, which is typically small in practice; (ii) statistical estimation error, matching the minimax-optimal rate for CP-TDA under TGMM from Theorem~\ref{thm:class-upp-bound}; and (iii) flow-induced mismatch of order $(\log n/n)^{1/4}$, arising from optimization and generalization error of the FlowTGMM objective. This flow mismatch rate is consistent with conventional rates in flow-based generative models \citep{irons2022triangular}. As shown in Appendix~\ref{append:proof2}, Tensor RealNVP is expressive enough to approximate the latent TGMM arbitrarily well, so no additional approximation error arises beyond this statistical mismatch.




\section{Simulation Study}
\label{sec:simu}






Our simulations confirm the theoretical properties of CP-TDA: \textsc{rc-PCA} reliably provides a valid warm start and DISTIP-CP converges to the true CP components. 
Under ideal CP low-rank setting, CP-TDA achieves substantially lower estimation and classification errors than existing tensor classifiers. Even when the true discriminant deviates from CP structure, CP-TDA remains superior, demonstrating strong robustness under model misspecification. Additional studies in Appendix \ref{append:simulation} examines: (i) heavy-tailed non-Gaussian distributions, (ii) heteroscedastic covariances violating the common-covariance assumption, (iii) rank misspecification and data-driven rank selection, (iv) ultra high-dimensional settings, and (v) class imbalance. Across all settings, CP-TDA maintains superior performance, confirming its practical reliability beyond theoretical assumptions.

\noindent
\underline{\textbf{Data generation.}} We generate data under the TGMM with order-$3$ tensors of dimension $(d_1,d_2,d_3)$ $=(30,30,30)$. Each class contains $n_1=n_2=100$ training samples, with $1000$ test samples for evaluation. Mode-wise covariances $\bSigma=[\Sigma_m]_{m=1}^3$ have unit diagonals with off-diagonal elements $3/d_m$. 
We construct $\{\ba_{rm}\}$ using both orthogonal bases and non-orthogonal bases defined by $\tilde \ba_{1m} = \ba_{1m}$ and $\tilde \ba_{rm} = \frac{\ba_{1m} + \eta \ba_{rm}}{ \norm{\ba_{1m} + \eta \ba_{rm}}_2}$ for $r \ge 2$, where $\eta = (\vartheta^{-2/M} - 1)^{1/2}$, $\vartheta = \delta / (r-1)$ and $\delta=0.1$. We examine two eigen-structure regimes: (i) small spectral separation, where all components have equal signal strength $w_r\in\{1.5, 2.0, 2.5\}$; and (ii) large spectral separation, where $w_{r}/w_{r+1}=1.25$ with $w_{\rm max}\in\{2, 3, 4\}$, yielding one dominant component. Class means are $\calM_1=0$ and $\calM_2=\calB \times_{m=1}^M \Sigma_m$, where $\calB=\sum_{r=1}^R w_r a_{r1}\circ\cdots\circ a_{rM}$. Each configuration is repeated 100 times.


\noindent
\underline{\textbf{\textsc{rc-PCA} ensures global convergence of DISTIP-CP.}} 
Figure \ref{fig:1} illustrates the behavior of different initialization schemes for DISTIP-CP with non-orthogonal bases under both small and large spectral separation regimes. It displays the logarithm of the CP basis estimation error, measured by $\max_{r,m} \| \hat{\ba}_{rm} \hat{\ba}_{rm}^\top - \ba_{rm} \ba_{rm}^\top \|_2$, across configurations with varying signal strengths. 

Alternating Rank-1 Least Squares (ARLS, \cite{anandkumar2014guaranteed}), the standard CP initialization, fails under small signal strength, where the sample discriminant tensor has a weak signal to noise ratio that violates its assumptions. In such cases, ARLS either diverges or converges to a local optimum, failing to enter the contraction region of the global optimum.

In contrast, \textsc{rc-PCA} consistently provides a reliable warm start in all cases, after which DISTIP-CP rapidly converges to the true CP components, matching Theorem \ref{thm:cp-initilization}. Initialization is thus the fundamental bottleneck in discriminant tensor estimation: without a proper warm start, global recovery is unattainable, whereas \textsc{rc-PCA} reliably enables it under weaker signal requirements.
This suggests that \textsc{rc-PCA} may also benefit other CP-based tensor methods whose performance depends critically on initialization quality.
Figures \ref{fig:2}-\ref{fig:3} in Appendix \ref{subsec:others} further confirm \textsc{rc-PCA}'s reliability under order-3 tensors with orthogonal CP bases and order-4 tensors.

\begin{figure}[ht]
    \centering
    \resizebox{0.8\textwidth}{!}{
    \begin{subfigure}[b]{0.43\textwidth}
        \centering
        \includegraphics[width=\textwidth]{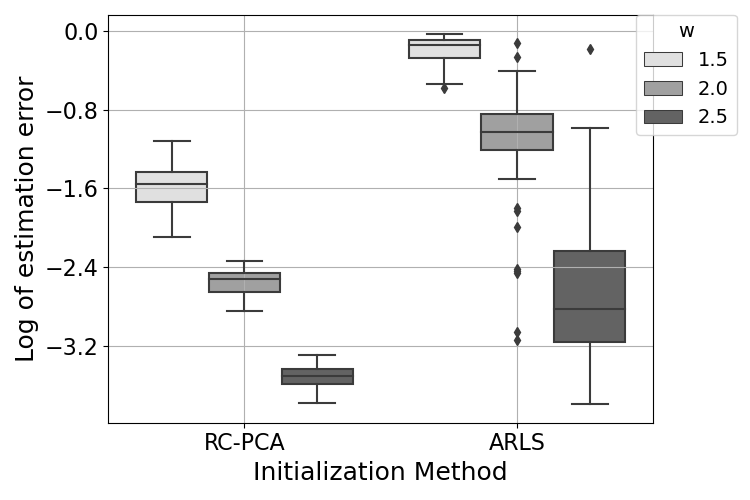}
        \caption{Small spectral separation}
        \label{fig:1a}
    \end{subfigure}
    \hspace{0.1\textwidth}
    \begin{subfigure}[b]{0.43\textwidth}
        \centering
        \includegraphics[width=\textwidth]{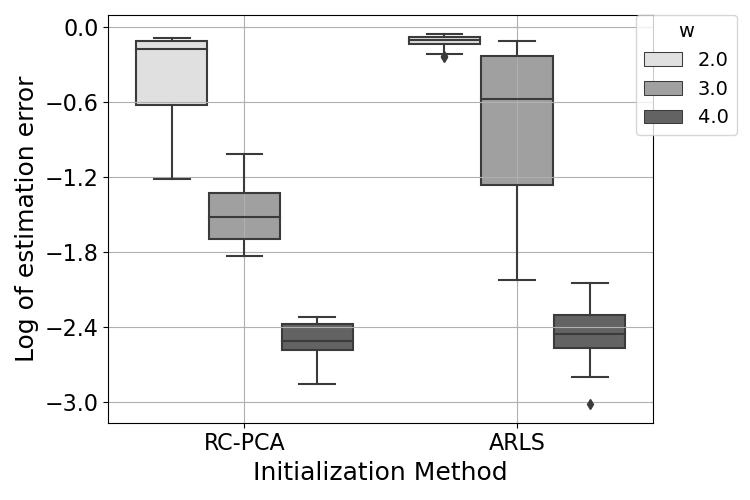}
        \caption{Large spectral separation}
        \label{fig:1b}
    \end{subfigure}
    }
    \caption{\small Logarithmic CP basis estimation error for \textsc{rc-PCA} and ARLS with non-orthogonal bases under (a) small spectral separation (equal $w_r$) and (b) large spectral separation (geometric decay with $w_{r}/w_{r+1} = 1.25$).}
    \label{fig:1}
\end{figure}

\noindent
\underline{\textbf{Performance under CP-structured discriminant model.}} Table \ref{tab:1-2} compares CP-TDA with leading tensor-based classifiers, including the sample discriminant tensor, TuckerLDA, CATCH \citep{pan2019covariate}, DATED \citep{wang2024parsimonious}, GuoSTM \citep{guoxian2016}, LaiSTDA \citep{lai2013sparse}, and ZhouLogitCP \citep{zhou2013tensor}. Implementation details are provided in Appendix \ref{subsec:others}. Across all settings, CP-TDA achieves substantially lower discriminant-tensor estimation error and misclassification error. 
Appendix \ref{subsec:others} extends these results to order-3 tensors with orthogonal CP bases (Table \ref{tab:3-4}) and order-4 tensors with both orthogonal and non-orthogonal bases (Tables \ref{tab:5-6}-\ref{tab:7-8}), under both spectral separation regimes. Overall, DISTIP-CP with \textsc{rc-PCA} initialization provides uniformly strong performance across signal levels, incoherence regimes,  and covariance structures.

\begin{table}[ht]
\centering
\resizebox{0.85\textwidth}{!}{
\begin{tabular}{cccccccc}
\toprule
\multirow{2}{*}{Metric} & \multirow{2}{*}{Algorithms} & \multicolumn{3}{c}{Small spectral separation} & \multicolumn{3}{c}{Large spectral separation} \\
\cmidrule(lr){3-5} \cmidrule(lr){6-8}
& & $w = 1.5$ & $w = 2.0$ & $w = 2.5$ & $w_{max} = 2$ & $w_{max} = 3$ & $w_{max} = 4$ \\
\midrule
\multirow{6}{*}{EstError} 
& Sample $\calB$ & 5.24\textsubscript{(0.19)} & 5.03\textsubscript{(0.15)} & 3.97\textsubscript{(0.10)} & 5.63\textsubscript{(0.15)} & 4.71\textsubscript{(0.15)} & 3.49\textsubscript{(0.12)} \\
& \cellcolor{gray!20} CP-TDA & \cellcolor{gray!20} 0.93\textsubscript{(0.02)} & \cellcolor{gray!20} 0.86\textsubscript{(0.02)} & \cellcolor{gray!20} 0.67\textsubscript{(0.01)} & \cellcolor{gray!20} 1.07\textsubscript{(0.02)} & \cellcolor{gray!20} 0.91\textsubscript{(0.02)} & \cellcolor{gray!20} 0.56\textsubscript{(0.01)} \\
& TuckerLDA & 1.23\textsubscript{(0.02)} & 1.02\textsubscript{(0.01)} & 0.84\textsubscript{(0.01)} & 1.34\textsubscript{(0.02)} & 1.07\textsubscript{(0.02)} & 0.81\textsubscript{(0.01)} \\
& CATCH & 1.10\textsubscript{(0.00)} & 1.01\textsubscript{(0.00)} & 1.00\textsubscript{(0.00)} & 1.09\textsubscript{(0.00)} & 1.00\textsubscript{(0.00)} & 0.98\textsubscript{(0.00)} \\
& DATED & 2.30\textsubscript{(0.01)} & 2.05\textsubscript{(0.01)} & 1.85\textsubscript{(0.00)} & 2.44\textsubscript{(0.01)} & 2.10\textsubscript{(0.01)} & 1.69\textsubscript{(0.00)} \\
& ZhouLogitCP & 1.96\textsubscript{(0.07)} & 1.89\textsubscript{(0.05)} & 1.49\textsubscript{(0.04)} & 2.11\textsubscript{(0.05)} & 1.77\textsubscript{(0.05)} & 1.30\textsubscript{(0.04)} \\
\midrule
\multirow{8}{*}{Misclass} 
& Sample $\calB$ & 0.25\textsubscript{(0.02)} & 0.12\textsubscript{(0.01)} & 0.06\textsubscript{(0.01)} & 0.28\textsubscript{(0.02)} & 0.17\textsubscript{(0.02)} & 0.07\textsubscript{(0.01)} \\
& \cellcolor{gray!20} CP-TDA & \cellcolor{gray!20} 0.08\textsubscript{(0.01)} & \cellcolor{gray!20} 0.03\textsubscript{(0.00)} & \cellcolor{gray!20} 0.00\textsubscript{(0.00)} & \cellcolor{gray!20} 0.11\textsubscript{(0.01)} & \cellcolor{gray!20} 0.05\textsubscript{(0.00)} & \cellcolor{gray!20} 0.00\textsubscript{(0.00)} \\
& TuckerLDA & 0.13\textsubscript{(0.01)} & 0.05\textsubscript{(0.01)} & 0.02\textsubscript{(0.00)} & 0.15\textsubscript{(0.01)} & 0.07\textsubscript{(0.00)} & 0.01\textsubscript{(0.00)} \\
& CATCH & 0.36\textsubscript{(0.03)} & 0.29\textsubscript{(0.02)} & 0.22\textsubscript{(0.02)} & 0.38\textsubscript{(0.03)} & 0.31\textsubscript{(0.02)} & 0.28\textsubscript{(0.02)} \\
& DATED & 0.37\textsubscript{(0.01)} & 0.33\textsubscript{(0.01)} & 0.26\textsubscript{(0.01)} & 0.39\textsubscript{(0.01)} & 0.34\textsubscript{(0.01)} & 0.29\textsubscript{(0.01)} \\
& GuoSTM & 0.28\textsubscript{(0.01)} & 0.23\textsubscript{(0.01)} & 0.18\textsubscript{(0.01)} & 0.26\textsubscript{(0.01)} & 0.20\textsubscript{(0.01)} & 0.17\textsubscript{(0.01)} \\
& LaiSTDA & 0.42\textsubscript{(0.01)} & 0.37\textsubscript{(0.01)} & 0.30\textsubscript{(0.01)} & 0.43\textsubscript{(0.01)} & 0.39\textsubscript{(0.01)} & 0.33\textsubscript{(0.01)} \\
& ZhouLogitCP & 0.33\textsubscript{(0.02)} & 0.29\textsubscript{(0.01)} & 0.23\textsubscript{(0.01)} & 0.35\textsubscript{(0.01)} & 0.30\textsubscript{(0.01)} & 0.21\textsubscript{(0.01)} \\
\bottomrule
\end{tabular}
}
\caption{\small Discriminant tensor estimation error (EstError, $\|\widehat{\mathcal{B}} - \mathcal{B}\|_F/\|\mathcal{B}\|_F$) and misclassification rates (Misclass) for binary classification with non-orthogonal CP bases, under small spectral separation and large spectral separation. Best performance are highlighted in gray.}
\label{tab:1-2}
\end{table}

\noindent
\underline{\textbf{Performance under misspecified discriminant structures.}}
We next evaluate robustness when the true discriminant tensor departs from the ideal CP structure. Two misspecification mechanisms are considered: (i) a Tucker low-rank tensor with controlled off-diagonal energy, and (ii) a CP tensor contaminated by sparse perturbation. Detailed descriptions are provided in Appendix~\ref{subsec:structureRobust}. 

Tables \ref{tab:robustness}-\ref{tab:robustness_Berror} show misclassification rates and, when available, estimation errors. Across all degrees of misspecification, moderate to large Tucker off-diagonal energy ($\alpha=0.3,0.5,0.7$) and sparse perturbations of varying magnitude ($\tau=0.2,0.4,0.6$), CP-TDA remains the top performer. In the Tucker setting, its misclassification error is substantially lower than that of TuckerLDA, CATCH, GuoSTM, LaiSTDA, ZhouLogitCP, and DATED, even when the true structure is closer to Tucker than to CP. In the CP-plus-sparse setting, CP-TDA again achieves the lowest misclassification rates and estimation errors, while several competing methods deteriorate sharply or fail to recover meaningful structure.

\begin{table}[ht]
\centering
\resizebox{0.9\textwidth}{!}{
\begin{tabular}{lccccccc}
\toprule
Setting & \cellcolor{gray!20} CP-TDA & TuckerLDA & CATCH & GuoSTM & LaiSTDA & ZhouLogitCP & DATED \\
\midrule
$\alpha=0.3$   & \cellcolor{gray!20} 0.055\textsubscript{(0.004)} & 0.131\textsubscript{(0.006)} & 0.218\textsubscript{(0.003)} & 0.169\textsubscript{(0.001)} & 0.465\textsubscript{(0.005)} & 0.260\textsubscript{(0.012)} & 0.194\textsubscript{(0.006)} \\
$\alpha=0.5$   & \cellcolor{gray!20} 0.049\textsubscript{(0.003)} & 0.076\textsubscript{(0.005)} & 0.179\textsubscript{(0.003)} & 0.096\textsubscript{(0.001)} & 0.467\textsubscript{(0.005)} & 0.256\textsubscript{(0.012)} & 0.158\textsubscript{(0.006)} \\
$\alpha=0.7$   & \cellcolor{gray!20} 0.023\textsubscript{(0.004)} & 0.025\textsubscript{(0.002)} & 0.131\textsubscript{(0.003)} & 0.059\textsubscript{(0.001)} & 0.446\textsubscript{(0.006)} & 0.207\textsubscript{(0.013)} & 0.131\textsubscript{(0.005)} \\
\midrule
$\tau=0.2$   & \cellcolor{gray!20} 0.034\textsubscript{(0.002)} & 0.096\textsubscript{(0.006)} & 0.245\textsubscript{(0.003)} & 0.114\textsubscript{(0.001)} & 0.464\textsubscript{(0.004)} & 0.250\textsubscript{(0.013)} & 0.213\textsubscript{(0.006)} \\
$\tau=0.4$   & \cellcolor{gray!20} 0.022\textsubscript{(0.002)} & 0.096\textsubscript{(0.005)} & 0.129\textsubscript{(0.002)} & 0.108\textsubscript{(0.001)} & 0.463\textsubscript{(0.004)} & 0.284\textsubscript{(0.011)} & 0.190\textsubscript{(0.005)} \\
$\tau=0.6$ & \cellcolor{gray!20} 0.016\textsubscript{(0.002)} & 0.092\textsubscript{(0.003)} & 0.109\textsubscript{(0.001)} & 0.094\textsubscript{(0.001)} & 0.354\textsubscript{(0.003)} & 0.217\textsubscript{(0.006)} & 0.148\textsubscript{(0.004)} \\
\bottomrule
\end{tabular}
}
\caption{\small Misclassification rates under misspecified discriminant structures. Upper panel: Tucker low-rank discriminant with varied off-diagonal energy ratio $\alpha$ (larger $\alpha$ means greater deviation from CP). Lower panel: CP discriminant tensor with sparse perturbations of varied magnitude $\tau$ at sparsity level $\rho=0.01$. }
\label{tab:robustness}
\end{table}

\begin{table}[ht]
\centering
\resizebox{0.9\textwidth}{!}{
\begin{tabular}{lccccccc}
\toprule
Setting & \cellcolor{gray!20} CP-TDA & TuckerLDA & CATCH & GuoSTM & LaiSTDA & ZhouLogitCP & DATED \\
\midrule
$\alpha=0.3$   & \cellcolor{gray!20} 0.895\textsubscript{(0.015)} & 1.006\textsubscript{(0.015)} & 1.003\textsubscript{(0.000)} & NA & NA & 2.884\textsubscript{(0.092)} & 2.042\textsubscript{(0.003)} \\
$\alpha=0.5$   & \cellcolor{gray!20} 0.798\textsubscript{(0.013)} & 0.892\textsubscript{(0.013)} & 1.003\textsubscript{(0.000)} & NA & NA & 2.665\textsubscript{(0.068)} & 1.932\textsubscript{(0.003)} \\
$\alpha=0.7$   & \cellcolor{gray!20} 0.763\textsubscript{(0.013)} & 0.764\textsubscript{(0.012)} & 1.002\textsubscript{(0.000)} & NA & NA & 2.464\textsubscript{(0.088)} & 1.812\textsubscript{(0.003)} \\
\midrule
$\tau=0.2$   & \cellcolor{gray!20} 0.693\textsubscript{(0.014)} & 0.916\textsubscript{(0.014)} & 1.003\textsubscript{(0.000)} & NA & NA & 2.424\textsubscript{(0.059)} & 1.925\textsubscript{(0.003)} \\
$\tau=0.4$   & \cellcolor{gray!20} 0.671\textsubscript{(0.011)} & 0.889\textsubscript{(0.012)} & 1.002\textsubscript{(0.000)} & NA & NA & 2.283\textsubscript{(0.056)} & 1.846\textsubscript{(0.002)} \\
$\tau=0.6$   & \cellcolor{gray!20} 0.622\textsubscript{(0.011)} & 0.811\textsubscript{(0.012)} & 1.001\textsubscript{(0.000)} & NA & NA & 2.083\textsubscript{(0.050)} & 1.694\textsubscript{(0.002)} \\
\bottomrule
\end{tabular}
}
\caption{\small Discriminant tensor estimation errors ($\|\widehat{\mathcal{B}} - \mathcal{B}\|_F/\|\mathcal{B}\|_F$) under misspecified discriminant structures. Settings follow Table~\ref{tab:robustness}.}
\label{tab:robustness_Berror}
\end{table}

\section{Real Data Analysis}
\label{sec:appl}

We evaluate the proposed Semiparametric Tensor Discriminant Networks (STDN) on graph classification tasks from the widely used D\&D and PROTEINS datasets from the \href{https://chrsmrrs.github.io/datasets/}{TUDataset} benchmark. These datasets are high-dimensional, structurally rich, and sample-limited, which is precisely the regime where statistical efficiency and principled discriminant modeling are essential.

Our framework employs a tensor neural network encoder to generate multiway graph representations, a tensor Gaussianizing flow that maps these representations into a TGMM-aligned latent space, and a CP-TDA head for final classification. This architecture combines the expressive capacity of learned representations with the parsimony, interpretability, and statistical optimality of CP-TDA, and is encoder-agnostic (any network producing tensor-valued features can be used).

For our experiments, we adopt the state-of-the-art encoder from TTG-NN \citep{wen2024tensorview}, which integrates persistence-image topology, graph convolutions, and tensor low-rank transformation layers. To obtain tensor-valued features, we truncate the encoder before vectorization and replace its final vector attention module with a lightweight multiway attention layer, yielding features of shape $(2H \times H \times H)$. These tensor features are fed into the FlowTGMM + CP-TDA head, and the encoder and flow are trained end-to-end. We set $H=32$, producing tensor features of dimension $64 \times 32 \times 32$. Training details are provided in Appendix \ref{subsec:realdata}. All results follow the evaluation protocol of TTG-NN and \cite{xu2018powerful}, using 10-fold cross-validation. 


\noindent
\underline{\textbf{Performance relative to neural network baselines.}} 
Table~\ref{tab:graph_network} compares STDN against 14 state-of-the-art graph neural network classifiers; additional kernel method comparisons are in Appendix~\ref{subsec:realdata}. STDN achieves the highest accuracy on both D\&D and PROTEINS, outperforming all baselines including recent state-of-the-art models. These gains arise because deep neural networks typically employ dense linear classifier heads that do not exploit multiway structure, suffer from curse of dimensionality (akin to nonparametric regression) in small samples, and lack statistical guarantees. In contrast, our CP-TDA head imposes a parsimonious multilinear discriminant model, yields an optimal decision rule, and provides a theoretically grounded regularizations. The empirical improvements confirm that combining deep tensor representations with statistically principled discriminant analysis is materially more effective than purely neural network classifiers.

\begin{table}[ht]
\centering
\resizebox{0.8\textwidth}{!}{
\begin{tabular}{lcc}
\toprule
\textbf{Model} & \textbf{DD (acc ± sd, \%)} & \textbf{PROTEINS (acc ± sd, \%)} \\
\midrule
GCN \citep{kipf2017semi} & 79.12 ± 3.07 & 70.31 ± 1.93 \\
ChebNet \citep{defferrard2016convolutional} & N/A & 75.50 ± 0.40 \\
GIN \citep{xu2018powerful} & 75.40 ± 2.60 & 76.16 ± 2.76 \\
DGCNN \citep{zhang2018end} & 79.37 ± 0.94 & 75.54 ± 0.94 \\
DiffPool \citep{ying2018hierarchical} & 77.90 ± 2.40 & 73.63 ± 3.60 \\
MinCutPool \citep{bianchi2020spectral} & 77.60 ± 3.10 & 76.52 ± 2.58 \\
EigenGCN \citep{ma2019graph} & 75.90 ± 3.90 & 74.10 ± 3.10 \\
SAGPool \citep{lee2019self} & 76.45 ± 0.97 & 71.86 ± 0.97 \\
HaarPool \citep{wang2020haar} & 77.40 ± 3.40 & 73.23 ± 2.51 \\
PersLay \citep{carriere2020perslay} & N/A & 74.80 ± 0.30 \\
FC-V \citep{obray2021filtration} & N/A & 74.54 ± 0.48 \\
SIN \citep{bodnar2021weisfeiler} & N/A & 76.50 ± 3.40 \\
TOGL \citep{horn2021topological} & 75.70 ± 2.10 & 76.00 ± 3.90 \\
TTG-NN \citep{wen2024tensorview} & 80.90 ± 2.57 & 77.62 ± 3.92 \\
\midrule
\cellcolor{gray!20} STDN & \cellcolor{gray!20} 81.86 ± 2.19 & \cellcolor{gray!20} 78.93 ± 2.07 \\
\bottomrule
\end{tabular}
}
\caption{\small Graph classification accuracy on D\&D and PROTEINS datasets (mean $\pm$ standard deviation). STDN (highlighted) combines a neural network encoder with CP-TDA head, achieving the highest accuracy.}
\label{tab:graph_network}
\end{table}

\noindent
\underline{\textbf{Performance relative to statistical tensor classifiers.}} 
To compare directly against statistical tensor classifiers, we extract tensor features from graphs via the trained encoder-flow pipeline and evaluate CATCH, DATED, GuoSTM, LaiSTDA, TuckerLDA, and ZhouLogitCP under identical 10-fold cross-validation. Table~\ref{tab:real_data} shows STDN achieves the best performance on both datasets. Competing methods rely on sparsity or Tucker structures, or lack direct discriminant modeling, failing to capture the most informative multilinear discriminant directions in the encoded features. By constraining the discriminant tensor to CP low-rank form, STDN recovers the dominant components of the decision boundary more accurately, yielding substantially lower misclassification error.

\begin{table}[ht]
\centering
\resizebox{0.75\textwidth}{!}{
\begin{tabular}{lcc}
\toprule
\textbf{Method} & \textbf{DD (acc ± sd, \%)} & \textbf{PROTEINS (acc ± sd, \%)} \\
\midrule
CATCH \citep{pan2019covariate} & 77.31 ± 4.04 & 76.25 ± 1.94 \\
DATED \citep{wang2024parsimonious} & 78.40 ± 3.72 & 75.18 ± 1.25 \\
GuoSTM \citep{guoxian2016} & 77.14 ± 4.35 & 74.29 ± 2.52 \\
LaiSTDA \citep{lai2013sparse} & 76.30 ± 3.75 & 74.29 ± 1.82 \\
TuckerLDA & 78.49 ± 3.48 & 74.82 ± 1.01 \\
ZhouLogitCP \citep{zhou2013tensor} & 76.97 ± 3.72 & 76.07 ± 1.97 \\
\midrule
\cellcolor{gray!20} STDN & \cellcolor{gray!20} 81.86 ± 2.19 & \cellcolor{gray!20} 78.93 ± 2.07 \\
\bottomrule
\end{tabular}
}
\caption{\small Comparison with statistical tensor classifiers on D\&D and PROTEINS datasets (mean $\pm$ standard deviation). All methods use identical tensor features from the encoder-flow pipeline. STDN (highlighted) with CP-TDA head achieves the highest accuracy.}
\label{tab:real_data}
\end{table}

Additionally, the learned tensor representations exhibit numerically low multilinear rank. While not conclusive, this empirical pattern suggests that CP low-rank discriminant structure may be present in real data, providing practical justification for our modeling assumptions.

Across neural network and statistical baselines, STDN consistently achieves the best predictions. These results demonstrate that statistical discriminant structure and deep representation learning are mutually reinforcing: the encoder-flow pipeline provides rich tensor features, and CP-TDA extracts essential low-rank discriminant directions with theoretical guarantees and superior empirical accuracy.

\section{Conclusion}\label{sec:conclude}

This paper develops a new framework for high-dimensional tensor discriminant analysis that addresses key limitations of existing tensor classification approaches. We introduce CP low-rank structure for the discriminant tensor, a modeling perspective not previously explored, and develop CP-TDA algorithms with \textsc{rc-PCA} initialization that is essential for entering the global contraction region. Unlike prior CP methods, \textsc{rc-PCA} remains stable under weaker signal strength and with non-orthogonal CP bases, enabling global convergence of DISTIP-CP in regimes where standard initialization fails. Our theoretical analysis provides sharp perturbation bounds under correlated noise and establishes the first minimax-optimal misclassification rates for tensor discriminant analysis.

Beyond the correctly specified TGMM setting, we develop a semiparametric extension using learned tensor representations (STDN) and show, both theoretically and empirically, that CP low-rank discriminant analysis performs strongly even when the true model deviates from CP low-rankness. These results collectively demonstrate that statistical structure (here, CP-based discriminant modeling) can complement and strengthen modern representation learning, yielding interpretable, sample efficient, and robust tensor classifiers. Future work includes extending these ideas to broader tensor learning tasks, such as tensor regression and tensor clustering.

\section*{Data availability statement}

The data that support the findings of this study are openly available in TUDataset at \url{https://chrsmrrs.github.io/datasets/}.


%
%

\ \\
\spacingset{1.18} 
\bibliographystyle{apalike} 
\bibliography{\mybib}

%
%
\clearpage
\setcounter{page}{1}
\begin{appendices}
    \begin{center}
        {\Large Supplementary Material of ``\TITLE''}

    \end{center}
    

\section{Additional Simulation Results}
\label{append:simulation}

\subsection{Additional results for \textsc{rc-PCA} initialization and performance under CP-structured discriminant model}
\label{subsec:others}
This subsection complements Section \ref{sec:simu} by examining \textsc{rc-PCA}'s benefits and CP-TDA's performance across different tensor orders and basis orthogonality conditions. 

\noindent\underline{\textbf{Setup.}} For order-3 tensors, we adopt the configuration from Section \ref{sec:simu} with dimensions $(30, 30, 30)$, sample sizes $n_1 = n_2 = 100$, and CP rank $R = 5$, but use orthogonal CP bases instead of non-orthogonal ones. For order-4 tensors, we set dimensions $d_1 = d_2 = d_3 = d_4 = 20$, sample sizes $n_1 = n_2 = 150$, and CP rank $R = 5$, using identity covariance matrices $\Sigma_m$ for $m \in [4]$. Signal strengths follow two regimes: (i) small spectral separation with equal signal strength $w_r \in \{2.5, 3.0, 3.5\}$; and (ii) large spectral separation where $w_{r}/w_{r+1} = 1.25$ with $w_{\max} \in \{3, 4, 5\}$.

\begin{figure}[htbp]
    \centering
    \resizebox{0.7\textwidth}{!}{
    \begin{subfigure}[b]{0.425\textwidth}
        \centering
        \includegraphics[width=\textwidth]{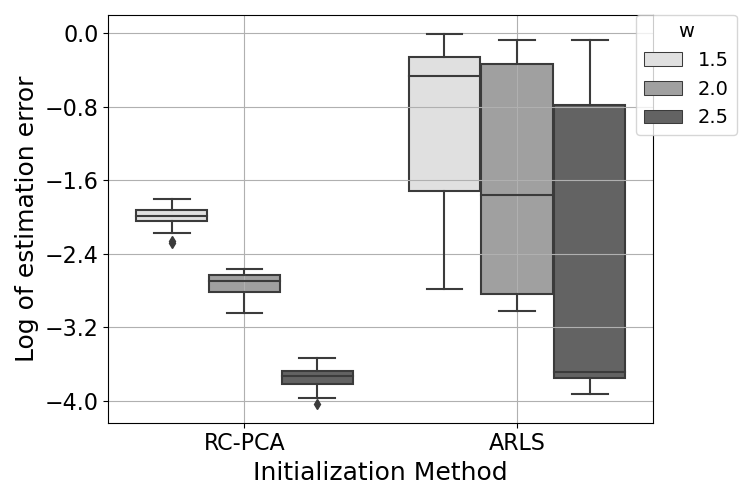}
        \caption{Small spectral separation}
        \label{fig:2a}
    \end{subfigure}
    \hspace{0.02\textwidth}
    \begin{subfigure}[b]{0.425\textwidth}
        \centering
        \includegraphics[width=\textwidth]{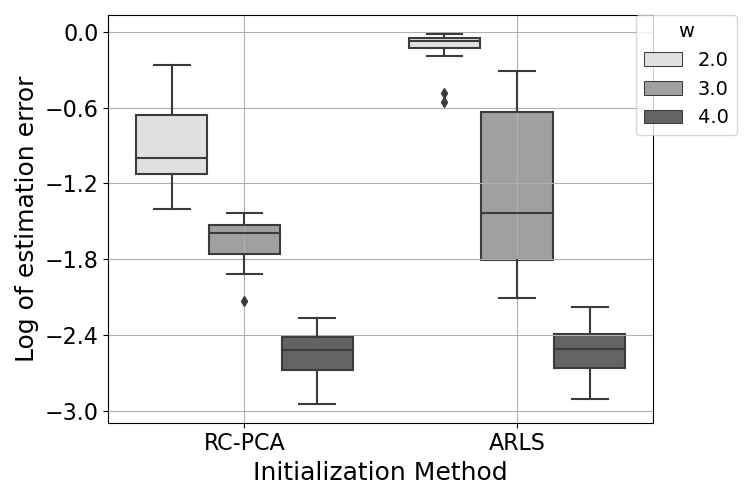}
        \caption{Large spectral separation}
        \label{fig:2b}
    \end{subfigure}
    }
   \caption{\small Logarithmic CP basis estimation error for \textsc{rc-PCA} and ARLS with orthogonal bases on order-3 tensors under (a) small spectral separation (equal $w_r$) and (b) large spectral separation (geometric decay with $w_{r}/w_{r+1} = 1.25$).}
    \label{fig:2}
\end{figure}

\begin{table}[htbp]
\centering
\resizebox{0.75\textwidth}{!}{
\begin{tabular}{cccccccc}
\toprule
\multirow{2}{*}{Metric} & \multirow{2}{*}{Algorithms} & \multicolumn{3}{c}{Small spectral separation} & \multicolumn{3}{c}{Large spectral separation} \\
\cmidrule(lr){3-5} \cmidrule(lr){6-8}
& & $w = 1.5$ & $w = 2.0$ & $w = 2.5$ & $w_{max} = 2$ & $w_{max} = 3$ & $w_{max} = 4$ \\
\midrule
\multirow{6}{*}{EstError} 
& Sample $\mathcal{B}$ & 5.10\textsubscript{(0.19)} & 4.89\textsubscript{(0.15)} & 3.86\textsubscript{(0.10)} & 5.47\textsubscript{(0.15)} & 4.58\textsubscript{(0.15)} & 3.39\textsubscript{(0.12)} \\
& \cellcolor{gray!20} CP-TDA & \cellcolor{gray!20} 0.93\textsubscript{(0.02)} & \cellcolor{gray!20} 0.81\textsubscript{(0.02)} & \cellcolor{gray!20} 0.60\textsubscript{(0.02)} & \cellcolor{gray!20} 1.04\textsubscript{(0.02)} & \cellcolor{gray!20} 0.80\textsubscript{(0.01)} & \cellcolor{gray!20} 0.58\textsubscript{(0.01)} \\
& TuckerLDA & 1.38\textsubscript{(0.02)} & 1.22\textsubscript{(0.02)} & 0.85\textsubscript{(0.02)} & 1.39\textsubscript{(0.02)} & 1.10\textsubscript{(0.01)} & 0.81\textsubscript{(0.01)} \\
& CATCH & 1.01\textsubscript{(0.00)} & 1.00\textsubscript{(0.00)} & 1.00\textsubscript{(0.00)} & 1.01\textsubscript{(0.00)} & 1.00\textsubscript{(0.00)} & 1.00\textsubscript{(0.00)} \\
& DATED & 2.28\textsubscript{(0.01)} & 2.08\textsubscript{(0.01)} & 1.75\textsubscript{(0.01)} & 2.33\textsubscript{(0.01)} & 2.00\textsubscript{(0.01)} & 1.60\textsubscript{(0.00)} \\
& ZhouLogitCP & 2.14\textsubscript{(0.09)} & 2.09\textsubscript{(0.06)} & 1.63\textsubscript{(0.04)} & 2.24\textsubscript{(0.10)} & 1.74\textsubscript{(0.07)} & 1.28\textsubscript{(0.04)} \\
\midrule
\multirow{8}{*}{Misclass} 
& Sample $\mathcal{B}$ & 0.25\textsubscript{(0.02)} & 0.12\textsubscript{(0.01)} & 0.06\textsubscript{(0.01)} & 0.28\textsubscript{(0.02)} & 0.17\textsubscript{(0.02)} & 0.07\textsubscript{(0.01)} \\
& \cellcolor{gray!20}  CP-TDA & \cellcolor{gray!20} 0.09\textsubscript{(0.00)} & \cellcolor{gray!20} 0.04\textsubscript{(0.00)} & \cellcolor{gray!20} 0.00\textsubscript{(0.00)} & \cellcolor{gray!20} 0.11\textsubscript{(0.00)} & \cellcolor{gray!20} 0.03\textsubscript{(0.00)} & \cellcolor{gray!20} 0.00\textsubscript{(0.00)} \\
& TuckerLDA & 0.19\textsubscript{(0.01)} & 0.11\textsubscript{(0.01)} & 0.01\textsubscript{(0.00)} & 0.16\textsubscript{(0.01)} & 0.06\textsubscript{(0.00)} & \cellcolor{gray!20} 0.00\textsubscript{(0.00)} \\
& CATCH & 0.33\textsubscript{(0.00)} & 0.29\textsubscript{(0.00)} & 0.26\textsubscript{(0.00)} & 0.34\textsubscript{(0.00)} & 0.29\textsubscript{(0.00)} & 0.22\textsubscript{(0.00)} \\
& DATED & 0.31\textsubscript{(0.00)} & 0.28\textsubscript{(0.00)} & 0.24\textsubscript{(0.00)} & 0.32\textsubscript{(0.00)} & 0.29\textsubscript{(0.00)} & 0.23\textsubscript{(0.01)} \\
& GuoSTM & 0.28\textsubscript{(0.00)} & 0.25\textsubscript{(0.00)} & 0.21\textsubscript{(0.00)} & 0.28\textsubscript{(0.00)} & 0.26\textsubscript{(0.00)} & 0.22\textsubscript{(0.00)} \\
& LaiSTDA & 0.35\textsubscript{(0.00)} & 0.32\textsubscript{(0.00)} & 0.28\textsubscript{(0.00)} & 0.36\textsubscript{(0.00)} & 0.33\textsubscript{(0.00)} & 0.29\textsubscript{(0.00)} \\
& ZhouLogitCP & 0.28\textsubscript{(0.00)} & 0.25\textsubscript{(0.00)} & 0.21\textsubscript{(0.01)} & 0.28\textsubscript{(0.01)} & 0.25\textsubscript{(0.01)} & 0.18\textsubscript{(0.01)} \\
\bottomrule
\end{tabular}
}
\caption{\small Discriminant tensor estimation error (EstError, $\|\widehat{\mathcal{B}} - \mathcal{B}\|_F/\|\mathcal{B}\|_F$) and misclassification rates (Misclass) for binary classification with orthogonal CP bases on order-3 tensors, under small spectral separation and large spectral separation. Best performance are highlighted in gray.}
\label{tab:3-4}
\end{table}

\noindent\underline{\textbf{Implementation details.}} 
We compare CP-TDA with several established tensor-based classifiers. 
As a baseline without structural constraints, we include the sample discriminant tensor obtained from the plug-in estimator 
$\widehat{\calB} = (\overline{\cX}^{(2)} - \overline{\cX}^{(1)}) \times_{m=1}^{M} \widehat{\Sigma}_m^{-1}$ 
in equation \eqref{eqn:lda-discrim-tensor}. TuckerLDA applies Tucker decomposition 
$\cB = \mathcal{F} \times_1 \bU_1 \times_2 \cdots \times_M \bU_M$ to the sample discriminant tensor with Tucker ranks $(5,5,5)$ for order-3 tensors and $(5,5,5,5)$ for order-4 tensors. 

We evaluate CATCH (Covariate-Adjusted Tensor Classification in High-dimensions; \cite{pan2019covariate}), implemented via the authors’ publicly available R package; DATED \citep{wang2024parsimonious}, which uses tensor envelope structure for dimensionality reduction with mode-wise dimensions $(5,5,5)$ for order-3 tensors and $(5,5,5,5)$ for order-4 tensors; GuoSTM \citep{guoxian2016}, a support tensor machine using multilinear principal component analysis (MPCA) with output dimensions $(5,5,5)$ for order-3 tensors and $(5,5,5,5)$ for order-4 tensors; and LaiSTDA \citep{lai2013sparse}, which imposes sparsity on discriminant directions with the same mode-wise dimensions. We also compare with ZhouLogitCP \citep{zhou2013tensor}, a tensor generalized linear model that fits logistic regression with CP low-rank coefficients, fitted with ridge penalties $\rho \in \{0.001, 0.005, 0.01, 0.05, 0.1, 0.5\}$ at CP rank $R = 5$ and selected via BIC. 

\begin{figure}[htbp!]
    \centering
    \resizebox{0.7\textwidth}{!}{
    \begin{tabular}{cc}
        \begin{subfigure}[b]{0.475\textwidth}
            \centering
            \includegraphics[width=\textwidth]{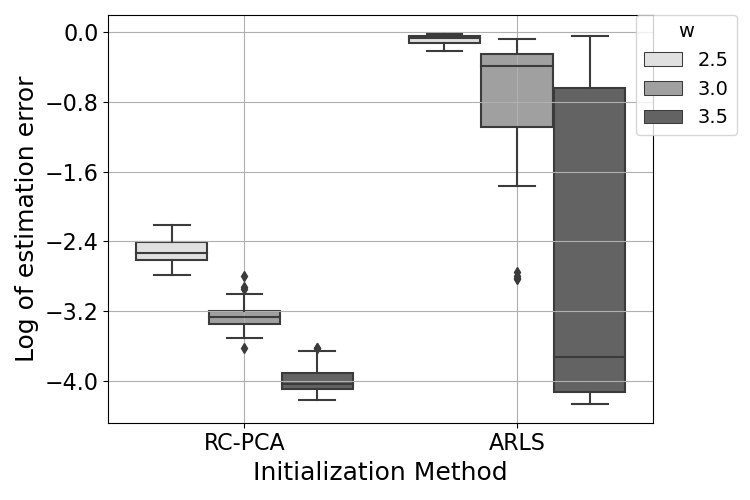}
            \caption{Small spectral separation and orthogonal bases}
            \label{fig:3a}
        \end{subfigure}
        &
        \begin{subfigure}[b]{0.475\textwidth}
            \centering
            \includegraphics[width=\textwidth]{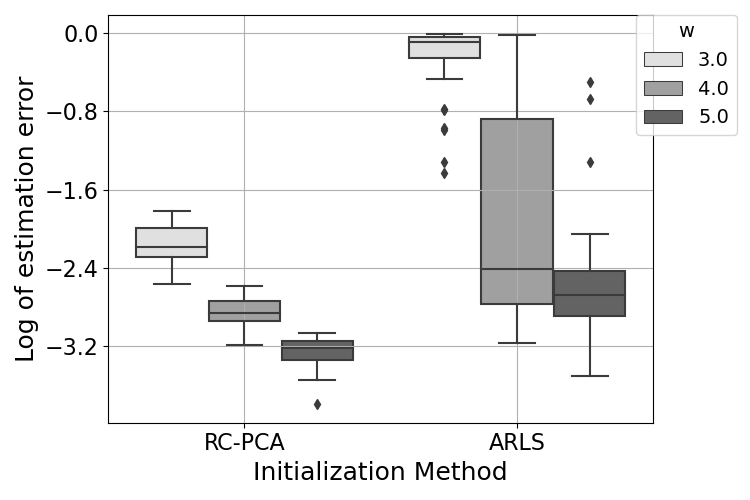}
            \caption{Large spectral separation and orthogonal bases}
            \label{fig:3b}
        \end{subfigure}
        \\[1em]
        \begin{subfigure}[b]{0.475\textwidth}
            \centering
            \includegraphics[width=\textwidth]{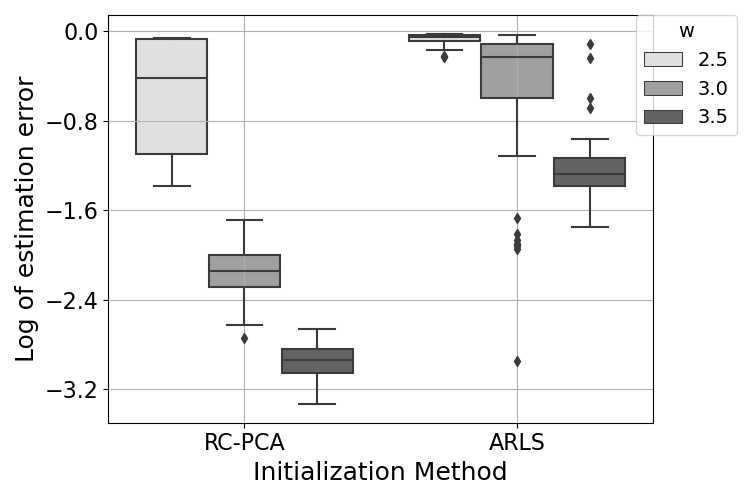}
            \caption{Small spectral separation and non-orthogonal bases}
            \label{fig:4a}
        \end{subfigure}
        &
        \begin{subfigure}[b]{0.475\textwidth}
            \centering
            \includegraphics[width=\textwidth]{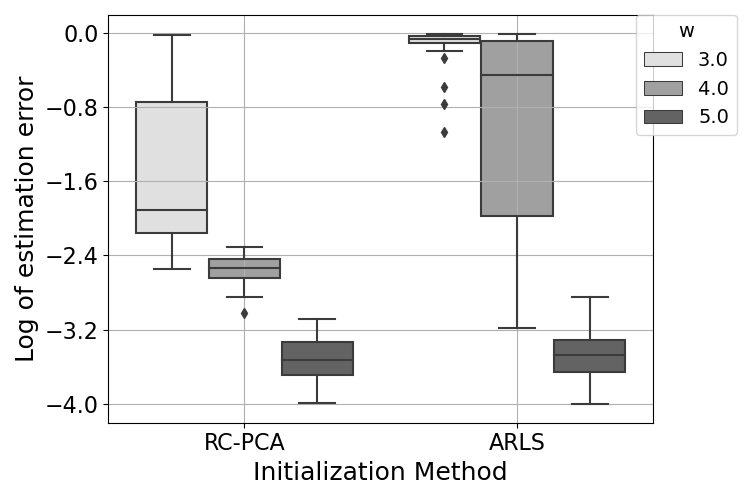}
            \caption{Large spectral separation and non-orthogonal bases}
            \label{fig:4b}
        \end{subfigure}
    \end{tabular}
    }
    \caption{\small Logarithmic CP basis estimation error for \textsc{rc-PCA} and ARLS with orthogonal and non-orthogonal bases on order-4 tensors under (a,c) small spectral separation (equal $w_r$) and (b,d) large spectral separation (geometric decay with $w_{r}/w_{r+1} = 1.25$).}
    \label{fig:3}
\end{figure}

\noindent\underline{\textbf{Results.}} Figures \ref{fig:2}-\ref{fig:3} and Tables \ref{tab:5-6}-\ref{tab:7-8} present results across different tensor orders and basis orthogonality conditions. \textsc{rc-PCA} consistently provides reliable initialization across all settings, while ARLS fails under weak signal conditions (Figures \ref{fig:2}-\ref{fig:3}). Performance comparisons show that CP-TDA achieves substantially lower estimation and classification errors than competing methods across all configurations (Tables \ref{tab:5-6}-\ref{tab:7-8}). Results for order-3 tensors with orthogonal bases and order-4 tensors with both orthogonal and non-orthogonal bases are consistent with the order-3 non-orthogonal setting in Section \ref{sec:simu}, confirming CP-TDA's robustness across tensor orders and basis types.

\begin{table}[htbp!]
\centering
\resizebox{0.75\textwidth}{!}{
\begin{tabular}{cccccccc}
\toprule
\multirow{2}{*}{Metric} & \multirow{2}{*}{Algorithms} & \multicolumn{3}{c}{Small spectral separation} & \multicolumn{3}{c}{Large spectral separation} \\
\cmidrule(lr){3-5} \cmidrule(lr){6-8}
& & $w = 2.5$ & $w = 3.0$ & $w = 3.5$ & $w_{max} = 3$ & $w_{max} = 4$ & $w_{max} = 5$ \\
\midrule
\multirow{6}{*}{EstError} 
& Sample $\mathcal{B}$ & 5.92\textsubscript{(0.54)} & 4.48\textsubscript{(0.37)} & 2.88\textsubscript{(0.24)} & 5.37\textsubscript{(0.29)} & 4.33\textsubscript{(0.28)} & 2.57\textsubscript{(0.24)} \\
& \cellcolor{gray!20} CP-TDA & \cellcolor{gray!20} 0.95\textsubscript{(0.02)} & \cellcolor{gray!20} 0.86\textsubscript{(0.03)} & \cellcolor{gray!20} 0.57\textsubscript{(0.03)} & \cellcolor{gray!20} 0.96\textsubscript{(0.05)} & \cellcolor{gray!20} 0.75\textsubscript{(0.02)} & \cellcolor{gray!20} 0.59\textsubscript{(0.02)} \\
& TuckerLDA & 1.21\textsubscript{(0.07)} & 1.04\textsubscript{(0.05)} & 0.91\textsubscript{(0.02)} & 1.24\textsubscript{(0.07)} & 1.05\textsubscript{(0.03)} & 0.81\textsubscript{(0.03)} \\
& CATCH & 1.11\textsubscript{(0.00)} & 1.11\textsubscript{(0.00)} & 1.11\textsubscript{(0.00)} & 1.11\textsubscript{(0.00)} & 1.11\textsubscript{(0.00)} & 1.11\textsubscript{(0.00)} \\
& DATED & 1.34\textsubscript{(0.07)} & 1.28\textsubscript{(0.04)} & 1.23\textsubscript{(0.02)} & 1.42\textsubscript{(0.07)} & 1.30\textsubscript{(0.03)} & 1.23\textsubscript{(0.03)} \\
& ZhouLogitCP & 3.91\textsubscript{(0.14)} & 3.39\textsubscript{(0.10)} & 2.65\textsubscript{(0.08)} & 3.78\textsubscript{(0.15)} & 3.32\textsubscript{(0.09)} & 2.16\textsubscript{(0.08)} \\
\midrule
\multirow{8}{*}{Misclass} 
& Sample $\mathcal{B}$ & 0.36\textsubscript{(0.02)} & 0.21\textsubscript{(0.01)} & 0.04\textsubscript{(0.01)} & 0.30\textsubscript{(0.02)} & 0.19\textsubscript{(0.01)} & 0.01\textsubscript{(0.00)} \\
& \cellcolor{gray!20} CP-TDA & \cellcolor{gray!20} 0.06\textsubscript{(0.01)} & \cellcolor{gray!20} 0.02\textsubscript{(0.00)} & \cellcolor{gray!20} 0.00\textsubscript{(0.00)} & \cellcolor{gray!20} 0.07\textsubscript{(0.01)} & \cellcolor{gray!20} 0.00\textsubscript{(0.00)} & \cellcolor{gray!20} 0.00\textsubscript{(0.00)} \\
& TuckerLDA & 0.14\textsubscript{(0.03)} & 0.08\textsubscript{(0.01)} & \cellcolor{gray!20} 0.00\textsubscript{(0.00)} & 0.18\textsubscript{(0.02)} & 0.04\textsubscript{(0.01)} & \cellcolor{gray!20} 0.00\textsubscript{(0.00)} \\
& CATCH & 0.38\textsubscript{(0.01)} & 0.34\textsubscript{(0.01)} & 0.29\textsubscript{(0.01)} & 0.39\textsubscript{(0.00)} & 0.36\textsubscript{(0.01)} & 0.31\textsubscript{(0.01)} \\
& DATED & 0.39\textsubscript{(0.01)} & 0.36\textsubscript{(0.01)} & 0.32\textsubscript{(0.01)} & 0.40\textsubscript{(0.01)} & 0.38\textsubscript{(0.01)} & 0.34\textsubscript{(0.01)} \\
& GuoSTM & 0.40\textsubscript{(0.01)} & 0.37\textsubscript{(0.01)} & 0.33\textsubscript{(0.01)} & 0.41\textsubscript{(0.00)} & 0.38\textsubscript{(0.01)} & 0.35\textsubscript{(0.01)} \\
& LaiSTDA & 0.41\textsubscript{(0.01)} & 0.38\textsubscript{(0.01)} & 0.34\textsubscript{(0.01)} & 0.42\textsubscript{(0.01)} & 0.39\textsubscript{(0.01)} & 0.36\textsubscript{(0.01)} \\
& ZhouLogitCP & 0.47\textsubscript{(0.01)} & 0.43\textsubscript{(0.01)} & 0.38\textsubscript{(0.02)} & 0.46\textsubscript{(0.01)} & 0.41\textsubscript{(0.02)} & 0.35\textsubscript{(0.01)} \\
\bottomrule
\end{tabular}
}
\caption{\small Discriminant tensor estimation error (EstError, $\|\widehat{\mathcal{B}} - \mathcal{B}\|_F/\|\mathcal{B}\|_F$) and misclassification rates (Misclass) for binary classification with orthogonal CP bases on order-4 tensors, under small spectral separation and large spectral separation. Best performance are highlighted in gray.}
\label{tab:5-6}
\end{table}

\begin{table}[htbp!]
\centering
\resizebox{0.75\textwidth}{!}{
\begin{tabular}{cccccccc}
\toprule
\multirow{2}{*}{Metric} & \multirow{2}{*}{Algorithms} & \multicolumn{3}{c}{Small spectral separation} & \multicolumn{3}{c}{Large Signal Strengths} \\
\cmidrule(lr){3-5} \cmidrule(lr){6-8}
& & $w = 2.5$ & $w = 3.0$ & $w = 3.5$ & $w_{max} = 3$ & $w_{max} = 4$ & $w_{max} = 5$ \\
\midrule
\multirow{6}{*}{EstError} 
& Sample $\mathcal{B}$ & 5.94\textsubscript{(0.42)} & 4.24\textsubscript{(0.27)} & 2.94\textsubscript{(0.19)} & 4.77\textsubscript{(0.24)} & 3.76\textsubscript{(0.21)} & 2.36\textsubscript{(0.15)} \\
& \cellcolor{gray!20} CP-TDA & \cellcolor{gray!20} 1.04\textsubscript{(0.02)} & \cellcolor{gray!20} 0.94\textsubscript{(0.03)} & \cellcolor{gray!20} 0.62\textsubscript{(0.03)} & \cellcolor{gray!20} 1.05\textsubscript{(0.05)} & \cellcolor{gray!20} 0.82\textsubscript{(0.02)} & \cellcolor{gray!20} 0.64\textsubscript{(0.02)} \\
& TuckerLDA & 1.32\textsubscript{(0.07)} & 1.13\textsubscript{(0.05)} & 0.99\textsubscript{(0.02)} & 1.35\textsubscript{(0.07)} & 1.15\textsubscript{(0.03)} & 0.88\textsubscript{(0.03)} \\
& CATCH & 1.12\textsubscript{(0.00)} & 1.11\textsubscript{(0.00)} & 1.11\textsubscript{(0.00)} & 1.12\textsubscript{(0.00)} & 1.11\textsubscript{(0.00)} & 1.11\textsubscript{(0.00)} \\
& DATED & 1.55\textsubscript{(0.06)} & 1.48\textsubscript{(0.03)} & 1.45\textsubscript{(0.02)} & 1.65\textsubscript{(0.07)} & 1.51\textsubscript{(0.03)} & 1.45\textsubscript{(0.02)} \\
& ZhouLogitCP & 4.32\textsubscript{(0.16)} & 3.76\textsubscript{(0.12)} & 2.96\textsubscript{(0.10)} & 4.18\textsubscript{(0.18)} & 3.69\textsubscript{(0.10)} & 2.33\textsubscript{(0.08)} \\
\midrule
\multirow{8}{*}{Misclass} 
& Sample $\mathcal{B}$ & 0.38\textsubscript{(0.02)} & 0.23\textsubscript{(0.01)} & 0.05\textsubscript{(0.01)} & 0.32\textsubscript{(0.02)} & 0.20\textsubscript{(0.01)} & 0.01\textsubscript{(0.00)} \\
& \cellcolor{gray!20} CP-TDA & \cellcolor{gray!20} 0.08\textsubscript{(0.01)} & \cellcolor{gray!20} 0.00\textsubscript{(0.00)} & \cellcolor{gray!20} 0.00\textsubscript{(0.00)} & \cellcolor{gray!20} 0.09\textsubscript{(0.02)} & \cellcolor{gray!20} 0.02\textsubscript{(0.00)} & \cellcolor{gray!20} 0.00\textsubscript{(0.00)} \\
& TuckerLDA & 0.16\textsubscript{(0.01)} & 0.03\textsubscript{(0.01)} & \cellcolor{gray!20} 0.00\textsubscript{(0.00)} & 0.19\textsubscript{(0.01)} & 0.04\textsubscript{(0.00)} & \cellcolor{gray!20} 0.00\textsubscript{(0.00)} \\
& CATCH & 0.39\textsubscript{(0.01)} & 0.35\textsubscript{(0.01)} & 0.30\textsubscript{(0.01)} & 0.40\textsubscript{(0.01)} & 0.37\textsubscript{(0.01)} & 0.32\textsubscript{(0.01)} \\
& DATED & 0.40\textsubscript{(0.01)} & 0.37\textsubscript{(0.01)} & 0.33\textsubscript{(0.00)} & 0.41\textsubscript{(0.01)} & 0.38\textsubscript{(0.01)} & 0.35\textsubscript{(0.01)} \\
& GuoSTM & 0.41\textsubscript{(0.01)} & 0.38\textsubscript{(0.01)} & 0.34\textsubscript{(0.01)} & 0.42\textsubscript{(0.01)} & 0.39\textsubscript{(0.01)} & 0.36\textsubscript{(0.01)} \\
& LaiSTDA & 0.42\textsubscript{(0.01)} & 0.39\textsubscript{(0.01)} & 0.35\textsubscript{(0.01)} & 0.43\textsubscript{(0.01)} & 0.40\textsubscript{(0.01)} & 0.37\textsubscript{(0.01)} \\
& ZhouLogitCP & 0.48\textsubscript{(0.01)} & 0.44\textsubscript{(0.01)} & 0.38\textsubscript{(0.02)} & 0.47\textsubscript{(0.01)} & 0.42\textsubscript{(0.02)} & 0.36\textsubscript{(0.01)} \\
\bottomrule
\end{tabular}
}
\caption{\small Discriminant tensor estimation error (EstError, $\|\widehat{\mathcal{B}} - \mathcal{B}\|_F/\|\mathcal{B}\|_F$) and misclassification rates (Misclass) for binary classification with non-orthogonal CP bases on order-4 tensors, under small spectral separation and large spectral separation. Best performance are highlighted in gray.}
\label{tab:7-8}
\end{table}

\subsection{Robustness to deviations from a CP discriminant tensor}\label{subsec:structureRobust}
\noindent\underline{\textbf{Setup.}} We examine CP-TDA's robustness when the true discriminant tensor 
deviates from the ideal CP low-rank structure, using $M = 3$ with dimensions $(30, 30, 30)$, 
sample sizes $n_1 = n_2 = 100$, and 1000 test samples. Mode-wise covariances 
$\bSigma = [\Sigma_m]_{m=1}^3$ have unit diagonals with off-diagonal elements 
$3/d_m$. Class means are $\cM_1 = 0$ and $\cM_2 = \calB \times_{m=1}^3 \Sigma_m$. 
Two misspecification mechanisms of $\calB$ are considered:

\noindent \textbf{1. Tucker low-rank with small off-diagonals.} Fix Tucker ranks $r_1 = r_2 = r_3 = 5$ and form the Tucker core as
$$
\mathcal{G} = \mathcal{G}_{\text{diag}} + \alpha\mathcal{G}_{\text{off}}, \quad \mathcal{G}_{\text{diag}}(i, j, k) = 
\begin{cases}
2 & \text{if } i = j = k \\
0 & \text{otherwise}
\end{cases}, \quad \mathcal{G}_{\text{off}}(i, j, k) \sim 
\begin{cases}
0 & \text{if } i = j = k \\
\mathcal{N}(0,1) & \text{otherwise}
\end{cases}
$$
where $\mathcal{G}_{\text{diag}} \in \mathbb{R}^{5 \times 5 \times 5}$ is a super-diagonal core tensor and $\mathcal{G}_{\text{off}} \in \mathbb{R}^{5\times 5\times 5}$ is an off-diagonal perturbation tensor with non-diagonal entries drawn independently from a standard normal distribution. We normalize $\mathcal{G}_{\text{off}}$ by scaling it to satisfy $\|\mathcal{G}_{\text{off}}\|_F = \|\mathcal{G}_{\text{diag}}\|_F$. The off-diagonal energy ratio $\alpha \in \{0.3, 0.5, 0.7\}$ controls the magnitude of deviation from CP low-rank structure. Let $\{\bA_m \in \mathbb{R}^{30 \times 5}\}_{m=1}^3$ have orthonormal columns. The discriminant tensor is $\mathcal{B} = \mathcal{G} \times_1 \bA_1 \times_2 \bA_2 \times_3 \bA_3$, which reduces to CP low-rank tensor when $\alpha = 0$.

\noindent \textbf{2. CP plus sparse perturbation.} Let $\mathcal{B}_{\text{CP}} = \sum_{r=1}^{R} \lambda_r \, \ba_{1r} \otimes \ba_{2r} \otimes \ba_{3r}$ with $ R = 5, \; \lambda_r = 2$, and orthonormal bases $\{\ba_{mr}\}_{m,r}$. Draw a sparse perturbation $\mathcal{U}$ by selecting a random subset $\Omega \subset [d_1] \times [d_2] \times [d_3]$ of size $K = \lfloor \rho \, d_1 d_2 d_3 \rfloor$, where each index $(i, j, k)$ is sampled uniformly without replacement from all possible indices. Then construct
$$
\mathcal{U}(i, j, k) = 
\begin{cases}
\pm\gamma & \text{if } (i,j,k) \in \Omega \\
0 & \text{if } (i,j,k) \notin \Omega
\end{cases}
$$
where for each $(i,j,k) \in \Omega$, sign is chosen with equal probability, and $\gamma = \tau\|\mathcal{B}_{\text{CP}}\|_F/\sqrt{K}$, so that $\|\mathcal{U}\|_F = \tau\|\mathcal{B}_{\text{CP}}\|_F$. We then set $\mathcal{B} = \mathcal{B}_{\text{CP}} + \mathcal{U}$ with $\rho=0.01$ and $\tau \in \{0.2, 0.4, 0.6\}$.

\subsection{Non-Gaussian tensor predictors}\label{subsec:nonGaussian}
\noindent\underline{\textbf{Setup.}} This subsection assesses performance when tensor predictors follow non-Gaussian distributions, specifically tensor elliptical distributions with heavy tails. We consider two families:
(i) The tensor $t$-distribution, $\mathcal{T}t(\mathcal{M},\bSigma,\nu)$ (Definition 1 in \cite{wang2023regression}), where degrees of freedom $\nu$ govern tail heaviness, with smaller $\nu$ producing heavier tails. 
(ii) The tensor-variate (asymmetric) Laplace distribution, $\mathcal{TL}(\mathcal{M},\bSigma,\lambda)$ (Definition 2.3 in \cite{yurchenko2021}), where $\lambda$ is the rate parameter. Both are implemented via exact samplers taking a mean tensor $\mathcal{M}$ and mode-wise covariance matrices $\bSigma=\{\Sigma_m\}_{m=1}^M$; here we use identity covariances. 

We adopt the following setting: tensor order $M = 3$ with dimensions $(30, 30, 30)$, sample sizes $n_1 = n_2 = 200$, and the true discriminant tensor $\mathcal{B}$ has CP rank $R=5$ with orthonormal bases and equal weights $w_r=5$. Class means are $\mathcal{M}_1=\mathbf{0}$ and $\mathcal{M}_2=\mathcal{B}$. For the heavy-tail controls, we use $\nu\in\{2,5\}$ for the tensor-$t$ case and $\lambda\in\{2.0,4.0\}$ for the tensor Laplace case. Each configuration is repeated 100 times.

\noindent\underline{\textbf{Implementation details.}} Methods are evaluated with true ranks/dimensions: GuoSTM uses MPCA output dimensions $(5,5,5)$; LaiSTDA and DATED use mode-wise dimensions $(5,5,5)$; TuckerLDA uses Tucker ranks $(5,5,5)$; ZhouLogitCP is trained with ridge penalties $\rho\in\{0.001,0.005,0.01, \\ 0.05,0.1,0.5\}$ at rank $R=5$; and CP-TDA uses true rank $R=5$. 

\begin{table}[htbp!]
\centering
\resizebox{0.85\textwidth}{!}{
\begin{tabular}{lccccccc}
\toprule
Setting & CP-TDA & TuckerLDA & CATCH & GuoSTM & LaiSTDA & ZhouLogitCP & DATED \\
\midrule
$\nu=2$ 
& \cellcolor{gray!20} 0.770\textsubscript{(0.011)} & 0.894\textsubscript{(0.012)} & 0.999\textsubscript{(0.000)} & NA & NA & 1.289\textsubscript{(0.026)} & 0.9996\textsubscript{(0.0001)} \\
$\nu=5$ 
& \cellcolor{gray!20} 0.327\textsubscript{(0.009)} & 0.432\textsubscript{(0.008)} & 0.993\textsubscript{(0.000)} & NA & NA & 1.321\textsubscript{(0.025)} & 0.986\textsubscript{(0.001)} \\
\midrule
$\lambda=2$ 
& \cellcolor{gray!20} 0.409\textsubscript{(0.008)} & 0.513\textsubscript{(0.008)} & 0.994\textsubscript{(0.000)} & NA & NA & 1.329\textsubscript{(0.042)} & 0.990\textsubscript{(0.000)} \\
$\lambda=4$ 
& \cellcolor{gray!20} 0.653\textsubscript{(0.004)} & 0.757\textsubscript{(0.004)} & 0.997\textsubscript{(0.000)} & NA & NA & 1.292\textsubscript{(0.024)} & 0.998\textsubscript{(0.000)} \\
\bottomrule
\end{tabular}
}
\caption{\small Discriminant tensor estimation error ($\|\widehat{\mathcal{B}} - \mathcal{B}\|_F/\|\mathcal{B}\|_F$)  for the tensor-$t$ distribution (upper) and tensor-Laplace distribution (lower). Best performance are highlighted in gray.}
\label{tab:nonGaussian_Berror}
\end{table}

\begin{table}[htbp!]
\centering
\resizebox{0.85\textwidth}{!}{
\begin{tabular}{lccccccc}
\toprule
Setting & CP-TDA & TuckerLDA & CATCH & GuoSTM & LaiSTDA & ZhouLogitCP & DATED \\
\midrule
$\nu=2$   & \cellcolor{gray!20} 0.075\textsubscript{(0.019)} & 0.127\textsubscript{(0.023)} & 0.432\textsubscript{(0.006)} & 0.497\textsubscript{(0.002)} & 0.460\textsubscript{(0.004)} & 0.262\textsubscript{(0.009)} & 0.444\textsubscript{(0.009)} \\
$\nu=5$   & \cellcolor{gray!20} 0.001\textsubscript{(0.000)} & 0.002\textsubscript{(0.000)} & 0.239\textsubscript{(0.004)} & 0.453\textsubscript{(0.012)} & 0.448\textsubscript{(0.005)} & 0.227\textsubscript{(0.008)} & 0.205\textsubscript{(0.007)} \\
\midrule
$\lambda=2$ & \cellcolor{gray!20} 0.002\textsubscript{(0.000)} & 0.003\textsubscript{(0.000)} & 0.253\textsubscript{(0.003)} & 0.478\textsubscript{(0.007)} & 0.431\textsubscript{(0.004)} & 0.211\textsubscript{(0.010)} & 0.223\textsubscript{(0.005)} \\
$\lambda=4$ & \cellcolor{gray!20} 0.013\textsubscript{(0.001)} & 0.015\textsubscript{(0.001)} & 0.351\textsubscript{(0.003)} & 0.496\textsubscript{(0.004)} & 0.467\textsubscript{(0.004)} & 0.292\textsubscript{(0.008)} & 0.328\textsubscript{(0.006)} \\
\bottomrule
\end{tabular}
}
\caption{\small Misclassification rates for tensor-$t$ distribution (upper) and tensor-Laplace distribution (lower). Best performance are highlighted in gray.}
\label{tab:nonGaussian}
\end{table}

\noindent\underline{\textbf{Results.}} Tables \ref{tab:nonGaussian_Berror}-\ref{tab:nonGaussian} report discriminant tensor estimation errors and misclassification rates under tensor-$t$ and tensor-Laplace distributions. Theoretically, based on \cite{fang1990statistical}, when tensor predictors follow a tensor elliptical distribution with known class means $\mathcal{M}_k$ and common covariance matrices $\bSigma$, Fisher's rule remains optimal for Tensor LDA and achieves the minimax misclassification rate. Our simulation results align with this theory: CP-TDA exhibits pronounced robustness to heavy-tailed non-Gaussian settings, consistently achieving the lowest estimation errors and misclassification rates across both distribution families. Moreover, performance improves as distributions become lighter-tailed (larger $\nu$ or smaller $\lambda$), with CP-TDA retaining substantial advantages over all competing methods.

\subsection{Robustness to heteroscedastic mode-wise covariance matrices}
\noindent\underline{\textbf{Setup and implementation details.}} We study violations of the common mode-wise covariance assumption in CP-TDA using order $M = 3$ with dimensions $(30, 30, 30)$, true discriminant tensor $\mathcal{B}$ of CP rank $R = 5$ with orthonormal bases and equal weights $w_r=5$. Unlike Section \ref{sec:simu}, mode-wise covariance matrices may differ between classes. For class $k \in \{1, 2\}$, we set ${\bSigma}^{(k)} = \{\Sigma_m^{(k)}\}_{m=1}^3$ with $\Sigma_m^{(k)} = \sigma_k \bI_{d_m}$ for $m = 1, 2, 3$, and consider three heteroscedastic configurations $(\sigma_1, \sigma_2) \in \{(1, 2), (1, 2.5), (1, 3)\}$. All competing methods use the same ranks and tuning parameters as in Appendix \ref{subsec:nonGaussian}. We additionally implement Tensor Gaussianizing Flow (FlowTGMM) proposed in Section \ref{subsec:FlowTGMM} to handle this heteroscedastic covariance setting. FlowTGMM transforms tensor class-conditional distributions to approximate TGMM with common mode-wise covariance, after which CP-TDA is applied for classification.

\begin{table}[htbp]
\centering
\resizebox{0.95\textwidth}{!}{
\begin{tabular}{lccccccc}
\toprule
Setting & CP-TDA & TuckerLDA & CATCH & GuoSTM & LaiSTDA & ZhouLogitCP & DATED \\
\midrule
$(\sigma_1, \sigma_2) = (1, 2)$ & \cellcolor{gray!20} 0.026\textsubscript{(0.001)} & 0.035\textsubscript{(0.001)} & 0.353\textsubscript{(0.002)} & 0.457\textsubscript{(0.006)} & 0.400\textsubscript{(0.004)} & 0.339\textsubscript{(0.017)} & 0.459\textsubscript{(0.002)} \\
$(\sigma_1, \sigma_2) = (1, 2.5)$ & \cellcolor{gray!20} 0.102\textsubscript{(0.002)} & 0.152\textsubscript{(0.003)} & 0.342\textsubscript{(0.003)} & 0.440\textsubscript{(0.009)} & 0.372\textsubscript{(0.003)} & 0.377\textsubscript{(0.011)} & 0.479\textsubscript{(0.001)} \\
$(\sigma_1, \sigma_2) = (1, 3)$ & \cellcolor{gray!20} 0.220\textsubscript{(0.005)} & 0.337\textsubscript{(0.004)} & 0.327\textsubscript{(0.002)} & 0.458\textsubscript{(0.010)} & 0.340\textsubscript{(0.003)} & 0.366\textsubscript{(0.007)} & 0.483\textsubscript{(0.001)} \\
\bottomrule
\end{tabular}
}
\caption{\small Misclassification rates under heteroscedastic mode-wise covariance matrices with varying $(\sigma_1, \sigma_2)$. Best performance are highlighted in gray.}
\label{tab:nonshared_cov}
\end{table}

\noindent\underline{\textbf{Results.}} Table \ref{tab:nonshared_cov} presents misclassification rates under three heteroscedastic configurations with increasing differences between class-specific covariances. CP-TDA demonstrates robustness to moderate violations of the common covariance assumption, maintaining the lowest misclassification rates across all settings. Note that when Tensor Gaussianizing Flow (FlowTGMM) is used to transform the heteroscedastic setting to approximate homoscedasticity before applying CP-TDA, the misclassification rate drops to 0.00 even for the original covariance configuration $(\sigma_1, \sigma_2) = (1, 3)$. This demonstrates the effectiveness of FlowTGMM in addressing covariance heterogeneity. Overall, CP-TDA remains robust under moderate covariance violations, and combining it with FlowTGMM provides a principled and effective solution for more severe heteroscedasticity.

\subsection{Data-driven rank selection and robustness to rank misspecification}
\label{subsec:rankSelection}
\noindent\underline{\textbf{Setup and implementation details.}} We study performance when the true rank is unknown and must be selected from data. Settings mirror Section \ref{sec:simu}: order-3 discriminant tensor of dimensions $(30, 30, 30)$ with true CP rank $R = 5$, non-orthogonal CP bases. We consider two signal strength configurations: equal weights $(w_r = 2)$ versus unequal weights $(w_{\max} = 4)$. 

Rank selection proceeds via 10-fold cross-validation over $R \in \{3, 4, 5, 6, 7, 10\}$. Methods configure ranks as follows: GuoSTM sets MPCA output dimensions to $(R,R,R)$; LaiSTDA uses mode-wise dimensions $(R,R,R)$; DATED selects latent ranks $(R,R,R)$; TuckerLDA chooses Tucker rank $(R,R,R)$; and CP-TDA uses CP rank $R$. ZhouLogitCP performs model selection via BIC over ranks $R \in \{3, 4, 5, 6, 7, 10\}$ and ridge penalties $\{0.001, 0.005, 0.01, 0.05, 0.1, 0.5\}$. CATCH requires no rank selection.

\noindent\underline{\textbf{Results.}} Table \ref{tab:rank_selection} presents misclassification rates and most frequently selected ranks across different methods. CP-TDA and TuckerLDA consistently select the correct rank $R = 5$ across both signal configurations, achieving the lowest misclassification errors. Most competing methods overselect ranks (selecting $R = 9$), ZhouLogitCP selects a smaller rank $(R = 3)$ due to BIC's penalty for model complexity; both result in substantially higher errors.

\begin{table}[htbp!]
\centering
\resizebox{0.95\textwidth}{!}{
\begin{tabular}{lccccccc}
\toprule
Setting & CP-TDA & TuckerLDA & CATCH & GuoSTM & LaiSTDA & ZhouLogitCP & DATED \\
\midrule
$w_r = 2$, misclass & 0.001\textsubscript{(0.000)} & 0.023\textsubscript{(0.000)} & 0.188\textsubscript{(0.005)} & 0.028\textsubscript{(0.001)} & 0.381\textsubscript{(0.009)} & 0.164\textsubscript{(0.016)} & 0.291\textsubscript{(0.010)} \\
$w_r = 2$, rank & 5  & 5 & NA & 9 & 9 & 3 & 9 \\
\midrule
$w_{\max} = 4$, misclass & 0.000\textsubscript{(0.000)} & 0.003\textsubscript{(0.000)} & 0.109\textsubscript{(0.005)} & 0.008\textsubscript{(0.000)} & 0.323\textsubscript{(0.008)} & 0.130\textsubscript{(0.016)} & 0.246\textsubscript{(0.009)} \\
$w_{\max} = 4$, rank & 5  & 5 & NA & 9 & 9 & 3 & 9 \\
\bottomrule
\end{tabular}
}
\caption{\small Data-driven rank selection performance. Rows labeled ``misclas'' report mean misclassification error (standard deviation in subscript). Rows labeled ``rank'' show the most frequently selected rank across 100 replicates. CATCH does not require rank selection (marked ``NA'').}
\label{tab:rank_selection}
\end{table}

\noindent\underline{\textbf{Further assess rank specification.}} We examine CP-TDA’s robustness to rank misspecification by fitting with intentionally incorrect ranks $R \in \{1,2,3,4,6,7,10\}$. Table~\ref{tab:rank_misspec} shows that CP-TDA performs stably across a wide range of misspecified ranks. When $R < 5$, the proposed method recovers the strongest discriminant components first, and progressively incorporates additional signal as $R$ increases toward the true rank, leading to steadily improving misclassification rates and estimation errors. When $R > 5$, the inclusion of extra components introduces mild noise, but performance remains close to optimal even for substantially overspecified ranks. The key tradeoff is signal versus noise: underspecification loses discriminant information, while overspecification dilutes signal with noise. Overall, CP-TDA exhibits strong robustness to both misspecification settings.

\begin{table}[htbp!]
\centering
\resizebox{0.95\textwidth}{!}{
\begin{tabular}{lcccccccc}
\toprule
Setting & $R=1$ & $R=2$ & $R=3$ & $R=4$ & $R=5$ & $R=6$ & $R=7$ & $R=10$ \\
\midrule
$w_r=2$, misclass & 0.064\textsubscript{(0.003)} & 0.037\textsubscript{(0.003)} & 0.020\textsubscript{(0.002)} & 0.012\textsubscript{(0.003)} & 0.000\textsubscript{(0.000)} & 0.016\textsubscript{(0.002)} & 0.020\textsubscript{(0.002)} & 0.029\textsubscript{(0.003)} \\
$w_r=2$, EstError & 0.945\textsubscript{(0.003)} & 0.943\textsubscript{(0.007)} & 0.910\textsubscript{(0.009)} & 0.910\textsubscript{(0.016)} & 0.907\textsubscript{(0.012)} & 1.058\textsubscript{(0.016)} & 1.144\textsubscript{(0.019)} & 1.325\textsubscript{(0.022)} \\
\midrule
$w_{\max}=4$, misclass & 0.021\textsubscript{(0.001)} & 0.009\textsubscript{(0.001)} & 0.006\textsubscript{(0.001)} & 0.005\textsubscript{(0.001)} & 0.000\textsubscript{(0.001)} & 0.007\textsubscript{(0.001)} & 0.010\textsubscript{(0.001)} & 0.013\textsubscript{(0.001)} \\
$w_{\max}=4$, EstError & 0.817\textsubscript{(0.005)} & 0.742\textsubscript{(0.010)} & 0.731\textsubscript{(0.014)} & 0.754\textsubscript{(0.013)} & 0.715\textsubscript{(0.012)} & 0.885\textsubscript{(0.019)} & 0.948\textsubscript{(0.019)} & 1.086\textsubscript{(0.021)} \\
\bottomrule
\end{tabular}
}
\caption{\small Robustness of CP-TDA to rank misspecification under both equal weights ($w_r = 2$) and unequal weights ($w_{\max} = 4$) settings.}
\label{tab:rank_misspec}
\end{table}

\subsection{Ultra high-dimensional tensor predictors}\label{sec:highdim}
\noindent\underline{\textbf{Setup and implementation details.}}
We investigate performance when tensor predictors are ultra high-dimensional. The discriminant tensor $\cB$ has CP rank $R = 2$ with non-orthogonal bases as in Section \ref{sec:simu} and equal weights $w_1 = w_2 = 2$. Tensor dimensions $(256, 128, 128)$ mimic the neuroimaging setup of \citet{zhou2013tensor}, so that ambient dimension far exceeds sample size. Each configuration is repeated 100 times. All methods follow the implementation in Appendix \ref{subsec:others}, adapted to this ultra high-dimensional setting with true rank $R = 2$.

\noindent\underline{\textbf{Results.}} Table \ref{tab:highdim} summarizes discriminant tensor estimation errors and misclassification rates in the ultra high-dimensional regime. CP-TDA achieves the best performance among all methods; TuckerLDA ranks second with noticeably larger errors, while other competing methods struggle in this extreme setting. Notably, ZhouLogitCP performs poorly despite this setting mimicking their neuroimaging application. This is because their logistic regression framework with ridge penalties, while theoretically motivated, faces substantial computational and statistical challenges in recovering CP low-rank structure when $d \gg n$, as discussed in Appendix \ref{sec:zhoulimit}.

\begin{table}[htbp!]
\centering
\resizebox{0.95\textwidth}{!}{
\begin{tabular}{lccccccc}
\toprule
Metric & CP-TDA & TuckerLDA & CATCH & GuoSTM & LaiSTDA & ZhouLogitCP & DATED \\
\midrule
Misclass & \cellcolor{gray!20} 0.004\textsubscript{(0.003)} & 0.015\textsubscript{(0.006)} & 0.452\textsubscript{(0.012)} & 0.408\textsubscript{(0.012)} & 0.493\textsubscript{(0.010)} & 0.327\textsubscript{(0.033)} & 0.507\textsubscript{(0.018)} \\
EstError & \cellcolor{gray!20} 0.589\textsubscript{(0.043)} & 0.870\textsubscript{(0.015)} & NA & NA & NA & 1.696\textsubscript{(0.148)} & 1.001\textsubscript{(0.000)} \\
\bottomrule
\end{tabular}
}
\caption{\small Ultra high-dimensional tensor setting with dimensions $(256,128,128)$ and true CP rank $R=2$. The first row reports mean misclassification rates. The second row reports mean normalized estimation error $\|\widehat{\mathcal{B}}-\mathcal{B}\|_F / \|\mathcal{B}\|_F$. NA indicates no discriminant tensor estimate is output. Best performance are highlighted in gray.}
\label{tab:highdim}
\end{table}

\subsection{Class imbalance}
\label{subsec:unbalanced}
\noindent\underline{\textbf{Setup and implementation details.}}
We evaluate performance under class imbalance using tensor dimensions $(30,30,30)$, true CP rank $R = 5$ and non-orthogonal CP bases with equal weights ($w_r = 2.5$). Total training size is fixed at 200 with three class proportions: (i) $n_1 = 40$, $n_2 = 160$ ($2:8$); (ii) $n_1 = 60$, $n_2 = 140$ ($3:7$); and (iii) $n_1 = 80$, $n_2 = 120$ ($4:6$), where $n_1$ and $n_2$ denote samples from classes 1 and 2, respectively. All methods follow the implementation in Appendix \ref{subsec:others}, with class proportions as the only change.

\noindent\underline{\textbf{Results.}} Table~\ref{tab:unbalanced} shows that CP-TDA consistently attains the lowest misclassification rates across all imbalance levels. Its accuracy remains stable even under severe class imbalance, demonstrating strong robustness to varying class proportions, whereas competing methods deteriorate markedly as 
imbalance increases.

\begin{table}[htbp!]
\centering
\resizebox{0.95\textwidth}{!}{
\begin{tabular}{lccccccc}
\toprule
Setting & CP-TDA & TuckerLDA & CATCH & GuoSTM & LaiSTDA & ZhouLogitCP & DATED \\
\midrule
$(n_1, n_2) = (40, 160)$ & \cellcolor{gray!20} 0.080\textsubscript{(0.005)} & 0.162\textsubscript{(0.007)} & 0.423\textsubscript{(0.004)} & 0.221\textsubscript{(0.018)} & 0.450\textsubscript{(0.004)} & 0.260\textsubscript{(0.011)} & 0.460\textsubscript{(0.004)} \\
$(n_1, n_2) = (60, 140)$ & \cellcolor{gray!20} 0.026\textsubscript{(0.001)} & 0.030\textsubscript{(0.002)} & 0.258\textsubscript{(0.005)} & 0.103\textsubscript{(0.025)} & 0.441\textsubscript{(0.005)} & 0.116\textsubscript{(0.010)} & 0.430\textsubscript{(0.004)} \\
$(n_1, n_2) = (80, 120)$ & \cellcolor{gray!20} 0.003\textsubscript{(0.000)} & 0.012\textsubscript{(0.001)} & 0.132\textsubscript{(0.004)} & 0.079\textsubscript{(0.024)} & 0.423\textsubscript{(0.005)} & 0.089\textsubscript{(0.010)} & 0.412\textsubscript{(0.006)} \\
\bottomrule
\end{tabular}
}
\caption{\small Misclassification rates under class imbalance with total training size $200$. Three imbalance ratios are evaluated: $2:8, \;3:7$, and $4:6$. Best performance are highlighted in gray.}
\label{tab:unbalanced}
\end{table}

\section{Additional Real Data Results}
\label{subsec:realdata}

\noindent\underline{\textbf{Training details.}} Hyperparameters are tuned by grid search: encoder learning rates $\{10^{-4},10^{-3}\}$, flow learning rates $\{3\times10^{-4},10^{-3}\}$, CP rank $R\in\{8,12,18,24\}$, and gradient clipping $\{10,50,100\}$. We use weight decay $10^{-4}$, $K=12$ flow blocks, batch size $64$, and set $H=32$, producing tensor features of dimension $64 \times 32 \times 32$.

\noindent\underline{\textbf{Comparison with graph kernel methods.}}
\begin{table}[htbp]
\centering
\resizebox{0.75\textwidth}{!}{
\begin{tabular}{lcc}
\toprule
\textbf{Model} & \textbf{D\&D (acc $\pm$ sd, \%)} & \textbf{PROTEINS (acc $\pm$ sd, \%)} \\
\midrule
HGK-SP \citep{morris2016faster}       & 78.26 $\pm$ 0.76 & 75.54 $\pm$ 0.94 \\
HGK-WL \citep{morris2016faster}       & 79.01 $\pm$ 0.43 & 74.53 $\pm$ 0.35 \\
WL \citep{shervashidze2011weisfeiler}     & 79.78 $\pm$ 0.36 & 73.06 $\pm$ 0.47 \\
\midrule
\rowcolor{gray!20}
\textbf{STDN (ours)}               & \textbf{81.86 $\pm$ 2.19} & \textbf{78.93 $\pm$ 2.07} \\
\bottomrule
\end{tabular}
}
\caption{\small Graph classification accuracy on D\&D and PROTEINS datasets (mean $\pm$ standard deviation). STDN (highlighted) combines a neural network encoder  
with a CP-TDA head, achieving the highest accuracy.}
\label{tab:kernels_std}
\end{table}

\section{Limitations of ZhouLogitCP \citep{zhou2013tensor}}
\label{sec:zhoulimit}
In this section, we examine the limitations of ZhouLogitCP \citep{zhou2013tensor} and demonstrate the advantages of our CP-TDA for high-dimensional tensor classification.

ZhouLogitCP has three key limitations. First, its optimization is highly non-convex, and the objective function may have exponentially many stationary points in high dimensions \citep{arous2019landscape}. Proposition 1 in \cite{zhou2013tensor} assumes the set of stationary points (modulo scaling and permutation indeterminacy) are isolated, similar to Assumption 3 of \cite{zheng2020nonparametric}, but this assumption is violated as dimensions increase. 
Second, the block relaxation algorithm is a variant of Alternating Least Squares (ALS), which requires good initialization for global convergence \citep{kolda2009tensor, tang2025revisit}. ZhouLogitCP uses random initialization, which may converge to local optima arbitrarily far from the global solution, even with multiple restarts. In contrast, our \textsc{rc-PCA} initialization provides theoretical guarantees for global convergence. Notably, existing work on high-dimensional logistic regression \citep{HanWillettZhang2022} requires i.i.d. elements of each tensor predictor to construct proper initialization while ours works under correlated elements.
Third, regarding sample efficiency, ZhouLogitCP requires sample size at least as large as the effective number of parameters $p_e = R(\sum_{m=1}^M d_m - M + 1)$, otherwise necessitating penalization. CP-TDA only requires $\widehat{\Sigma}_m$ to be positive definite, imposing the mild condition on the sample size $\min\{n_1,n_2\} > d_m/d_{-m}$ for each mode $m$.

To isolate and demonstrate the impact of initialization quality, we design an experiment that eliminates the sample efficiency issue. We adopt the following setting: keep the CP-discriminant configuration used earlier but fix the true rank at $R = 3$ with equal weights $w = 2.0$ and tensor dimensions $(30, 30, 30)$. Moreover, we set sample sizes $n_1 = n_2 = 200$ so that we satisfy \cite{zhou2013tensor}'s sample requirement, allowing their algorithm to run without penalization.

We compare two initialization mechanisms. First, we use \textsc{rc-PCA} initialization with rank $R = 3$, yielding $\{\hat \ba_{rm}^{\text{rcpca}}, r\in [3], m \in[3] \}$. We form basis matrices
\[
\bA_1^{(0)} = [\widehat{\ba}_{11}^{\text{rcpca}}, \ldots, \widehat{\ba}_{1R}^{\text{rcpca}}], \quad 
\bA_2^{(0)} = [\widehat{\ba}_{21}^{\text{rcpca}}, \ldots, \widehat{\ba}_{2R}^{\text{rcpca}}], \quad 
\bA_3^{(0)} = [\widehat{\ba}_{31}^{\text{rcpca}}, \ldots, \widehat{\ba}_{3R}^{\text{rcpca}}],
\]
and use $(\bA_1^{(0)}, \bA_2^{(0)}, \bA_3^{(0)})$ as a single warm start for ZhouLogitCP's block relaxation algorithm with fixed rank $R = 3$ and no ridge penalty. Second, we try varying numbers of random initializations $K \in \{10, 50, 100, 200, 500, 1000, 2000\}$, selecting for each $K$ the initialization with largest training log-likelihood $\ell_{\max}^{(K)}$ (equivalent to smallest BIC since rank and penalty are fixed). Results from 50 replications are reported in Table \ref{tab:init_comparison}.

\begin{table}[htbp!]
\centering
\resizebox{1.05\textwidth}{!}{%
\begin{tabular}{lcccccccccc}
\toprule
Metric & CP-TDA & \makecell{ZhouLogitCP\\w/ RC-PCA} & K=10 & K=20 & K=50 & K=100 & K=200 & K=500 & K=1000 & K=2000 \\
\midrule
Misclass & 0.062\textsubscript{(0.09)} & 0.182\textsubscript{(0.026)} & 0.368\textsubscript{(0.033)} & 0.376\textsubscript{(0.033)} & 0.371\textsubscript{(0.036)} & 0.379\textsubscript{(0.048)} & 0.376\textsubscript{(0.033)} & 0.369\textsubscript{(0.039)} & 0.370\textsubscript{(0.028)} & 0.368\textsubscript{(0.026)} \\
Log-Lik & NA & 5.529\textsubscript{(0.904)} & 8.677\textsubscript{(0.808)} & 9.472\textsubscript{(0.897)} & 9.965\textsubscript{(0.874)} & 10.550\textsubscript{(0.789)} & 11.248\textsubscript{(0.771)} & 11.674\textsubscript{(0.707)} & 12.527\textsubscript{(0.645)} & 13.416\textsubscript{(0.613)} \\
\bottomrule
\end{tabular}%
}
\caption{\small Comparison of initialization strategies for ZhouLogitCP. ZhouLogitCP w/ \textsc{rc-PCA} uses a single warm start from \textsc{rc-PCA}; $K$ denotes the number of random initializations. The first row reports mean misclassification rate (standard deviation in subscript); the second row reports mean training log-likelihood (standard deviation in subscript).}
\label{tab:init_comparison}
\end{table}

Table~\ref{tab:init_comparison} demonstrates two key findings. First, for ZhouLogitCP, \textsc{rc-PCA} initialization substantially outperforms random initialization, confirming the superiority of our principled approach. However, even with \textsc{rc-PCA} initialization, ZhouLogitCP still substantially underperforms CP-TDA, likely due to inefficiency of the noisy log-likelihood under moderate sample sizes. Second, while increasing random initializations from $K=10$ to $K=2000$ monotonically improves training log-likelihood, test misclassification rate remains nearly constant. This indicates that the empirical loss landscape is too noisy and irregular that none of the random initializations land near the global optimum, resulting in poor generalizability even with thousands of restarts.

\section{Training Algorithm for Semiparametric CP-TDA (STDN)}
\label{app:training}
Each training epoch consists of two phases. In the outer snapshot phase, we pass the entire training set through the encoder-flow pipeline to obtain latent tensors $\{\cX_i^{(L)}\}_{i=1}^{n_0}$ under the current parameters $(\beta,\varphi)$ and compute the
plug-in statistics $\{\hat\pi_k, \hat\cM_k^{(L)}, \hat\Sigma_m^{(L)}, \hat\cB^{cp,(L)}\}$ on the full training set. These quantities are frozen for the subsequent inner optimization phase. 

In the inner phase, we traverse the training data in mini-batches. For each mini-batch, we
first minimize the empirical cross-entropy loss using logits constructed from
$\cX^{(L)} = g_\varphi(h_\beta(\cZ))$ and the frozen plug-in statistics, backpropagating gradients
only through $h_\beta$ to update the encoder parameters $\beta$. We then minimize the mini-batch
FlowTGMM loss $\cL_{n_0,\mathrm{Flow}}$ in Proposition~\ref{prop:log-likelihood}, again with
$\{\hat\cM_k^{(L)},\hat\Sigma_m^{(L)}\}$ held fixed, to update the flow parameters $\varphi$.
Algorithm~\ref{alg:snapshot-training} summarizes the complete snapshot-training procedure.

\begin{algorithm}[h!]
    \SetKwInOut{Input}{Input}
    \SetKwInOut{Output}{Output}
    \Input{Dataset $\{(\cZ_i, Y_i)\}_{i=1}^{n_0}$; encoder $h_\beta$; flow $g_\varphi$; CP rank $R$; batch size $B$; number of epochs $E$; learning rates $\eta_\beta, \eta_\varphi$.}
    \Output{Trained $(\hat\beta, \hat\varphi)$ and classifier head parameters $(\widehat\calB^{{\rm cp},(L)}, \widehat{\calM}_k^{(L)}, \hat\pi_1, \hat\pi_2)$.}
    
    Initialize $\beta$, $\varphi$. \\
    \For{epoch $e = 1, \ldots, E$}{
        \textbf{Outer snapshot phase}: \\
        Compute $\cX_i^{(0)}=h_\beta(\cZ_i)$ and $\cX_i^{(L)}=g_\varphi(\cX^{(0)})$ for all $i \in [n_0]$.
        
        Compute $\{\widehat{\pi}_k, \hat\calM_k^{(L)}, \widehat{\Sigma}_m^{(L)}, \hat\calB^{(L)}\}$ from $\{\cX_i^{(L)}, Y_i\}_{i=1}^{n_0}$ via equation \eqref{eqn:lda-discrim-tensor}.
        
        Refine $\widehat{\calB}^{(L)}$ using Algorithms~\ref{alg:initialize-cp}-\ref{alg:tensorlda-cp} to obtain the CP rank-$R$ tensor $\widehat{\calB}^{{\rm cp},(L)}$.
        
        \textbf{Inner optimization phase}: \\
        \For{each mini-batch $\{(\cZ_j, Y_j)\}_{j=1}^B$}{
                Compute $\cX_j^{(L)}= g_\varphi(h_\beta(\cZ_j))$ for all $j \in [B]$.
                
                \textit{Encoder step:} Compute scores $s_j \leftarrow \langle \widehat\calB^{{\rm cp},(L)}, \cX_j^{(L)} - \frac{\widehat{\calM}_0^{(L)}+\widehat{\calM}_1^{(L)}}{2} \rangle + \log(\hat\pi_2/\hat\pi_1)$. 
               Update $\beta \leftarrow \beta - \eta_{\beta} \nabla_{\beta} \frac{1}{B} \sum_{j\in [B]} [Y_j \log \sigma(s_j) + (1-Y_j)\log(1-\sigma(s_j))]$, where $\sigma$ is the logistic sigmoid.
                
                \textit{Flow step:} update $\varphi \leftarrow \varphi - \eta_\varphi \nabla_\varphi \frac{1}{B} \sum_{j=1}^B \Big[-\log f(\cX_j^{(L)} \mid Y_j; \widehat{\calM}_{Y_j}^{(L)}, \widehat{\bSigma}^{(L)}) - \sum_{\ell=1}^L \log |\det \bJ_{g^{(\ell)}}(\cX_j^{(\ell-1)})| \Big]$.
            }
        }
        \Return{$(\hat\beta, \hat\varphi, \widehat\calB^{{\rm cp},(L)}, \widehat{\calM}_k^{(L)}, \hat\pi_1, \hat\pi_2)$}
    \caption{Snapshot Training for Semiparametric Tensor Discriminant Network}
    \label{alg:snapshot-training}
\end{algorithm}

\section{Proofs for Theorems in Section \ref{sec:theorems}} \label{append:proof:main}

\subsection{Proofs for Initialization (Theorem \ref{thm:cp-initilization})} \label{append:proof:initia}

\subsubsection*{Proof I: Noise Decomposition}

The initial estimator $\hat\cB$ can be viewed as perturbed version of the ground truth $\cB$ where the noise is actually the estimation error. For notational simplicity, we sometimes give explicit expressions only in the case of $m=1$ and $M=3$. 
We first decompose the perturbing error under $M=3$ without loss of generality. The extension towards higher tensor order $M>3$ is straightforward. 
Let $C_{m,\sigma}=(1-2/n_L)\tr(\otimes_{k\neq m}\Sigma_{k})/d_{-m}, \; C_\sigma = \prod_{m=1}^M C_{m,\sigma}=(1-2/n_L)^M [\tr(\bSigma)/d]^{M-1}=(1-2/n_L)^M[\prod_{m=1}^M \tr(\Sigma_m)/d]^{M-1}$.
Recall that from \eqref{eqn:lda-rule} and \eqref{eqn:lda-discrim-tensor},
\begin{equation*}
\calB = \bbrackets{\calM_{2}-\calM_{1}; \Sigma_1^{-1}, \Sigma_2^{-1}, \Sigma_3^{-1}} 
\quad \text{and} \quad 
\hat\calB = \bbrackets{\bar\calX^{(2)}-\bar\calX^{(1)}; \hat\Sigma_1^{-1}, \hat\Sigma_2^{-1}, \hat\Sigma_3^{-1}}  .
\end{equation*}
Recall
\begin{align*}
\hat C_{\sigma} =\frac{\prod_{m=1}^M \hat\Sigma_{m,11}}{\hat{\Var}(\cX_{1\cdots1})} ,
\end{align*}
where $\hat{\Var}(\cX_{1\cdots1})$ is the pooled sample variance of the first element of $\cX_i^{(k)}$.
Define $\bD_S = \mats(\cM_{2}-\cM_{1})$. Let $\bU_{M+m}:=\bU_{m}$ for all $1\le m\le M$.

\noindent Let $\bDelta_m = \hat\Sigma_m^{-1} - C_{m,\sigma}^{-1}\Sigma_m^{-1}:=\hat\Sigma_m^{-1} - \tilde\Sigma_m^{-1}$ for $m=1,...,M-1$, and
$\bDelta_M = \hat\Sigma_M^{-1} - C_{M,\sigma}^{-1} C_{\sigma} \Sigma_m^{-1}:=\hat\Sigma_M^{-1} - \tilde\Sigma_M^{-1}$. Due to the identifiability issues associated with the tensor normal distribution, we can rescale the covariance matrices $\Sigma_{m}$ such that $C_{m,\sigma}=1$ for $1\le m\le M-1$. In this sense, we slightly abuse notation by using $\Sigma_m$ instead of $\tilde\Sigma_m$.

Then, we have the following decomposition of the error
\begin{align}\label{eqn: error decomposition}
 \cE &= \hat\calB - \calB = \bbrackets{ (\bar\calX^{(2)}-\bar\calX^{(1)}) 
- (\calM_{2}-\calM_{1});\; \Sigma_1^{-1}, \Sigma_2^{-1}, \Sigma_3^{-1}}  \notag\\
&\quad  + \bbrackets{(\bar\calX^{(2)}-\bar\calX^{(1)})- (\calM_{2}-\calM_{1});\; \bDelta_1, \Sigma_2^{-1}, \Sigma_3^{-1}} 
+ \bbrackets{(\bar\calX^{(2)}-\bar\calX^{(1)})- (\calM_{2}-\calM_{1});\; \Sigma_1^{-1}, \bDelta_2, \Sigma_3^{-1}} \notag\\
&\quad  + \bbrackets{(\bar\calX^{(2)}-\bar\calX^{(1)})- (\calM_{2}-\calM_{1});\; \Sigma_1^{-1}, \Sigma_2^{-1}, \bDelta_3} 
+ \bbrackets{(\bar\calX^{(2)}-\bar\calX^{(1)})- (\calM_{2}-\calM_{1});\; \bDelta_1, \bDelta_2, \Sigma_3^{-1}} \notag\\
&\quad  + \bbrackets{(\bar\calX^{(2)}-\bar\calX^{(1)})- (\calM_{2}-\calM_{1});\; \bDelta_1, \Sigma_2^{-1}, \bDelta_3} 
+ \bbrackets{(\bar\calX^{(2)}-\bar\calX^{(1)})- (\calM_{2}-\calM_{1});\; \Sigma_1^{-1}, \bDelta_2, \bDelta_3} \notag\\
&\quad  + \bbrackets{(\bar\calX^{(2)}-\bar\calX^{(1)}) - (\calM_{2}-\calM_{1});\; \bDelta_1, \bDelta_2, \bDelta_3} \notag\\
&\quad + 
\bbrackets{(\calM_{2}-\calM_{1});\; \bDelta_1, \Sigma_2^{-1}, \Sigma_3^{-1}} 
+ 
\bbrackets{ (\calM_{2}-\calM_{1});\; \Sigma_1^{-1}, \bDelta_2, \Sigma_3^{-1}} 
+ 
\bbrackets{ (\calM_{2}-\calM_{1});\; \Sigma_1^{-1}, \Sigma_2^{-1}, \bDelta_3} \notag\\
&\quad + 
\bbrackets{ (\calM_{2}-\calM_{1});\; \bDelta_1, \bDelta_2, \Sigma_3^{-1}} 
+ 
\bbrackets{ (\calM_{2}-\calM_{1});\; \bDelta_1, \Sigma_2^{-1}, \bDelta_3} 
+ 
\bbrackets{ (\calM_{2}-\calM_{1});\; \Sigma_1^{-1}, \bDelta_2, \bDelta_3} \notag\\
&\quad + \bbrackets{(\calM_{2}-\calM_{1});\; \bDelta_1, \bDelta_2, \bDelta_3} \notag\\
&:=\cE_0+\sum_{k=1}^3 \cE_{1,k} +\sum_{k=1}^3 \cE_{2,k} + \cE_3 +\sum_{k=1}^3 \cE_{4,k} +\sum_{k=1}^3 \cE_{5,k} + \cE_6,
\end{align}
where
\begin{align*}
\cE_0 & = \bbrackets{ (\bar\calX^{(2)}-\bar\calX^{(1)}) 
- (\calM_{2}-\calM_{1});\; \Sigma_1^{-1}, \Sigma_2^{-1}, \Sigma_3^{-1}} ,\\
\cE_{1,1} & = \bbrackets{(\bar\calX^{(2)}-\bar\calX^{(1)})- (\calM_{2}-\calM_{1});\; \bDelta_1, \Sigma_2^{-1}, \Sigma_3^{-1}} ,\\
\cE_{2,1} & = \bbrackets{(\bar\calX^{(2)}-\bar\calX^{(1)})- (\calM_{2}-\calM_{1});\; \bDelta_1, \bDelta_2, \Sigma_3^{-1}} ,\\
\cE_3 & = \bbrackets{(\bar\calX^{(2)}-\bar\calX^{(1)})- (\calM_{2}-\calM_{1});\; \bDelta_1, \bDelta_2, \bDelta_3} ,\\
\cE_{4,1} & = \bbrackets{(\calM_{2}-\calM_{1});\; \bDelta_1, \Sigma_2^{-1}, \Sigma_3^{-1}} ,\\
\cE_{5,1} & = \bbrackets{(\calM_{2}-\calM_{1});\; \bDelta_1, \bDelta_2, \Sigma_3^{-1}} ,\\
\cE_6 & = \bbrackets{(\calM_{2}-\calM_{1});\; \bDelta_1, \bDelta_2, \bDelta_3} .\\
\end{align*}


\subsubsection*{Proof II: Initialization by \textsc{c-PCA} without Procedure \ref{alg:initialize-random}}

Taken $\hat{\cB}$ are the perturbed version of the ground truth, we then calculate the initialization through the \textsc{rc-PCA} Algorithm \ref{alg:initialize-cp}. 

In this step, we prove a general results of \textsc{rc-PCA} on any noisy tensor $\hat{\cB} = \cB + \cE$ where $\cB$ assumes a CP low-rank structure $\cB=\sum_{r=1}^R w_r \circ_{m=1}^M \ba_{rm}$ and $\cE$ is the perturbation error tensor, decomposed in \eqref{eqn: error decomposition}.  
CP basis for the $m$-th mode $\{\ba_{rm}\}$, $r\in[R]$ are not necessarily orthogonal. 

We first derive the convergence rate of the CP basis when Procedure \ref{alg:initialize-random} is not applied. Specifically, we consider eigengaps satisfy $\min\{ w_i- w_{i+1}, w_{i}-w_{i-1}\} \ge c w_R$ for all $1\le i\le R$, with $w_0=\infty, w_{R+1}=0$, and $c$ is sufficiently small constant. Let $\cS$ be a subset of $[M]$ that maximizes $\min(d_{\cS}, d_{\cS^C})$ with $d_{\cS}=\prod_{m\in\cS}d_m$ and $d_{\cS^C}=\prod_{m\in\cS^C}d_m$.

Let $\ba_{r,S}=\vect ( \circ_{m \in S}\ba_{rm})$ and $\ba_{r,S^c} =\vect (\circ_{m \in [M]\backslash S}~\ba_{rm})$. Define let $\bA_{S}=(\ba_{1,S},...,\ba_{R,S}) $ and $\bA_{S^c}=(\ba_{1,S^c},...,\ba_{R,S^c}) $.
Let $\bU = (\bu_1,\ldots,\bu_R)$ and $\bV = (\bv_1,\ldots,\bv_R)$ be the orthonormal matrices in Lemma \ref{lemma-transform-ext} with $\bA$ and $\bB$ there replaced respectively by $\bA_S$ and $\bA_{S^c}$. By Lemma \ref{lemma-transform-ext},  
\begin{align}\label{eq1:thm:initial}
\|\ba_{r,S} \ba_{r,S}^\top -  \bu_{r} \bu_{r}^{\top} \|_{2}\vee 
\| \ba_{r,S^c} \ba_{r,S^c}^\top -  \bv_{r}  \bv_{r}^{\top} \|_{2} \le \delta,\quad \big\|{\rm mat}_{S}(\cB) - \bU\Lambda \bV^\top\big\|_{2}\le \sqrt{2}\delta w_1. 
\end{align}
Let $\cE^* = {\rm mat}_{S}(\cE)={\rm mat}_{S}(\hat\cB- \cB)$. We have 
\begin{align*}
\big\|{\rm mat}_{S}(\hat\cB) - \bU\Lambda \bV^\top\big\|_{2}\le \sqrt{2}\delta w_1 + \|\cE^*\|_{2}. 
\end{align*}
As $w_1>w_2>...>w_R> w_{R+1}=0$, Wedin's perturbation theorem \citep{wedin1972perturbation} provides 
\begin{align} 
\max\left\{ \|\widehat \ba_{r,S}\widehat \ba_{r,S}^\top - \bu_{r} \bu_{r}^{\top} \|_{2}, \|\widehat \ba_{r,S^c} \widehat \ba_{r,S^c}^\top - \bv_{r} \bv_{r}^{\top} \|_{2} \right\} 
\le \frac{2 \sqrt{2} w_1 \delta + 2 \|\cE^*\|_{2} }{\min\{ w_{r-1}-w_r, w_r-w_{r+1}\}}. \label{eq3:thm:initial}
\end{align}
Combining \eqref{eq1:thm:initial} and \eqref{eq3:thm:initial}, we have
\begin{align}\label{eq4:thm:initial}
\max\left\{ \|\widehat \ba_{r,S} \widehat \ba_{r,S}^\top - \ba_{r,S} \ba_{r,S}^{\top} \|_{2}, \|\widehat \ba_{r,S^c} \widehat \ba_{r,S^c}^\top - \ba_{r,S^c} \ba_{r,S^c}^{\top} \|_{2} \right\} &\le \delta+\frac{2\sqrt{2} w_1 \delta + 2 \|\cE^*\|_{2} }{\min\{ w_{r-1}-w_r, w_r-w_{r+1}\}}.
\end{align}

\noindent Consider the decomposition of the error term in \eqref{eqn: error decomposition}. The first term $\cE_0$ is a Gaussian tensor. Specifically, from the assumption on the tensor-variate $\cX$, we have $\cE_0 \sim \cT\cN(0; \frac{1}{n}\bSigma^{-1})$, where $\bSigma^{-1}= [\Sigma_m^{-1}]_{m=1}^M$. As $\| \otimes_{m=1}^M \Sigma_m\|_2 \le C_0$, by Lemma \ref{lemma:Gaussian matrix}, in an event $\Omega_{11}$ with probability at least $1-\exp(-c_1 (d_S+d_{S^c}))$, for all $m=1,2,3$, we have
\begin{align}
\left\|\mats(\cE_0) \right\|_2&\le C \frac{\sqrt{d_S}+\sqrt{d_{S^c}}}{\sqrt{n}}.    \label{eq:init_E0a}
\end{align}
For $\cE_{1,k},\cE_{2,k},\cE_3$, the operator norm of their matricization is of a smaller order compared to that of a Gaussian tensor. By Lemma \ref{lemma:precision matrix} and as $\| \otimes_{m=1}^M \Sigma_m^{-1}\|_2 \le C_0$, in an event $\Omega_{12}$ with probability at least $1-n^{-c_1}-\sum_{m=1}^M \exp(-c_1 d_m)$,
\begin{align*}
\bDelta_m &\le   C \sqrt{ \frac{d_m}{n d_{-m}}  }.
\end{align*}
For the bound of $\bDelta_3$, we need to use Lemma \ref{lemma:precision matrix}(ii) with $t_1\asymp t_2\asymp \log(n)$ to derive $|\hat C_{\sigma} - C_{\sigma}|=o(1)$ in the event $\Omega_{12}$.
Since $nd_{-m}\gtrsim d_m$ for all $m$, in the event $\Omega_{12}$, $\bDelta_m\lesssim 1$.
Thus, by Lemma \ref{lemma:Gaussian matrix}, in the event $\Omega_{11}\cap \Omega_{12}$ with probability at least $1-n^{-c_2}-\sum_{m=1}^M \exp(-c_2 d_m)$, we have
\begin{align}
\left\| \mats(\cE_{1,1}) \right\|_2    & \le \left\| \mats\left( \bbrackets{(\bar\cX^{(2)}-\bar\cX^{(1)})- (\cM_{2}-\cM_{1});\; \bI_{d_1}, \Sigma_2^{-1}, \Sigma_3^{-1}} \right) \right\|_2 \cdot \|\bDelta_1\|_2 \notag \\
&\le C_1 \frac{\sqrt{d_S}+\sqrt{d_{S^c}}}{\sqrt{n}} \cdot \sqrt{ \frac{d_1}{n d_{-1}}  }  \notag \\
&\le C \frac{\sqrt{d_S}+\sqrt{d_{S^c}}}{\sqrt{n}} , \label{eq:init_E1a}
\end{align}
for all $m=1,2,3$. Similarly, in the event $\Omega_{11}\cap \Omega_{12}$, $\| \mats(\cE_{1,k}) \|_2$, $\| \mats(\cE_{2,k}) \|_2$ and $\| \mats(\cE_3) \|_2$, $k=1, 2,3$, have the same upper bound in \eqref{eq:init_E1a}.
Recall $\bD_S = \mats(\cM_{2}-\cM_{1})$, by lemma \ref{lemma:precision matrix}, in the event $\Omega_{12}$,
\begin{align}
\left\| \mats(\cE_{4,k} ) \right\|_2    &\le C  \|\bD_S\|_2 \max_m \sqrt{\frac{d_m}{n  d_{-m}}}, \qquad k=1,2,3. \label{eq:init_E2a}
\end{align}
Again, in the event $\Omega_{12}$, $\| \mats(\cE_{5,k}) \|_2$, $\| \mats(\cE_6) \|_2$ have the same upper bound in \eqref{eq:init_E2a}.
It follows that, in the event $\Omega_1=\Omega_{11}\cap \Omega_{12}$ with probability at least $1-n^{-c_2}-\sum_{m=1}^M \exp(-c_2 d_m)$, we have
\begin{align}\label{eq5:thm:initial}
\| \mats(\cE)\|_2 =\|\cE^*\|_2 &\le  C \frac{\sqrt{d_S}+\sqrt{d_{S^c}}}{\sqrt{n}} + C  \|\bD_S\|_2 \max_{1\le m\le M} \sqrt{\frac{d_m}{n  d_{-m}}}  .    
\end{align}

We formulate each $\widehat \bu_r\in\RR^{d_S}$ to be a $|S|$-way tensor $\widehat \bU_r$. Let $\widehat \bU_{rm}={\rm mat}_m(\widehat \bU_r)$, which is viewed as an estimate of $\ba_{rm}\vect(\circ_{\ell \in S\backslash \{m\}}~\ba_{r\ell})^\top\in\RR^{d_m\times (d_S/d_m)}$. Then $\hat\ba_{rm}^{\rm rcpca}$ is the top left singular vector of $\widehat \bU_{rm}$. 
By Lemma \ref{prop-rank-1-approx}, for any $m\in S$
\begin{align*}
\|\hat\ba_{rm}^{\rm rcpca} \hat\ba_{rm}^{{\rm rcpca}\top} - \ba_{rm} \ba_{rm}^\top \|_{2}^2 \wedge (1/2) \le 
\|\widehat \ba_{r,S} \widehat \ba_{r,S}^\top - \ba_{r,S} \ba_{r,S}^{\top} \|_{2}^2 .
\end{align*}
Similar bound can be obtained for $\|\hat\ba_{rm}^{\rm rcpca} \hat\ba_{rm}^{{\rm rcpca}\top} - \ba_{rm} \ba_{rm}^\top \|_{2}$ 
for $m\in S^c$. Substituting \eqref{eq4:thm:initial} and \eqref{eq5:thm:initial} into the above equation, as $ \|\mats( \cM_{2} - \cM_{1} ) \|_{2} \asymp w_1$, we have the desired results.

\subsubsection*{Proof III: Initialization by \textsc{rc-PCA}}

Consider the situation when Procedure \ref{alg:initialize-random} is applied.

\noindent
\textsc{Step I.} Random projections in Procedure \ref{alg:initialize-random} create desirable eigen-gaps between the first and second largest eigen-values, and achieve desired upper bounds.

Let $\ba_{r,S_1}=\vect ( \circ_{m \in S_1}\ba_{rm})$ and $\ba_{r,S_1^c} =\vect (\circ_{m \in S_1} \ba_{rm})$, $\Lambda=\diag(w_1,...,w_R)$ and $w_1\ge w_2\ge\cdots \ge w_R$. Define $\bA_{S_1}=(\ba_{1,S_1},...,\ba_{R,S_1}) $ and $\bA_{S_1^c}=(\ba_{1,S_1^c},...,\ba_{R,S_1^c}) $. Let $\delta_k = \| \bA_k^\top  \bA_k - I_{R}\|_{2}$ and $\delta_{S_1} = \| \bA_{S_1}^\top  \bA_{S_1} - I_{R}\|_{2}$. Note that $\|\bD_S\|_2 = \|\mats( \cM_{2} - \cM_{1} ) \|_{2} \asymp w_1$.

\begin{lemma}\label{lem:rcpca}
Assume $\delta_1 (w_1/w_R)\le c$ for a sufficiently small positive constant $c$.
Apply random projection in Procedure \ref{alg:initialize-random} to the whole sample discriminant tensor $\widehat\cB$ with $L\ge Cd_1^2 \vee Cd_1R^{2(w_1/w_R)^2}$. Denote the estimated CP basis vectors as $\widetilde \ba_{\ell m}$, for $1\le \ell\le L, 1\le m\le M$. Then in an event with probability at least $1-n^{-c}-d_1^{-c}-\sum_{m=2}^M e^{-c d_m}$, we have for any CP basis vectors tuple $(\ba_{rm}, 1\le m\le M)$, there exist $j_r\in[L]$ such that
\begin{align}
\left\| \widetilde \ba_{j_r,m} \widetilde \ba_{j_r,m}^\top  - \ba_{rm} \ba_{rm}^\top \right\|_2  &\le \psi_i,  \qquad 2\le m\le M,\\
\left\| \widetilde \ba_{j_r,1} \widetilde \ba_{j_r,1}^\top  - \ba_{r1} \ba_{r1}^\top \right\|_2  &\le \psi_i + C \sqrt{R-1}\prod_{m=2}^M \delta_m (w_1/w_R) + C \sqrt{R-1}\psi_i^{M-1} (w_1/w_R),
\end{align}
where $1\le r\le R$ and 
\begin{align}\label{eq:psi-i}
\psi_i=C \left( \frac{\sqrt{d_1d_{S_1}}+\sqrt{d_1d_{S_1^c}} }{\sqrt{n}w_i} + \frac{\|\bD_S\|_2}{w_i} \max_{1\le k\le M} \sqrt{\frac{d_k}{n  d_{-k}}} \right) + \frac{2\delta_1 w_1}{w_R}.    
\end{align}
\end{lemma}

\begin{proof}
Define
\begin{align*}
\Xi(\theta)={\rm mat}_{S_1} (\widehat \cB\times_1\theta )=\sum_{r=1}^R w_r (\ba_{r1}^\top \theta) \ba_{r,S_1} \ba_{r,S_1^c}^\top + {\rm mat}_{S_1}(\cE\times_1\theta).    
\end{align*}

First, consider the upper bound of $\| {\rm mat}_{S_1}(\cE\times_1\theta) \|_2$. By concentration inequality for matrix Gaussian sequence (see, for example Theorem 4.1.1 in \cite{tropp2015introduction}) and employing similar arguments in \eqref{eq5:thm:initial} of Proof II, we have, in an event $\Omega_0$ with probability at least $1-n^{-c}-d_1^{-c}-\sum_{m=2}^M e^{-c d_m}$,
\begin{align} \label{eq:init-e}
\| {\rm mat}_{S_1}(\cE\times_1\theta) \|_2 &= \left\|\sum_{i=1}^{d_1} \theta_{i} \cE_{i \cdot\cdot } \right\|_2    \le C \max\left\{ \left\|{\rm mat}_{(12),(3)} (\cE) \right\|_2  ,   \left\|{\rm mat}_{(13),(2)} (\cE) \right\|_2 \right\} \cdot \sqrt{\log(d_1)} \notag\\
&\le C \left( \frac{\sqrt{d_1d_{S_1}}+\sqrt{d_1d_{S_1^c}}}{\sqrt{n}} + \|\bD_S\|_2 \max_{1\le m\le M} \sqrt{\frac{d_m}{n  d_{-m}}} \right) \sqrt{\log(d_1)} .
\end{align}
When $M=3$, $\cE_{i \cdot\cdot }$ represents the $i$-th slice of $\cE$, ${\rm mat}_{(12),(3)}(\cdot)$ denotes the reshaping of order-three tensor into a matrix by collapsing its first and second indices as rows, and the third indices as columns. In the last step, we apply the arguments in \eqref{eq5:thm:initial} of Proof II.

Consider the $i$-th factor and rewrite $\Xi(\theta)$ as follows
\begin{align}
\Xi(\theta)= w_i (\ba_{i1}^\top \theta) \ba_{i,S_1} \ba_{i,S_1^c}^\top + \sum_{r\neq i}^R w_r (\ba_{r1}^\top \theta) \ba_{r,S_1} \ba_{r,S_1^c}^\top + {\rm mat}_{S_1}(\cE\times_1\theta).
\end{align}
Suppose now we repeatedly sample $\theta_{\ell}\sim \theta$, for $\ell=1,...,L$. By the anti-concentration inequality for Gaussian random variables (see Lemma B.1 in \cite{anandkumar2014tensor}), we have
\begin{align}
\PP\left( \max_{1\le \ell \le L} \ba_{i1}^\top \theta_{\ell} \le \sqrt{2 \log(L)} - \frac{\log\log(L)}{4 \sqrt{\log(L)}} - \sqrt{2\log(8)} \right)\le \frac14,    
\end{align}
where $\odot$ denotes Kronecker product. Let
\begin{align*}
\ell_*=\arg\max_{1\le \ell\le L}  \ba_{i1}^\top \theta_{\ell}   .
\end{align*}
Note that $\ba_{i1}^\top \theta_{\ell}$ and $( I_{d_1} - \ba_{i1}\ba_{i1}^\top ) \theta_{\ell}$ are independent. Since the definition of $\ell_*$ depends only on $\ba_{i1}^\top \theta_{\ell}$, this implies that the distribution of $( I_{d_1} - \ba_{i1}\ba_{i1}^\top ) \theta_{\ell}$ does not depend on $\ell_*$.

By Gaussian concentration inequality of $1$-Lipschitz function, we have
\begin{align*}
\PP\left( \max_{r\le R}  \ba_{r1}^\top \big( I_{d_1} - \ba_{i1}\ba_{i1}^\top \big) \theta_{\ell} \ge \sqrt{4 \log(R)} +\sqrt{2 \log(8)} \right) \le \frac14   . 
\end{align*}
Moreover, for the reminder bias term $\ba_{r1}^\top \ba_{i1}\ba_{i1}^\top \theta_{\ell}$, we have,
\begin{align*}
\left\| \sum_{r\neq i} w_r \ba_{r 1}^\top  \ba_{i1}\ba_{i1}^\top \theta_{\ell} \cdot \ba_{r,S_1} \ba_{r,S_1^c}^\top \right\|_2    
&\le \ba_{i1}^\top \theta_{\ell} \cdot \left\| \bA_{S_1} \left( \Lambda \odot \left(\diag (\bA_1^\top \ba_{i1} ) - e_{ii}\right) \right)\bA_{S_1^c}^\top \right\|_2 \\
&\le \ba_{i1}^\top \theta_{\ell}  \|\bA_{S_1} \|_2\|\bA_{S_1^c} \|_2  \|\Lambda\|_2 \left\| \diag (\bA_1^\top \ba_{i1} ) - e_{ii} \right\|_2 \\
&\le (1+\delta_{S_1}\vee \delta_{S_1^c}) \delta_1 w_1 \ba_{i1}^\top \theta_{\ell},
\end{align*}
where $\odot$ denotes Hadamard product, $e_{ii}$ is a $d_1\times d_1$ matrix with the $(i,i)$-th element be 1 and all the others be 0.

Thus, we obtain the top eigengap 
\begin{align}\label{eq:lem_gap}
&w_i (\ba_{i1}^\top \theta) -  \left\| \sum_{r\neq i}^R w_r (\ba_{r1}^\top \theta) \ba_{r,S_1} \ba_{r,S_1^c}^\top \right\|_2   \notag\\
&\ge (1- 2\delta_1 w_1/w_i)\left( \sqrt{2 \log(L)} - \frac{\log\log(L)}{4 \sqrt{\log(L)}} - \sqrt{2\log(8)} \right) w_{i} - \left( \sqrt{4 \log(R)} +\sqrt{2\log(8)} \right) w_{i}(w_1/w_i) \notag\\
&\ge C_0 \sqrt{\log(d_1)} w_{i},
\end{align}
with probability at least $\frac12$, by letting $L\ge Cd_1 \vee CR^{2(w_1/w_R)^2}$.

Since $\theta_{\ell}$ are independent samples, we instead take $L_i=L_{i1}+\cdots+L_{iK}$ for $K=\lceil C_1\log (d_1)/\log(2) \rceil$ and $L_{i1},...,L_{iK}\ge Cd_1 \vee CR^{2(w_1/w_R)^2}$. We define
\begin{align*}
\ell_*^{(k)}=\arg\max_{1\le \ell\le L_{ik}}  \ba_{i1}^\top \theta_{\ell}, \quad   \ell_{*}=\arg\max_{1\le \ell\le L_{i}}  \ba_{i1}^\top \theta_{\ell}.   
\end{align*}
We then have, by independence of $\theta_{\ell}$, that the above statement \eqref{eq:lem_gap} for the $i$-th factor holds in an event $\Omega_i$ with probability at least $1-d_1^{-C_1}$. By Wedin's perturbation theory, we have in the event $\Omega_0\cap\Omega_i$,
\begin{align*}
\left\| \widetilde \bu_{\ell_*} \widetilde \bu_{\ell_*}^\top -  \ba_{i,S_1} \ba_{i,S_1}^\top \right\|_2 \vee \left\| \widetilde \bv_{\ell_*} \widetilde \bv_{\ell_*}^\top -  \ba_{i,S_1^c} \ba_{i,S_1^c}^\top \right\|_2 \le \frac{\| {\rm mat}_{S_1}(\cE\times_1\theta) \|_2}{w_{i}\sqrt{\log(d_1)}} + \frac{2\delta_1 w_1}{w_R},    
\end{align*}
where $\widetilde \bu_{\ell_*}$ and $\widetilde \bv_{\ell_*}$ are the top left and right singular vector of $\Xi(\theta_{\ell_*})$. By Lemma \ref{prop-rank-1-approx} and \eqref{eq:init-e},
\begin{align}\label{eq1:lem_rcpca}
\left\| \widetilde \ba_{\ell_*,m} \widetilde \ba_{\ell_*,m}^\top - \ba_{im} \ba_{im}^\top \right\|_2 \le C \left( \frac{\sqrt{d_1d_{S_1}}+\sqrt{d_1d_{S_1^c}} }{\sqrt{n}w_i} + \frac{\|\bD_S\|_2}{w_i} \max_{1\le k\le M} \sqrt{\frac{d_k}{n  d_{-k}}} \right) + \frac{2\delta_1 w_1}{w_R}, 
\end{align}
for all $2\le m\le M,1\le i\le R.$

Now consider to obtain $\widetilde \ba_{\ell_*,1}$. Write $\psi_i$ to be the error bound on the right hand side of \eqref{eq1:lem_rcpca}. Note that
\begin{align*}
\widehat \cB \times_{m=2}^M \widetilde \ba_{\ell_*,m}  =& \prod_{m=2}^M \left(\widetilde \ba_{\ell_*,m}^\top \ba_{im} \right) w_{i}  \ba_{i1}  +  \sum_{r\neq i} \prod_{m=2}^M \left(\widetilde \ba_{\ell_*,m}^\top \ba_{rm} \right) w_{r}  \ba_{r1}  +\cE  \times_{m=2}^M \widetilde \ba_{\ell_*,m}   .
\end{align*}
By \eqref{eq1:lem_rcpca}, in the event $\Omega_0\cap\Omega_1$,
\begin{align*}
&\left\| \cE  \times_{m=2}^M \widetilde \ba_{\ell_*,m}  \right\|_2    \le \| \cE \|_{\rm op},\\
&\prod_{m=2}^M \left(\widetilde \ba_{\ell_*,m}^\top \ba_{im} \right)  \ge (1-\psi_i^2)^{(M-1)/2}  .
\end{align*}
Since 
\begin{align}\label{eq:a_decomp}
\max_{j\neq i}\big| \ba_{j m}^\top \widetilde \ba_{\ell_*,m}  \big| &=\max_{j\neq i}\big| \widetilde \ba_{\ell_*,m}^\top \ba_{im} \ba_{im}^\top \ba_{j m} + \widetilde \ba_{\ell_*,m}^\top (I - \ba_{im} \ba_{im}^\top ) \ba_{j m}   \big|    \notag\\
&\le \max_{j\neq i}\big| \widetilde \ba_{\ell_*,m}^\top \ba_{im} \big| \big| \ba_{im}^\top \ba_{j m} \big| + \max_{j\neq i}\big\|\widetilde \ba_{\ell_*,m}^\top (I - \ba_{im} \ba_{im}^\top ) \big\|_2 \big\| (I - \ba_{im} \ba_{im}^\top )  \ba_{j m}   \big\|_2 \notag\\
&\le \sqrt{1-\psi_i^2} \delta_m + \psi_i \sqrt{1-\delta_m^2} \le \delta_m +\psi_i,
\end{align}
we have
\begin{align*}
\left\| \sum_{r\neq i} \prod_{m=2}^M \left(\widetilde \ba_{\ell_*,m}^\top \ba_{rm} \right) w_{r}  \ba_{r1}  \right\|_2^2 &\le (R-1)(1+\delta_1)\prod_{m=2}^M(\delta_m+\psi_i)^2 w_1^2  \\
&\le C_M (R-1)  \left(\prod\nolimits_{m=2}^M \delta_m^2 +\psi_i^{2M-2} \right) w_1^2 \\
&\le C_M \left(\prod\nolimits_{m=3}^M \delta_m^2 +\psi_i^{2M-4} \right) w_1^2 .
\end{align*}
By Wedin's perturbation theory,
\begin{align}\label{eq2:lem_rcpca}
\left\| \widetilde \ba_{\ell_*,1} \widetilde \ba_{\ell_*,1}^\top - \ba_{i1} \ba_{i1}^\top \right\|_2 \le \psi_i + C \sqrt{R-1}\prod_{m=2}^M \delta_m (w_1/w_R) + C \sqrt{R-1}\psi_i^{M-1} (w_1/w_R) .    
\end{align}

Repeat the same argument again for all $1\le i\le R$ factors, and let $L=\sum_{i} L_i\ge Cd_1^2 \vee Cd_1 R^{2(w_1/w_R)^2} \ge Cd_1 R\log(d_1) \vee CR^{2(w_1/w_R)^2+1}\log(d_1)$. We have, in the event $\Omega_0\cap\Omega_1\cap\cdots\cap\Omega_R$ with probability at least $1-n^{-c}-d_1^{-c}-\sum_{m=2}^M e^{-c d_m}$, \eqref{eq1:lem_rcpca} and \eqref{eq2:lem_rcpca} hold for all $i$.

\end{proof}

\noindent \textsc{Step II.} {Clustering.} 

For simplicity, consider the most extreme case where $\min\{w_i-w_{i+1},w_{i}-w_{i-1}\} \le c w_R$ for all $i$, with $w_0=\infty, w_{R+1}=0$, and $c$ is sufficiently small constant. In such cases, we need to employ Procedure \ref{alg:initialize-random} to the entire sample discriminant tensor $\widehat \cB$. Let the eigenvalue ratio $w_1/w_R=O(R)$. In general, the statement in the theorem holds for number of initialization $L\ge C d_1^2 \vee C d_1 R^{2(w_1/w_R)^2}$, where $a \vee b =\max\{a, b\}$. 
We prove the statements through induction on factor index $i$ starting from $i=1$ proceeding to $i=R$. By the induction hypothesis, we already have estimators such that 
\begin{align} \label{eq:general-init}
\left\| \widehat \ba_{jm}^{\rm rcpca} \widehat \ba_{jm}^{\rm rcpca \top}  - \ba_{jm} \ba_{jm}^\top \right\|_2  &\le C\phi_0,  \qquad 1\le j\le i-1, 1\le m\le M,
\end{align}
in an event $\Omega$ with high probability, where
\begin{align*}
\phi_0^2 = C\left(\psi_R + \sqrt{R-1}\left(\psi_R \right)^{M-1} \left(\frac{w_1}{w_R} \right) + \max_{2\le m\le M} \delta_m \left(\frac{w_1}{w_R} \right) \right) .  
\end{align*}
Recall $\psi_i$ is defined in \eqref{eq:psi-i}.

Applying Lemma \ref{lem:rcpca}, we obtain that at the $i$-th step ($i$-th factor), we have 
\begin{align*}
\left\| \widetilde \ba_{\ell m} \widetilde \ba_{\ell m}^{\top}  - \ba_{im} \ba_{im}^\top \right\|_2  &\le \phi_0^2,  \qquad  1\le m\le M,
\end{align*}
in the event $\Omega$ with probability at least $1-n^{-c}-d_1^{-c}-\sum_{m=2}^M e^{-cd_m}$ for at least one $\ell\in[L]$. It follows that this estimator $\widetilde \ba_{\ell m}$ satisfies
\begin{align*}
\left\| \widehat\cB \times_{m=1}^{M} \widetilde \ba_{\ell m} \right\|_2   &\ge  \left\| \sum_{j=1}^R w_{j} \prod_{m=1}^M \ba_{j m}^\top \widetilde \ba_{\ell m} \right\|_2    -  \left\| \cE \times_{m=1}^{M} \widetilde \ba_{\ell m} \right\|_2 \\
&\ge \left\|  w_{i} \prod_{m=1}^M \ba_{i m}^\top \widetilde \ba_{\ell m}  \right\|_2  - \left\| \sum_{j\neq i}^R w_{j} \prod_{m=1}^M \ba_{j m}^\top \widetilde \ba_{\ell m}  \right\|_2   -   \left\| \cE \times_{m=1}^{M} \widetilde \ba_{\ell m} \right\|_2 .
\end{align*}
By \eqref{eq:a_decomp} and the last part of the proof of Lemma \ref{lem:rcpca}, as $\left\| \cE \times_{m=1}^{M} \widetilde \ba_{\ell m} \right\|_2/w_i\le \phi_0^2$ and $C\sqrt{R-1}(1+\delta_1)\prod_{m=2}^M (\delta_m+\psi_i) (w_1/w_R) \le \phi_0^2$, it follows that
\begin{align*}
\left\| \widehat\cB \times_{m=1}^{M} \widetilde \ba_{\ell m} \right\|_2   &\ge  (1-\phi_0^4)^{\frac{M}{2}} w_{i} - C\sqrt{R-1}(1+\delta_1)\prod_{m=2}^M (\delta_m+\psi_i) w_1 -\phi_0^2 w_{i }    \\
&\ge (1-3\phi_0^2) w_{i}.
\end{align*}
Now consider the best initialization $\ell_*\in [L]$ by using $\ell_* =\arg\max_{s} |\widehat\cB \times_{m=1}^{M} \widetilde \ba_{s m} |$. By the calculation above, it is immediate that
\begin{align}\label{eq:lbd}
\left\| \widehat\cB \times_{m=1}^{M} \widetilde \ba_{\ell m} \right\|_2   &\ge (1-3\phi_0^2) w_{i}.    
\end{align}
If $\| \ba_{i} \ba_{i}^\top - \widetilde \ba_{\ell_*} \widetilde \ba_{\ell_*}^\top \|_2 \ge C \phi_0$ for a sufficiently large constant $C$, we have that
\begin{align*}
\left\| \widehat\cB \times_{m=1}^{M} \widetilde \ba_{\ell m} \right\|_2   &\le  \left\| \sum_{j=1}^R w_{j} \prod_{m=1}^M \ba_{j m}^\top \widetilde \ba_{\ell m}  \right\|_2   + \left\| \cE \times_{m=1}^{M} \widetilde \ba_{\ell m} \right\|_2 \\
&\le  \left\| \sum_{j=i}^R w_{j} \prod_{m=1}^M \ba_{j m}^\top \widetilde \ba_{\ell m}  \right\|_2 + \phi_0^2 w_{i }+ R \nu^{M} w_1  \\
&\le  (1+\max_m \delta_m)(1-C^2\phi_0^2/2) w_{i} + \phi_0^2 w_{i }+ R \nu^{M} w_1 .
\end{align*}
If $\nu$ satisfies $R\nu^{M} (w_1/w_R) \le c\phi_0^2$ for a small positive constant $c$, as $\max_m \delta_m \le \phi_0^2$, we have
\begin{align*}
\left\| \widehat\cB \times_{m=1}^{M} \widetilde \ba_{\ell m} \right\|_2 \le (1-C' \phi_0^2) w_{i}   ,
\end{align*}
where $C'$ is a sufficiently large constant. It contradicts \eqref{eq:lbd} above. This implies that for $\ell=\ell_*$, we have 
\begin{align*}
\| \ba_i \ba_i^\top - \widetilde \ba_{\ell_*} \widetilde \ba_{\ell_*}^\top \|_2 \le C \phi_0    .
\end{align*}
By Lemma \ref{prop-rank-1-approx}, with $\widehat \ba_{im}^{\rm rcpca} =\widetilde \ba_{\ell_* ,m}$, in the event $\Omega$ with probability at least $1-n^{-c}-d_1^{-c}-\sum_{m=2}^M e^{-c d_M}$,
\begin{align*}
\left\| \widehat \ba_{im}^{\rm rcpca} \widehat \ba_{im}^{\rm rcpca \top}  - \ba_{im} \ba_{im}^\top \right\|_2  &\le C\phi_0, \quad 1\le m\le M   .
\end{align*}
This finishes the proof by an induction argument along with the requirements $R\nu^{M} (w_1/w_R) \le c\phi_0^2$. The general upper bound of \textsc{rc-PCA} is provided in \eqref{eq:general-init}.

\subsection{Proof of Theorem \ref{thm:cp-converge} for Algorithm \ref{alg:tensorlda-cp} DISTIP-CP}
\label{sec:proof-cp-converge}


\noindent\textsc{Step I.} \textbf{Upper bound for $\widehat \ba_{rm}^{(t)}$}. 


Recall that $\bB_m = \bA_m(\bA_m^{\top} \bA_m)^{-1} = (\bb_{1m},\cdots,\bb_{Rm})$ with $\bA_m = (\ba_{1m}, \cdots, \ba_{Rm})$. Let $\bg_{rm} = \bb_{rm}/\|\bb_{rm}\|_2$, $\widehat\bg_{rm}^{(t)} = \widehat\bb_{rm}^{(t)}/\|\widehat\bb_{rm}^{(t)}\|_2$, and  
\begin{align*}
\alpha&=\sqrt{(1-\delta_{\max})(1-1/(4R))}-(R^{1/2}+1)\psi_0.
\end{align*}  
Let $\psi_{0,\ell}=\psi_0$ and define sequentially 
\begin{align}
\phi_{t,m-1} &= (M-1)\alpha^{-1}\sqrt{2R/(1-1/(4R))}\max_{1\le \ell<M}\psi_{t,m-\ell}, \notag \\
\psi_{t,m} &= \Big(2\alpha^{1-M}\sqrt{R-1}(w_1/w_R) \prod\nolimits_{\ell=1}^{M-1}\psi_{t,m-\ell}\Big)
\vee \Big(C\alpha^{1-M}\eta^{\ideal}_{Rm,\phi_{t,m-1}} \Big),   \label{psi_mk2} \\
\eta^{\ideal}_{rm,\phi}&=\frac{\sqrt{d_m}}{\sqrt{n} w_r} +   (\phi \wedge 1) \frac{\sum_{k=1}^M \sqrt{d_k}}{\sqrt{n}  w_r}   +   \frac{\|\cM_{2} - \cM_{1}\|_{\rm op} }{ w_r} \max_{1\le k\le M} \sqrt{\frac{d_k}{n  d_{-k}}} ,  \notag
\end{align} 
for $m=1,\ldots,M$, $t=1,2,\ldots$ 
By induction, \eqref{eq:rho} and $\alpha>0,\rho<1$ give $\psi_{t,m}\le\psi_{t-1,m} \le\psi_0$. 
Here and in the sequel, we take the convention that $(t,\ell)=(t-1,M+\ell)$ with the subscript $(t,\ell)$, and that $\times_\ell \widehat\theta_{j,\ell}^{(t)} = \times_{M+\ell} \widehat\theta_{j,M+\ell}^{(t-1)}$ for any estimator $\widehat\theta_{j\ell}^{(t)}$.
Let 
\begin{align*}
\Omega^*_{t,m-1} =\cap_{\ell=1}^{M-1}\Omega_{t,m-\ell}
\end{align*} 
with $\Omega_{t,\ell}=\big\{\max\nolimits_{r\le R}\| \widehat \ba_{r\ell }^{(t)}\widehat \ba_{r\ell }^{(t)\top}  - \ba_{r\ell } \ba_{r\ell }^\top \|_{2} \le \psi_{t,\ell}\big\}$.
Given $\{\widehat \ba_{r,m-\ell}^{(t)}, r\in[R], \ell\in[M-1] \}$, the $t$-th iteration for tensor mode $m$ produces estimates $\widehat \ba_{rm}^{(t)}$ as the normalized version of $\widehat\cB \times_{\ell=m-1}^{m-M+1} \widehat \bb_{r,\ell}^{(t)\top}$. Because 
$\widehat\cB =\sum_{r=1}^R w_r \circ_{m=1}^{M} \ba_{rm} + \cE$, the ``noiseless'' version of this update is given by
\begin{equation}
\hat\calB\times_{l=1, l\ne m}^{M} \bb_{rl}= w_{r} \ba_{rm} + \cE\times_{l=1, l\ne m}^{M} \bb_{rl}^\top.
\end{equation}
Similarly, for any $1\le r\le R$, 
\begin{align*}
\tilde\cB_{rm}^{(t)}:=\hat\cB \times_{\ell=m-1}^{m-M+1} \widehat \bb_{r,\ell}^{(t)\top} 
= \sum_{i=1}^R \widetilde w_{i,r} \ba_{im}  + \cE \times_{\ell=m-1}^{m-M+1} \widehat \bb_{r,\ell}^{(t)\top} \in \RR^{d_m},
\end{align*}
where 
$\widetilde w_{i,r}= w_i\prod_{\ell=1 }^{M-1} \ba_{i,m-\ell}^\top\widehat \bb_{r,m-\ell}^{(t)}$.  
At $t$-th iteration, $\widehat \ba_{rm}^{(t)}=\tilde\cB_{rm}^{(t)}/\|\tilde\cB_{rm}^{(t)}\|_2$. 

We may assume without loss of generality $\ba_{j\ell}^\top\widehat \ba_{j\ell}^{(t)}\ge 0$ for all $(j,\ell)$. 
Similar to the proofs of Proposition 4 and Theorem 3 in \cite{han2023tensor}, we can show
\begin{align}\label{a-bd}
&\max_{r\le R}\|\widehat \ba_{r\ell}^{(t)} -  \ba_{r\ell} \|_2 \le \psi_{t,\ell}/\sqrt{1-1/(4R)}, \ \ 
\displaystyle \big\|\widehat \bb_{r\ell}^{(t)}\big\|_2 \le \|\widehat \bB_\ell^{(t)}\|_{\rm 2}
\le \bigg(\sqrt{1-\delta_\ell}-\frac{R^{1/2}\psi_0}{\sqrt{1-1/(4R)}}\bigg)^{-1}, \\
&\big\|\widehat \bg_{r\ell}^{(t)} -  \bb_{r\ell}/\| \bb_{r\ell}\|_2\big\|_2
\le (\psi_{t,\ell}/\alpha)\sqrt{2R/(1-1/(4R))}.  \label{b-bd}
\end{align}
Moreover, 
\eqref{a-bd} provides 
\begin{align}\label{g-bd}
\max_{i\neq r}\big| \ba_{i\ell}^\top\widehat \bg_{r\ell}^{(t)}\big| \le \psi_{t,\ell}/\sqrt{1-1/(4R)},\ \
\big| \ba_{r\ell}^\top\widehat \bg_{r\ell}^{(t)} \big| \ge \alpha, 
\end{align}
as $\widehat \ba_{i\ell}^{(t)\top}\widehat \bg_{r\ell}^{(t)}=I\{i=r\}/\|\widehat \bb_{r\ell}^{(t)}\|_2$. 
Then, for $i\neq r$,
\begin{align*}
\widetilde w_{i,r}/\widetilde w_{r,r} &= \big(w_{1}/w_{r}\big) \prod_{\ell=1 }^{M-1} \frac{ \left| (\ba_{i,m-\ell}- \widehat \ba_{i,m-\ell}^{(t)})^\top \widehat\bb_{r,m-\ell}^{(t)} \right| }{  \left| 1+ (\ba_{r,m-\ell}- \widehat \ba_{r,m-\ell}^{(t)})^\top \widehat\bb_{r,m-\ell}^{(t)} \right| }  \\
&\le \big(w_{1}/w_{r}\big)   \prod_{\ell=1 }^{M-1}  \frac{ [\psi_{t,m-\ell}/\sqrt{1-1/(4R)}]/[\sqrt{1-\delta_{\ell}}-R^{1/2}\psi_{t,m-\ell}/\sqrt{1-1/(4R)} ]  }{1- [\psi_{t,m-\ell}/\sqrt{1-1/(4R)}]/[\sqrt{1-\delta_{\ell}}-R^{1/2}\psi_{t,m-\ell}/\sqrt{1-1/(4R)} ] }  \\
&\le \big(w_{1}/w_{r}\big)   \prod_{\ell=1 }^{M-1}  \frac{\psi_{t,m-\ell} }{\alpha} .
\end{align*} 
It follows that
\begin{align}\label{eq:contraction}
\left\| \sum_{i=1}^R \widetilde w_{i,r} \ba_{im} / \widetilde w_{r,r} - \ba_{rm} \right\|_2^2  &= \sum_{i\neq r}^R \sum_{j\neq r}^R (\ba_{im}^\top \ba_{jm})  (\widetilde w_{i,r}/\widetilde w_{r,r})  (\widetilde w_{j,r}/\widetilde w_{r,r})  \\
&\le (R-1)(1+\delta_{m}) \big(w_{1}/w_{r}\big)^2   \prod_{\ell=1 }^{M-1}   \left( \frac{\psi_{t,m-\ell} }{\alpha} \right)^2 .
\end{align}
By basic geometry, we have
\begin{align}
\| \widehat \ba_{rm}^{(t)}\widehat \ba_{rm}^{(t)\top}  - \ba_{rm} \ba_{rm}^\top \|_{\rm 2} &=\|\sin\angle(\widehat\ba_{rm}^{(t)}, \ba_{rm}) \|_2 \le   \| \tilde\cB_{rm}^{(t)}/\widetilde w_{r,r} - \ba_{rm}\|_2  \notag\\
& \le \frac{w_1\sqrt{(R-1)(1+\delta_{m})}}{w_r} \prod_{\ell=1 }^{M-1}  \frac{\psi_{t,m-\ell} }{\alpha}  + \frac{\| \cE \times_{\ell=m-1}^{m-M+1} \widehat \bg_{r,\ell}^{(t)\top} \|_2 }{w_r \prod_{\ell=1 }^{M-1} \ba_{r,m-\ell}^\top\widehat \bg_{r,m-\ell}^{(t)} }  .
\end{align}
in $\Omega^*_{t,m-1}$. As $\cE \times_{\ell=m-1}^{m-M+1} \widehat \bg_{r,\ell}^{(t)\top}$ is linear in each $\widehat \bg_{r\ell}^{(t)}$, 
\begin{align*}
\big\|  \cE \times_{\ell=m-1}^{m-M+1} \widehat \bg_{r,\ell}^{(t)\top} \big\|_{2}
\le& (M-1)\max_{\ell<M} \| \widehat \bg_{r,m-\ell}^{(t)}- \bg_{r,m-\ell} \|_2 \| \Delta \|
+ \big\| \cE \times_{\ell\in [M] \backslash\{m\} } \bg_{r\ell}^{\top} \big\|_{2}, 
\end{align*}
where $\|\Delta \|= \max_{v_\ell\in \mathbb S^{d_\ell-1}\forall\ell}\big( \cE  \times_{\ell=1}^M v_\ell^\top\big)$.  
As we also have $\big\|  \cE \times_{\ell=m-1}^{m-M+1} \widehat \bg_{r,\ell}^{(t)\top} \big\|_{2}\le \|\Delta\|$, 
\eqref{b-bd} and \eqref{g-bd} yield 
\begin{align}\label{eqthm:norm-bd}
\big\|  \cE \times_{\ell=m-1}^{m-M+1} \widehat \bg_{r,\ell}^{(t)\top} \big\|_{2}
\le& \min\big\{\|\Delta\|, \phi_{t,m-1} \| \Delta \| 
+ \big\| \cE \times_{\ell\in [M] \backslash\{m\} } \bg_{r\ell}^{\top} \big\|_{2}
\end{align}
in $\Omega^*_{t,m-1}$, in view of the definition of $\phi_{t,m-1}$ in \eqref{psi_mk2}. 
By the Sudakov-Fernique and Gaussian concentration inequalities, similar to the proof of \eqref{eq5:thm:initial}, we can show
\begin{align*}
& \|\Delta \| \le C \frac{\sum_{k=1}^M \sqrt{d_k}}{\sqrt{n}} + C  \|\cM_{2} - \cM_{1}\|_{\rm op} \max_{1\le k\le M} \sqrt{\frac{d_k}{n  d_{-k}}} , \\
& \big\| \cE \times_{\ell\in [M] \backslash\{m\} } \bg_{r\ell}^{\top} \big\|_{2}  \le 
C \frac{\sqrt{d_m}}{\sqrt{n}} + C  \|\cM_{2} - \cM_{1}\|_{\rm op} \max_{1\le k\le M} \sqrt{\frac{d_k}{n  d_{-k}}} ,
\end{align*}
in an event $\Omega_1$ with at least probability $1-n^{-c}-\sum_{k=1}^M e^{-cd_k}$. 
Consequently, by \eqref{eqthm:norm-bd}, in $\Omega_1\cap\Omega^*_{m,k-1}$, 
\begin{align}\label{eqthm:bdd-ce}
&\frac{\| \cE \times_{\ell=m-1}^{m-M+1} \widehat \bg_{r,\ell}^{(t)\top} \|_2 }{w_r \prod_{\ell=1 }^{M-1} \ba_{r,m-\ell}^\top\widehat \bg_{r,m-\ell}^{(t)} }  \notag\\
\le& \frac{C \alpha^{1-M}\sqrt{d_m}}{\sqrt{n} w_r} +   \frac{C \alpha^{1-M}\|\cM_{2} - \cM_{1}\|_{\rm op} }{ w_r} \max_{1\le k\le M} \sqrt{\frac{d_k}{n  d_{-k}}} +   C \alpha^{1-M} (\phi_{t,m-1}\vee 1) \frac{\sum_{k=1}^M \sqrt{d_k}}{\sqrt{n}  w_r} .
\end{align}
Substituting \eqref{eqthm:bdd-ce} into \eqref{eqthm:norm-bd}, we have, in the event $\Omega_1\cap\Omega_{t,m-1}^*$, 
\begin{align}\label{bdd1:thm-projection0}
\| \widehat \ba_{rm}^{(t)}\widehat \ba_{rm}^{(t)\top}  - \ba_{rm} \ba_{rm}^\top \|_{\rm 2} 
\le \psi_{t,r,m} 
\end{align}
with 
\begin{align*}
\psi_{t,r,m} =\max\bigg\{ C\alpha^{1-M}\eta^{\ideal}_{rm,\phi_{t,m-1}} , 
\frac{w_1\sqrt{2R-2}}{w_r\alpha^{M-1}} \prod_{\ell=1}^{M-1}\psi_{t,m-\ell}\bigg\}. 
\end{align*}
Consequently, $\Omega_{t,m}\subset \Omega_1\cap\Omega^*_{t,m-1}$.
Let $\Omega_0=\{\max_{r\le R,m\le M}\;\|\hat\ba_{rm}^{(0)}\hat\ba_{rm}^{(0)\top} - \ba_{rm}\ba_{rm}^\top\|_2\le \psi_0 \}$
for any initial estimates $\hat\ba_{rm}^{(0)}$. Then \eqref{bdd1:thm-projection0} holds in the event $\Omega_0\cap\Omega_1$.

\medskip
\noindent\textsc{Step II.} \textbf{Number of iterations. } 

We now consider the number of iterations and the convergence of $\psi_{t,m}$ in \eqref{psi_mk2}. A simple way of dealing with the dynamics of \eqref{psi_mk2} is to compare $\psi_{t,r,m}$ with 
\begin{align}\label{psi*}
\psi^*_{t,r,m} &= \Big(2\alpha^{1-M}\sqrt{R-1}(w_1/w_r) \prod\nolimits_{\ell=1}^{M-1}\psi^*_{t,m-\ell}\Big)
\vee \Big(C\alpha^{1-M}\eta^{\ideal}_{rm,1} \Big) \notag \\
\psi^*_{t,m} &= \psi^*_{t,R,m}, 
\end{align}
with initialization $\psi^*_{0,r,m}=\psi_0$. Compared with \eqref{psi_mk2}, \eqref{psi*} is easier to analyze due to 
the use of static 1 in $\eta^{\ideal}_{rm,\phi_{t,m-1}}$ and the monotonicity of $\psi^*_{t,m}$ in $m$. It follows that 
\begin{align}\label{psi-compare}
& \psi_{t,r,m}\le\psi^*_{t,r,m}\le \psi^*_{t,m},\quad\forall (t,r,m). 
\end{align}
As $\rho<1$ with $\rho$ defined in \eqref{eq:rho}, $\psi_{1,1}^*\le \rho \psi_0 \vee \eta^{\ideal}_{R1,1}$ and this would contribute the extra factor $\rho$ in the application \eqref{psi*} to $\psi_{1,2}^*$, resulting in $\psi_{1,2}^*\le \rho^2 \psi_0 \vee \eta^{\ideal}_{R1,1}$, so on and so forth. In general, $\psi_{t,m}^* \le (\rho^{T_{(t-1)M+m}}\psi_0 ) \vee \eta^{\ideal}_{R1,1} $ with $T_1 =1$, $T_2=2, \ldots, T_{M}=2^{M-1}$, and $T_{k+1} = 1+\sum_{\ell=1}^{M-1}T_{k+1-\ell}$ for $k>M$. By induction, for $k=M, M+1,\ldots$. 
\begin{align*}
T_{k+1} \ge \gamma_M^{k-1}+\cdots+\gamma_M^{k-M+1} = \gamma_M^k \frac{1-\gamma_M^{-M+1} }{\gamma_M-1} 
= \gamma_M^k. 
\end{align*}
The function $f(\gamma) = \gamma^M - 2\gamma^{M-1}+1$ is decreasing in $(1,2-2/M)$ and increasing $(2-2/M,\infty)$. 
Because $f(1)=0$ and $f(2)=1>0$, we have $2-2/M < \gamma_M <2$. It follows that the required number of iteration is at most $T=\lceil M^{-1}\{1+ (\log\gamma_M)^{-1} \log \log (\psi_0/\eta^{\ideal}_{R1,1})/\log(1/\rho) \} \rceil$.  Furthermore, the desired upper bound for $\widehat\ba_{rm}$ after convergence is
\begin{align}
\| \widehat \ba_{rm}^{(t)}\widehat \ba_{rm}^{(t)\top}  - \ba_{rm} \ba_{rm}^\top \|_{\rm 2}  \le   C\alpha^{1-M}\eta^{\ideal}_{r1,1}
\end{align}
in the event $\Omega_0\cap\Omega_1$.

\medskip
\noindent\textsc{Step III.} \textbf{Upper bound for $\|\widehat \cB^{\rm cp} -\cB\|_{\rm F}$}. 

After convergence, let $\hat\bb_{rm}=\hat\bb_{rm}^{(t)}$. 
For weights estimation, we have
\begin{align*}
\hat w_{r} &= \hat\calB\times_{m=1}^M \hat  \bb_{rm}^{\top}   = \cE\times_{m=1}^M \hat  \bb_{rm}^{\top} + \cB\times_{m=1}^M \hat  \bb_{rm}^{\top}   \\
&= \cE\times_{m=1}^M \hat  \bb_{rm}^{\top} + w_r \prod_{m=1}^M (\ba_{rm}^\top \hat\bb_{rm}) + \sum_{i\neq r} w_i \prod_{m=1}^M (\ba_{im}^\top \hat\bb_{rm})   .
\end{align*}
As $\sqrt{R}\psi_{t,m}<1$ and $\rho<1$, it follows that
\begin{align*}
\left| \hat w_{r} - w_{r}  \right|   &\le \left| \cE\times_{m=1}^M \hat  \bb_{rm}^{\top} \right| + w_r \left| \prod_{m=1}^M (\ba_{rm}^\top \hat\bb_{rm}) -1 \right| + \left| \sum_{i\neq r} w_i \prod_{m=1}^M (\ba_{im}^\top \hat\bb_{rm}) \right|\\
&\le C \alpha^{-M} w_r \eta^{\ideal}_{r1,1} + w_r \left(1-  \prod_{m=1}^M (1-\alpha^{-1}\psi_{t,m})  \right) + \left| \sum_{i\neq r} w_i \alpha^{-M} \prod_{m=1}^M \psi_{t,m}   \right|   \\
&\le C \alpha^{-M} w_r \eta^{\ideal}_{r1,1} +  \alpha^{-1}\sum_{m=1}^M w_r \psi_{t,r,m} + \alpha^{-M} w_R (R-1) (w_1/w_R) \prod_{m=1}^M \psi_{t,m} \\
&\le C \alpha^{-M} w_r \eta^{\ideal}_{r1,1} +  \alpha^{-1}\sum_{m=1}^M w_r \psi_{t,r,m} + \alpha^{-M} w_R \min_m \psi_{t,m} \\
&\le C_{\alpha} w_r \eta^{\ideal}_{r1,1},
\end{align*}
which is free of $w_r$ by the definition of $\eta^{\ideal}_{r1,1}$.

We may assume without loss of generality $\ba_{r\ell}^\top\widehat \ba_{r\ell}\ge 0$ for all $(r,\ell)$.
Let $\ba_{r}=\vect(\circ_{m=1}^M \ba_{rm})$ and $\hat\ba_{r}=\vect(\circ_{m=1}^M \hat\ba_{rm})$.
Employing similar arguments in the proof of \eqref{eq:contraction}, we have
\begin{align*}
\norm{\hat{\cB}^{\rm cp}  - \cB}_F & = \norm{\sum_{r\in[R]}\hat w_r\circ_{m\in[M]}\hat{\ba}_{rm} - \sum_{r\in[R]} w_r\circ_{m\in[M]}\ba_{rm}}_{\rm F}  \\
&= \norm{\sum_{r\in[R]}\hat w_r \hat \ba_{r} - \sum_{r\in[R]} w_r \ba_{r}}_{\rm 2}  \\ 
& \le \norm{\sum_{r\in[R]}(\hat w_r - w_r)\hat{\ba}_{r}}_2  + \norm{\sum_{r\in[R]} w_r\hat{\ba}_{r} - \sum_{r\in[R]} w_r\ba_{r}}_2 \\
& \le \norm{\sum_{r\in[R]}(\hat w_r - w_r)\hat{\ba}_{r}}_2  + \sqrt{R}\max_{r\le R}\norm{ w_r\hat{\ba}_{r} - w_r\ba_{r}}_2  \\
& \le 2 \sqrt{R}\max_{r\le R} \abs{\hat w_r - w_r }  + \sqrt{R}\max_{r\le R}\norm{ w_r\hat{\ba}_{r} - w_r\ba_{r}}_2  \\
&\le C_{\alpha} \sqrt{R} w_r \eta^{\ideal}_{r1,1}.
\end{align*}
Note that $ \|\cM_{2} - \cM_{1}\|_{\rm op} \asymp w_1$. We can further simplify the bounds.

\subsection{Proof of Theorem \ref{thm:class-upp-bound}}

For simplicity, we mainly focus on the proof for a simple scenario where the prior probabilities $\pi_1 = \pi_2 = 1/2$. Additionally, we will provide key steps of the proof for more general settings correspondingly. Let $\hat \Delta = \sqrt{\langle \hat \calB^{\rm cp} \times_{m=1}^M \Sigma_m, \; \hat \calB^{\rm cp} \rangle}$, the misclassification error of $\hat\Upsilon_{\rm cp}$ is
\begin{align*}
\cR_{\btheta}(\hat\Upsilon_{\rm cp}) &= \frac{n_{L_1}}{n_{L_1} + n_{L_2}} \phi\left(\hat \Delta^{-1}\log(n_{L_2}/n_{L_1}) -\frac{\langle \hat \cM - \cM_1, \; \hat \calB^{\rm cp} \rangle}{\hat \Delta} \right) \\
& + \frac{n_{L_2}}{n_{L_1} + n_{L_2}} \bar \phi\left(\hat \Delta^{-1}\log(n_{L_2}/n_{L_1}) - \frac{\langle \hat \cM - \cM_2, \; \hat \calB^{\rm cp} \rangle}{\hat \Delta} \right)
\end{align*} 
and the optimal misclassification error is
\begin{align*}
\cR_{\rm opt}=\pi_1\phi(\Delta^{-1}\log(\pi_2/\pi_1)-\Delta/2)+\pi_2 \bar \phi(\Delta^{-1}\log(\pi_2/\pi_1)+\Delta/2),    
\end{align*}
where $\phi$ is the CDF of the standard normal, and $\bar \phi(\cdot) = 1 - \phi(\cdot)$.
While the simpler version when $\pi_1 = \pi_2 = \frac{1}{2},$ are
\begin{align*}
\cR_{\btheta}(\hat\Upsilon_{\rm cp}) = \frac{1}{2} \phi\left(-\frac{\langle \hat \cM - \cM_1, \; \hat \calB^{\rm cp} \rangle}{\hat \Delta} \right) + \frac{1}{2} \bar \phi\left(-\frac{\langle \hat \cM - \cM_2, \; \hat \calB^{\rm cp} \rangle}{\hat \Delta} \right)    
\end{align*}
and $\cR_{\rm opt}=\phi(-\Delta/2)=\frac12\phi(-\Delta/2)+\frac12\bar\phi(\Delta/2),$ respectively.
Define an intermediate quantity
\begin{align*}
\cR^{*} = \frac{1}{2} \phi\left(-\frac{\langle \cD, \; \hat \calB^{\rm cp} \rangle}{2 \hat \Delta} \right) + \frac{1}{2} \bar \phi\left(\frac{\langle \cD, \; \hat \calB^{\rm cp} \rangle}{2 \hat \Delta} \right).    
\end{align*}
By Theorem \ref{thm:cp-converge}, in an event $\Omega_1$ with probability at least $\PP(\Omega_0)-n^{-c_1} - \sum_{m=1}^M  e^{-c_1 d_m}$, \begin{align}\label{eq:B_cp}
\|\hat\calB^{\rm cp} - \cB\|_{\rm F} \le C\frac{ \sqrt{\sum_{k=1}^M  d_k R}}{\sqrt{n} }   +   C w_1  \max_{1\le k\le M} \sqrt{\frac{d_kR}{n  d_{-k}}} =o (\Delta ) .   
\end{align}

\noindent Firstly, we are going to show that $R^* -\cR_{\rm opt}(\btheta) \lesssim e^{-\Delta^2/8} \cdot \Delta^{-1} \cdot \|\hat \calB^{\rm cp}  - \calB \|_{\rm F}^{2}$. 
Applying Taylor's expansion to the two terms in $\cR^{*}$ at $-\Delta/2$ and $\Delta/2$, respectively, we obtain 
\begin{equation}
\begin{split}
\label{eqn: taylor 1}
\cR^{*} - \cR_{\rm opt}(\btheta) =& \frac{1}{2}\left(\frac{\Delta}{2} -\frac{\langle \cD, \; \hat \calB^{\rm cp} \rangle}{2\hat \Delta} \right) \phi^{\prime}(\frac{\Delta}{2}) + \frac{1}{2}\left(\frac{\Delta}{2} -\frac{\langle \cD, \; \hat \calB^{\rm cp} \rangle}{2\hat \Delta} \right) \phi^{\prime}(-\frac{\Delta}{2}) \\
& + \frac{1}{2}\left(\frac{\langle \cD, \; \hat \calB^{\rm cp} \rangle}{2\hat \Delta} - \frac{\Delta}{2} \right)^2 \phi^{\prime \prime}(t_{1,n}) + \frac{1}{2}\left(\frac{\langle \cD, \; \hat \calB^{\rm cp} \rangle}{2\hat \Delta} - \frac{\Delta}{2} \right)^2 \phi^{\prime \prime}(t_{2,n})
\end{split}
\end{equation}
where $t_{1,n}, \; t_{2,n}$ are some constants satisfying $| t_{1,n} |, \; | t_{2,n} |$ are between $\frac{\Delta}{2}$ and $\frac{\langle \cD, \; \hat \calB^{\rm cp} \rangle}{2\hat \Delta}$.

Since $\big(\frac{\Delta}{2} -\frac{\langle \cD, \; \hat \calB^{\rm cp} \rangle}{2\hat \Delta} \big)$ frequently appears in \eqref{eqn: taylor 1}, we need to bound its absolute value. Let $\gamma = \calB \times_{m=1}^M \Sigma_m^{1/2}$ and $\hat \gamma = \hat \calB^{\rm cp} \times_{m=1}^M \Sigma_m^{1/2}$, 
then by Lemma \ref{lemma:tensor norm inequality}, in the event $\Omega_1$, we have
\begin{align*}
& \left|  \Delta - \frac{\langle \cD, \; \hat \calB^{\rm cp} \rangle}{\hat \Delta} \right| = \left| \|\gamma\|_{\rm F} -  \frac{\langle \gamma ,\; \hat \gamma \rangle}{\|\hat \gamma\|_{\rm F}} \right| = \left| \frac{\|\gamma\|_{\rm F} \cdot \|\hat\gamma\|_{\rm F} - \langle \gamma ,\; \hat \gamma \rangle}{\|\hat\gamma\|_2}  \right| \\
\lesssim& \frac{1}{\Delta}\|\hat \gamma - \gamma\|_{\rm F}^2 \lesssim \frac{1}{\Delta}\|\hat \calB^{\rm cp} - \calB\|_{\rm F}^2.
\end{align*}
In fact, by triangle inequality,
\begin{align*}
| \hat \Delta - \Delta | &= \left\|\hat \calB^{\rm cp}  \times_{m=1}^M \Sigma_m^{1/2} \right\|_{\rm F}- \left\| \calB \times_{m=1}^M \Sigma_m^{1/2} \right\|_{\rm F} \le \left\|\left(\hat \calB^{\rm cp} - \calB\right) \times_{m=1}^M \Sigma_m^{1/2} \right\|_{\rm F} \le \left\| \hat \calB^{\rm cp} - \calB \right\|_{\rm F} \prod_{m=1}^M \left\|\Sigma_m\right\|_{2}^{1/2} \\
& \lesssim \left\|\hat \calB^{\rm cp} - \calB\right\|_{\rm F} \lesssim \frac{ \sqrt{\sum_{k=1}^M d_k R}}{\sqrt{n} }   +   w_1  \max_{1\le k\le M} \sqrt{\frac{d_kR}{n  d_{-k}}} = o(\Delta).
\end{align*}
Since $\|\hat \calB^{\rm cp} - \calB\|_{\rm F} = o(\Delta)$, it follows that $\langle \cD, \; \hat \calB^{\rm cp} \rangle/(2\hat \Delta) \rightarrow \Delta/2$.
Then, we have $| \phi^{\prime \prime}(t_{1,n}) | \asymp | \phi^{\prime \prime}(t_{2,n}) | \asymp \Delta  e^{-\frac{(\Delta/2)^2}{2}} = \Delta  e^{-\Delta^2/8}$.
Hence,
\begin{align*}
&\frac{1}{2}\Big(\frac{\langle \cD, \; \hat \calB^{\rm cp} \rangle}{2\hat \Delta} - \frac{\Delta}{2} \Big)^2 \phi^{\prime \prime}(t_{1,n}) + \frac{1}{2}\Big(\frac{\langle \cD, \; \hat \calB^{\rm cp} \rangle}{2\hat \Delta} - \frac{\Delta}{2} \Big)^2 \phi^{\prime \prime}(t_{2,n}) \\ 
&\asymp  \frac{1}{\Delta^2} \left\|\hat \calB^{\rm cp} - \calB\right\|_{\rm F}^4 \cdot \Delta \cdot e^{-\Delta^2/8}  \asymp  \frac{1}{\Delta} e^{-\Delta^2/8} \left\|\hat \calB^{\rm cp} - \calB\right\|_{\rm F}^4   .
\end{align*}
Then \eqref{eqn: taylor 1} can be further bounded such that
\begin{align*}
\cR^{*} - \cR_{\rm opt}(\btheta) &\asymp \Big(\frac{\Delta}{2} - \frac{\langle \cD, \; \hat \calB^{\rm cp} \rangle}{2\hat \Delta} \Big) e^{-\frac{(\Delta/2)^2}{2}} + O\Big(\frac{1}{\Delta}  e^{-\Delta^2/8} \left\|\hat \calB^{\rm cp} - \calB\right\|_{\rm F}^4 \Big)\\
    & \le e^{-\Delta^2/8} \cdot \left| \frac{\Delta}{2} - \frac{\langle \cD, \; \hat \calB^{\rm cp} \rangle}{2\hat \Delta} \right| +  O\left( \frac{1}{\Delta}  e^{-\Delta^2/8} \left\|\hat \calB^{\rm cp} - \calB\right\|_{\rm F}^4 \right) \\
    & \lesssim \frac{1}{\Delta} e^{-\Delta^2/8}  \left\|\hat \calB^{\rm cp} - \calB\right\|_{\rm F}^2  +    \frac{1}{\Delta}  e^{-\Delta^2/8} \left\|\hat \calB^{\rm cp} - \calB\right\|_{\rm F}^4  .
\end{align*}
Eventually we obtain $\cR^{*} - \cR_{\rm opt}(\btheta) \lesssim \Delta^{-1} e^{-\Delta^2/8} (\|\hat \calB^{\rm cp} - \calB \|_{\rm F}^2 \vee \|\hat \calB^{\rm cp} - \calB \|_{\rm F}^4)$ in the event $\Omega_1$ with probability at least $\PP(\Omega_0)-n^{-c_1} - \sum_{m=1}^M  e^{-c_1 d_m}$.

\noindent Next, focus on $\cR_{\btheta}(\hat\Upsilon_{\rm cp}) - \cR^{*}$. We apply Taylor's expansion to $\cR_{\btheta}(\hat\Upsilon_{\rm cp})$:
\begin{align}
\label{eqn: taylor 2}
   \cR_{\btheta}(\hat\Upsilon_{\rm cp}) &= \frac{1}{2} \left\{ \phi\Big(-\frac{\langle \cD, \; \hat \calB^{\rm cp} \rangle}{2\hat \Delta} \Big) + \frac{\langle \cD, \; \hat \calB^{\rm cp} \rangle/2 - \langle \hat \cM - \cM_1, \; \hat \calB^{\rm cp} \rangle}{\hat \Delta}\phi^{\prime} \Big(\frac{\langle \cD, \; \hat \calB^{\rm cp} \rangle}{2\hat \Delta} \Big)  \right.\notag\\
   & \left.+ O\left( \Delta \cdot e^{-\Delta^2/8} \right) \Big( \frac{\langle \hat \cM - \cM_1, \; \hat \calB^{\rm cp} \rangle - \langle \cD, \; \hat \calB^{\rm cp} \rangle/2}{\hat \Delta} \Big)^2 \;  \right\} \notag\\
   &+ \frac{1}{2} \left\{ \bar \phi \Big(\frac{\langle \cD, \; \hat \calB^{\rm cp} \rangle}{2\hat \Delta} \Big) + \frac{\langle \cD, \; \hat \calB^{\rm cp} \rangle/2 + \langle \hat \cM - \cM_2, \; \hat \calB^{\rm cp} \rangle}{\hat \Delta}\phi^{\prime}\Big(\frac{\langle \cD, \; \hat \calB^{\rm cp} \rangle}{2\hat \Delta} \Big) \right.\notag\\ 
   & \left.+ O\big( \Delta \cdot e^{-\Delta^2/8} \big) \Big( \frac{\langle \hat \cM - \cM_2, \; \hat \calB^{\rm cp} \rangle + \langle \cD, \; \hat \calB^{\rm cp} \rangle/2}{\hat \Delta} \Big)^2 \;  \right\} 
\end{align}
where the remaining term can be obtained similarly as \eqref{eqn: taylor 1} by using the fact that $|\phi^{\prime \prime}(t_n)| = O(\Delta \cdot e^{-\Delta^2/8})$. Now we aim to bound the following term:
\begin{align*} 
 \left| \frac{\langle \hat \cM - \cM_1, \; \hat \calB^{\rm cp} \rangle - \langle \cD, \; \hat \calB^{\rm cp} \rangle/2}{\hat \Delta} \right| 
\lesssim & \frac{1}{\Delta} \left| \langle \bar\calX^{(2)} - \cM_2 + \bar\calX^{(1)} - \cM_1, \; \hat \calB^{\rm cp} \rangle \right| \notag\\
\lesssim & \frac{1}{\Delta} \left| \langle  \bar\calX^{(1)} - \cM_1, \; \hat \calB^{\rm cp} \rangle \right|  + \frac{1}{\Delta} \left| \langle \bar\calX^{(2)} - \cM_2, \; \hat \calB^{\rm cp} \rangle \right|.  
\end{align*}
Note that, in the event $\Omega_1$, $\|\hat \calB^{\rm cp}\|_{\rm F} \le \|\hat\calB^{\rm cp} - \cB\|_{\rm F} + \|\cB\|_{\rm F} \lesssim \Delta $. By Lemma \ref{lemma:low-rank-tensor}, in an event $\Omega_2$ with probability at least $1-e^{-c_2\sum_{m=1}^M d_m R}$,
\begin{align*}
\left| \langle  \bar\calX^{(k)} - \cM_k, \; \hat \calB^{\rm cp} \rangle \right| &\lesssim \Delta \sqrt{\frac{\sum_{m=1}^M d_m R}{n}} ,\quad k=1,2. 
\end{align*}
It follows that, in the event $\Omega_1\cap \Omega_2$, 
\begin{align}
\left| \frac{\langle \hat \cM - \cM_1, \; \hat \calB^{\rm cp} \rangle - \langle \cD, \; \hat \calB^{\rm cp} \rangle/2}{\hat \Delta} \right| &\lesssim \Delta \sqrt{\frac{\sum_{m=1}^M d_m R}{n}},    \label{eqn: upper bound of taylor terms 2}\\
\left| \frac{\langle \hat \cM - \cM_2, \; \hat \calB^{\rm cp} \rangle + \langle \cD, \; \hat \calB^{\rm cp} \rangle/2}{\hat \Delta} \right| &\lesssim \Delta \sqrt{\frac{\sum_{m=1}^M d_m R}{n}} .   \label{eqn: upper bound of taylor terms 3}
\end{align}
Substituting \eqref{eqn: upper bound of taylor terms 2} and \eqref{eqn: upper bound of taylor terms 3} into \eqref{eqn: taylor 2}, we obtain,
\begin{align*}
\left| \cR_{\btheta}(\hat\Upsilon_{\rm cp}) - \cR^{*} \right| \lesssim& \left| \frac{\langle \cD, \; \hat \calB^{\rm cp} \rangle/2 - \langle \hat \cM - \cM_1, \; \hat \calB^{\rm cp} \rangle}{\hat \Delta} \phi^{\prime}(\frac{\langle \cD, \; \hat \calB^{\rm cp} \rangle}{2\hat \Delta})  \right.\\
& \left. + \frac{\langle \cD, \; \hat \calB^{\rm cp} \rangle/2 + \langle \hat \cM - \cM_2, \; \hat \calB^{\rm cp} \rangle}{\hat \Delta} \phi^{\prime}(\frac{\langle \cD, \; \hat \calB^{\rm cp} \rangle}{2\hat \Delta}) + O \left(\Delta^3 e^{-\Delta^2/8} \left(\frac{\sum_{m=1}^M d_m R}{n} \right) \right) \right|
\end{align*}
Since $\cD/2- (\hat \cM-\cM_1) + \cD/2 + (\hat \cM-\cM_2) = \cD - (\cM_2-\cM_1) = 0$, then it follows that
\begin{align*}
\left| \cR_{\btheta}(\hat\Upsilon_{\rm cp}) - \cR^{*} \right| \lesssim \Delta^3 e^{-\Delta^2/8} \left(\frac{\sum_{m=1}^M d_m R}{n} \right)  .  
\end{align*}

Finally, combining the two pieces, we obtain
\begin{align*}
\cR_{\btheta}(\hat\Upsilon_{\rm cp}) -\cR_{\rm opt}(\btheta) \le& \cR_{\btheta}(\hat\Upsilon_{\rm cp}) -  \cR^{*} + \cR^{*} - \cR_{\rm opt}(\btheta) \\
\lesssim & \frac{1}{\Delta} e^{-\Delta^2/8}  \left\|\hat \calB^{\rm cp} - \calB\right\|_{\rm F}^2  +    \frac{1}{\Delta}  e^{-\Delta^2/8} \left\|\hat \calB^{\rm cp} - \calB\right\|_{\rm F}^4 + \Delta^3 e^{-\Delta^2/8} \left(\frac{\sum_{m=1}^M d_m R}{n} \right) ,
\end{align*}
in the event $\Omega_1\cap \Omega_2$ with probability at least $\PP(\Omega_0)-n^{-c}-\sum_{m=1}^M e^{-cd_m}$.

Now consider the two case. On the one hand, when $\Delta = O(1)$, by \eqref{eq:B_cp}, with probability at least $\PP(\Omega_0)-n^{-c}-\sum_{m=1}^M e^{-cd_m }$, we have 
\begin{align*}
\cR_{\btheta}(\hat\Upsilon_{\rm cp}) - \cR_{\rm opt}(\btheta) \le C \frac{\sum_{m=1}^M d_m R}{n}  +     C \frac{ w_1^2 R }{\Delta^2 }  \max_{1\le m\le M} \frac{d_m}{n  d_{-m}}  .
\end{align*}
On the other hand, when $\Delta\to\infty$ as $n\to \infty$, by\eqref{eq:B_cp}, with probability at least $1-n^{-c}-\sum_{m=1}^M e^{-cd_m }$, we have 
\begin{align*}
& \cR_{\btheta}(\hat\Upsilon_{\rm cp}) - \cR_{\rm opt}(\btheta) \\
\le& C \Delta^3 e^{-\Delta^2/8} \left(\frac{\sum_{m=1}^M d_m R }{n} \right) +  C \Delta e^{-\Delta^2/8} \frac{w_1^2 R \max_{1\le m\le M}d_m^2}{\Delta^2 nd  } + C \Delta^3 e^{-\Delta^2/8} \left(\frac{w_1^2 R \max_{1\le m\le M}d_m^2}{\Delta^2 nd} \right)^2 \\
\le& C \Delta^3 e^{-\Delta^2/8} \left(\frac{\sum_{m=1}^M d_m R }{n}  +\frac{w_1^2 R \max_{1\le m\le M}d_m^2}{\Delta^2 nd} \right)\\
=& C \exp\left(-\left(\frac18-\frac{3\log(\Delta)}{\Delta^2} \right)\Delta^2\right) \left(\frac{\sum_{m=1}^M d_m R }{n}  +\frac{w_1^2 R \max_{1\le m\le M}d_m^2}{\Delta^2 nd} \right),
\end{align*}
where $3\log(\Delta)/\Delta^2$ is an $o(1)$ term as $n\to\infty$.

As $\Delta^2\asymp w_1^2+\cdots +w_R^2\gg w_1^2 (d_{\max}/d)$, we have
\begin{align*}
\frac{w_1^2 R \max_{1\le m\le M}d_m^2}{\Delta^2 nd} \ll   \frac{\sum_{m=1}^M d_m R }{n}.  
\end{align*}
That is, the second part in the excess misclassification rate, which comes from the estimation accuracy of the mode-$m$ precision matrix is negligible.

\subsection{Proof of Theorem \ref{thm:class-lower-bound}}

Note that the proof is not straightforward, partly because the excess risk $\cR_{\btheta}(\hat\Upsilon_{\rm cp}) -\cR_{\rm opt}(\btheta)$ does not satisfy the triangle inequality required by standard lower bound techniques. A crucial approach in this context is establishing a connection to an alternative risk function.
For a general classification rule $\Upsilon$, we define $L_{\btheta}(\Upsilon)=\PP_{\btheta}(\Upsilon(\cZ) \neq \Upsilon_{\theta}(\cZ))$, where $\Upsilon_{\theta}(\cZ)$ is the Fisher’s linear discriminant rule introduced in (\ref{eqn:lda-rule}). Lemma \ref{lemma:the first reduction} 
allows us to transform the excess risk $\cR_{\btheta}(\hat\Upsilon_{\rm cp}) -\cR_{\rm opt}(\btheta)$ into the risk function $L_{\btheta}(\hat \Upsilon_{\rm cp})$, as shown below:
\begin{equation} \label{eqn:loss function reduction}
\cR_{\btheta}(\hat\Upsilon_{\rm cp}) -\cR_{\rm opt}(\btheta) \ge \frac{\sqrt{2\pi}\Delta}{8} e^{\Delta^2/8} \cdot L_{\theta}^2(\hat \Upsilon_{\rm cp}).    
\end{equation}
We then apply Lemma \ref{lemma:Tsybakov variant} to derive the minimax lower bound for the risk function $L_{\btheta}(\hat \Upsilon_{\rm cp})$.

We carefully construct a finite collection of subsets of the parameter space $\calH$ that characterizes the hardness of the problem. 
Any $M$-th order tensor $\cM \in \mathbb{R}^{d_1 \times \cdots \times d_M}$ with CP rank $R$ can be expressed as $\cM= \cF \times_{m=1}^M \bA_m$. Here, the latent core tensor $\cF=\diag(w_1,...,w_R)$ is a diagonal tensor of dimensions $R \times \cdots \times R$, i.e. $\cF_{i,...,i}$ is non zero for all $i=1,..,R$, and all the other elements of $\cF$ are zero.
Denote the loading matrices $\bA_m=(\ba_{1m},...,\ba_{Rm}) \in \mathbb{R}^{d_m \times R}$ for each mode $m=1,\ldots,M$. 

First, let $\bA_m$ be a fixed matrix where the $(i,i)$-th elements, $i=1,...,R$, are set to one and all other elements are zero. Denote 
$\bA=\bA_M\otimes \bA_{M-1}\otimes \cdots\otimes \bA_1$. According to basic tensor algebra, this setup implies that 
$\vect(\cM) = \bA \vect(\cF)$ and $\|\cM\|_{\rm F} = \|\vect(\cF)\|_2$. Let $\be_1$ be the basis vector in the standard Euclidean space whose first entry is 1 and 0 elsewhere, and $\bI_{d}=[\bI_{d_m}]_{m=1}^M$.
Define the following parameter space
\begin{align*}
\cH_0 =& \big\{ \theta = (\cM_1, \; \cM_2, \; \bI_{d}) : \; \cM_1 = \cF \times_{m=1}^M \bA_m, \; \cM_2= -\cM_1; 
\ \vect(\cF)=\epsilon \bff+ \lambda \be_1, \bff \in \{ 0,1\}^R, \bff^\top \be_1=0 \big\},
\end{align*}
where $\epsilon=c/\sqrt{n}$, $c=O(1)$ and $\lambda$ is chosen to ensure that $\theta\in\cH$ such that
\begin{align*}
\Delta=(\cM_2-\cM_1)^\top\bSigma^{-1} (\cM_2 -\cM_1) = 4 \| \epsilon \bff+ \lambda \be_1\|_2^2 = 4 \epsilon^2 \| \bff\|_2^2 + 4 \lambda^2. 
\end{align*}
In addition to $\cH_0$, we also define $\bA_{\ell}$, for $\ell\neq m$, as fixed matrices where the $(i,i)$-th elements, $i=1,...,R$, are set to one and all other elements are zero. Let $\cF$ be a diagonal tensor such that the $(i,...,i)$-th elements, $i=1,..., R$, are set to one, and all other elements are zero. It implies that $\|\cM\|_{\rm F} = \| \bA_m\|_{\rm F}$. For $m=1,...,M$, define the following parameter spaces 
\begin{align*}
\cH_m = \big\{& \theta = (\cM_1, \; \cM_2, \; \bI_{d}) : \; \cM_1 = \cF \times_{k=1}^M \bA_k, \; \cM_2= -\cM_1; 
\ \vect(\bA_m)=\epsilon \bg_m+ \lambda_m \be_1, \bg_m \in \{ 0,1\}^{d_m R},\\
&\bg_m^\top \be_1=0 \big\},
\end{align*}
where $\epsilon=c/\sqrt{n}$, $c=O(1)$ and $\lambda_m$ is chosen to ensure that $\theta\in\cH$ such that
\begin{align*}
\Delta=(\cM_2-\cM_1)^\top\bSigma^{-1} (\cM_2 -\cM_1) = 4 \| \epsilon \bg_m+ \lambda_m \be_1\|_2^2 = 4 \epsilon^2 \| \bg_m\|_2^2 + 4 \lambda_m^2. 
\end{align*}
It is clear that $\cap_{\ell=0}^M \cH_{\ell} \subset \cH$. We shall show below separately for the minimax risks over each parameter space $\cH_{\ell}$.

First, consider $\cH_0$. By Lemma \ref{lemma:Varshamov-Gilbert Bound}, we can construct a sequence of $R$-dimensional vectors $\bff_1, \ldots ,\bff_N \in \{0,1\}^R$, such that $\bff_{i}^\top \be_1=0$, $\rho_H(\bff_i, \bff_j) \ge R/8, \; \forall 0 \le i < j \le N$, and $R \le (8/\log 2)\log N$, where $\rho_H$ denotes the Hamming distance.
To apply Lemma \ref{lemma:Tsybakov variant}, for $\forall \theta_{\bu}, \; \theta_{\bv} \in \cH_0, \; \theta_{\bu} \neq \theta_{\bv},$ we need to verify two conditions:  
\begin{enumerate}
\item[(i)] the upper bound on the Kullback-Leibler divergence between $\PP_{\theta_{\bu}}$ and $\PP_{\theta_{\bv}}$, and
\item[(ii)] the lower bound of $L_{\theta_{\bu}}(\hat \Upsilon_{\rm cp}) + L_{\theta_{\bv}}(\hat \Upsilon_{\rm cp})$ for $\bu \neq \bv$ and $\bu^\top \be_1=0, \bv^\top \be_1=0$.
\end{enumerate}

We calculate the Kullback-Leibler divergence first. For $\bff_{\bu} \in\{ 0,1\}^R$ and $\bff_{\bu}^\top \be_1=0$, define
\begin{align*}
\cM_{\bu} = \cF_{\bu} \times_{m=1}^M \bA_m, \; \vect(\cF_{\bu})= \epsilon \bff_{\bu} + \lambda \be_1, \; \theta_{\bu} = \big(\cM_{\bu}, \; -\cM_{\bu}, \; \bI_d \big) \in \cH_0 .   
\end{align*}
and consider the distribution $\cT\cN(\cM_{\bu}, \; \bI_d)$. 
Then the Kullback-Leibler divergence between $\PP_{\theta_{\bu}}$ and $\PP_{\theta_{\bv}}$ can be bounded by
\begin{align*}
    {\rm KL}(\PP_{\theta_{\bu}}, \PP_{\theta_{\bv}}) &= \frac{1}{2} \norm{\vect(\cM_{\bu}) - \vect(\cM_{\bv})}_2^2 = \frac{1}{2}\norm{\vect(\cF_{\bu}) - \vect(\cF_{\bv})}_2^2 \le \frac{c^2 R}{2n}  .
\end{align*}
In addition, by applying Lemma \ref{lemma:probability inequality}, we have that for any $\bff_{\bu}, \bff_{\bv} \in\{ 0,1\}^R$,
\begin{align*}
    L_{\theta_{\bu}}(\hat \Upsilon_{\rm cp}) + L_{\theta_{\bv}}(\hat \Upsilon_{\rm cp}) &\ge \frac{1}{\Delta} e^{-\Delta^2/8} \cdot \norm{\vect(\cF_{\bu}) - \vect(\cF_{\bv})}_2 \\
    &\ge \frac{1}{\Delta} e^{-\Delta^2/8} \sqrt{\frac{R}{8} \cdot \frac{c^2}{n}  } \\
    &\gtrsim \frac{1}{\Delta} e^{-\Delta^2/8} \sqrt{\frac{R}{n}} .
\end{align*}
So far we have verified the aforementioned conditions (i) and (ii). Lemma \ref{lemma:Tsybakov variant} immediately implies that, there exists some constant $C_{\gamma} > 0$, such that 
\begin{equation}\label{eqn:inter1}
\inf_{\hat \Upsilon_{\rm cp}} \sup_{\theta \in \calH_0} \PP\left(L_{\theta}(\hat \Upsilon_{\rm cp}) \ge C_{\gamma} \frac{1}{\Delta} e^{-\Delta^2/8} \sqrt{ \frac{R}{n} } \right) \ge 1-\gamma
\end{equation}
Combining \eqref{eqn:inter1} and \eqref{eqn:loss function reduction}, we have
\begin{equation}\label{eqn:inter2}
\inf_{\hat \Upsilon_{\rm cp}} \sup_{\theta \in \calH_0} \PP\left(\cR_{\btheta}(\hat\Upsilon_{\rm cp}) -\cR_{\rm opt}(\btheta) \ge C_{\gamma} \frac{1}{\Delta} e^{-\Delta^2/8} \cdot \frac{R}{n}  \right) \ge 1-\gamma   .
\end{equation}

Similarly, for each $\cH_m$, $m=1,...,M$, we can obtain that,  there exists some constant $C_{\gamma} > 0$, such that 
\begin{equation}\label{eqn:inter3}
\inf_{\hat \Upsilon_{\rm cp}} \sup_{\theta \in \calH_m} \PP\left(\cR_{\btheta}(\hat\Upsilon_{\rm cp}) -\cR_{\rm opt}(\btheta) \ge C_{\gamma} \frac{1}{\Delta} e^{-\Delta^2/8} \cdot \frac{d_m R }{n}  \right) \ge 1-\gamma   .
\end{equation}

Finally combining \eqref{eqn:inter2} and \eqref{eqn:inter3}, we obtain the desired lower bound for the excess misclassficiation error
\begin{equation}\label{eqn:mis}
\inf_{\hat \Upsilon_{\rm cp}} \sup_{\theta \in \calH} \PP\left(\cR_{\btheta}(\hat\Upsilon_{\rm cp}) -\cR_{\rm opt}(\btheta) \ge C_{\gamma} \frac{1}{\Delta} e^{-\Delta^2/8} \cdot \frac{\sum_{m=1}^M d_m R + R}{n}  \right) \ge 1-\gamma   .
\end{equation}

This implies that if $c_1<\Delta \le c_2$ for some $c_1,c_2>0$, we have 
\begin{align*}
\inf_{\hat \Upsilon_{\rm cp}} \sup_{\theta \in \calH} \PP\left(\cR_{\btheta}(\hat\Upsilon_{\rm cp}) -\cR_{\rm opt}(\btheta) \ge C_{\gamma} \cdot \frac{\sum_{m=1}^M d_m R }{n}  \right) \ge 1-\gamma   .    
\end{align*}
On the other hand, if $\Delta\to\infty$ as $n\to \infty$, then for any $\vartheta>0$
\begin{align*}
\inf_{\hat \Upsilon_{\rm cp}} \sup_{\theta \in \calH} \PP\left(\cR_{\btheta}(\hat\Upsilon_{\rm cp}) -\cR_{\rm opt}(\btheta) \ge C_{\gamma} \exp\left\{-\left(\frac18+\vartheta\right)\Delta^2 \right\} \frac{\sum_{m=1}^M d_m R }{n}  \right) \ge 1-\gamma   .    
\end{align*}



\section{Proofs for Section \ref{sec:Tensor LDA-TNN}}

\subsection{Proof of Proposition \ref{prop:log-likelihood}}
\label{app:proof1}
We derive the log-likelihood in \eqref{eqn:objective function} and establish the corresponding population objective.  By convention, maximum likelihood estimation minimizes the negative log-likelihood (the loss function). 

\noindent\textbf{Sample log-likelihood.} 
Given a labeled dataset $\{(\mathcal{X}_i, Y_i)\}_{i=1}^{n_0}$ where $\mathcal{X}_i \in \mathbb{R}^{d_1 \times \cdots \times d_M}$ are tensor-valued observations and $Y_i \in [K]$ are class labels, the sample log-likelihood is
$\mathcal{L} = \sum_{i=1}^{n_0} [\log \pi_{Y_i} + \log p_{\mathcal{X}|Y}(\mathcal{X}_i \mid Y_i)]$,
where $\pi_k$ are the class priors and $p_{\mathcal{X}|Y}(\cdot \mid k)$ is the class-conditional density. Since the prior terms $\log \pi_{Y_i}$ do not depend on the flow parameters $\varphi$ or the dataset $\{(\mathcal{X}_i, Y_i)\}_{i=1}^{n_0}$, they can be omitted during optimization.

Fix an observation $i$ and denote $\mathcal{X}^{(0)} = \mathcal{X}_\beta(\mathcal{Z}_i)$ as the encoder output. The Tensor RealNVP flow applies $L$ transformations sequentially: for $\ell = 1, \ldots, L$, define $\mathcal{X}^{(\ell)} = g^{(\ell)}(\mathcal{X}^{(\ell-1)})$, where each $g^{(\ell)}$ consists of a mode-wise linear mixer followed by an affine coupling, as specified in Definition \ref{def:realnvp}. Since the mode-wise mixer $\mathbf{H}_m^{(\ell)} \in \mathbb{R}^{d_m \times d_m}$ are invertible matrices, and the affine coupling uses elementwise exponentiation, each block $g^{(\ell)}$ is a bijection. Therefore, the full flow $g_\varphi = g^{(L)} \circ \cdots \circ g^{(1)}$ is a composition of bijections, with Jacobian given by the chain rule:
\[
\mathbf{J}_{g_{\varphi}}(\mathcal{X}^{(0)}) = \mathbf{J}_{g^{(L)}}(\mathcal{X}^{(L-1)}) \cdots \mathbf{J}_{g^{(2)}}(\mathcal{X}^{(1)}) \mathbf{J}_{g^{(1)}}(\mathcal{X}^{(0)}).
\]
Let $f(\cdot \mid Y = k; \mathcal{M}_k, \bSigma)$ denote the tensor normal density in the latent space, where $\bSigma = [\Sigma_m]_{m=1}^M$ represents the mode-wise covariances. By the change of variables formula for probability densities:
\begin{align*}
p_{\mathcal{X}|Y}(\mathcal{X}^{(0)} \mid k) &= f(g_{\varphi}(\mathcal{X}^{(0)}) \mid Y = k; \mathcal{M}_k, \bSigma) \left| \det \mathbf{J}_{g_{\varphi}}(\mathcal{X}^{(0)}) \right| \\
&= f(\mathcal{X}^{(L)} \mid Y = k; \mathcal{M}_k, \Sigma) \prod_{\ell=1}^{L} \left| \det \mathbf{J}_{g^{(\ell)}}(\mathcal{X}^{(\ell-1)}) \right|,
\end{align*}
where we used the multiplicative property of determinants: $\det(\mathbf{A}\mathbf{B}) = \det(\mathbf{A})\det(\mathbf{B})$.
Taking logarithms and summing over all observations $i = 1, \ldots, n_0$ yields the sample loss:
\begin{align*}
\mathcal{L}_{n_0, \text{Flow}}(\varphi) = -\sum_{i=1}^{n_0} \left[ \log f(\mathcal{X}_i^{(L)} \mid Y = Y_i; \widehat{\mathcal{M}}_{Y_i}, \widehat{\Sigma}) + \sum_{\ell=1}^{L} \log \left| \det \mathbf{J}_{g^{(\ell)}}(\mathcal{X}_i^{(\ell-1)}) \right| \right], 
\end{align*}
where $\mathcal{X}_i^{(\ell)} = g^{(\ell)}(\mathcal{X}_i^{(\ell-1)})$ for $\ell = 1, \ldots, L$ with $\mathcal{X}_i^{(0)} = \mathcal{X}_\beta(\mathcal{Z}_i)$, and $\widehat{\mathcal{M}}_k, \widehat{\Sigma}$ are computed from the latent features $\{\mathcal{X}_i^{(L)}\}_{i=1}^{n_0}$. This establishes \eqref{eqn:objective function}.

\noindent\textbf{Population loss.}
Let $\mathbb{P}_{\beta,k}$ denote the unknown distribution of encoder features $\mathcal{X}_\beta(\mathcal{Z}_i)$ conditioned on label $k$ with prior $\pi_k$. The population loss is obtained by replacing the empirical average with expectation and using population parameters:
\begin{align}
\mathcal{L}^*_{\text{Flow}}(\varphi) = -\sum_{k=1}^{K} \pi_k \mathbb{E}_{\mathcal{X} \sim \mathbb{P}_{\beta, k}} \left[ \log f(g_{\varphi}(\mathcal{X}) \mid Y = k; \mathcal{M}_k, \bSigma) + \sum_{\ell=1}^{L} \log \left| \det \mathbf{J}_{g^{(\ell)}}(\mathcal{X}^{(\ell-1)}) \right| \right], \label{eq:population_loss}
\end{align}
where $\mathcal{M}_k = \mathbb{E}[\mathcal{X}^{(L)} \mid Y = k]$ and $\bSigma$ are the population class means and mode-wise covariances.

\noindent\textbf{Log-Jacobian of a Tensor RealNVP block.}
We now derive the log-Jacobian formula \eqref{eqn:Jacobian} for a single flow block $g^{(\ell)}$. We suppress the superscript $\ell$ for notational clarity. By Definition \ref{def:realnvp}, each block $g$ consists of two components applied sequentially: (i) a mode-wise linear mixer, and (ii) an affine coupling based on a binary mask. The mixer transforms $\mathcal{X}^{(\ell-1)}$ to $\widetilde{\mathcal{X}}^{(\ell)} := \mathcal{X}^{(\ell-1)} \times_{m=1}^{M} \mathbf{H}_m^{(\ell)}$, where each $\mathbf{H}_m^{(\ell)} \in \mathbb{R}^{d_m \times d_m}$ is an invertible mode-mixer matrix. Under vectorization, this multilinear operation becomes
\[
\text{vec}(\widetilde{\mathcal{X}}^{(\ell)}) = (\mathbf{H}_M^{(\ell)} \otimes \cdots \otimes \mathbf{H}_1^{(\ell)}) \text{vec}(\mathcal{X}^{(\ell-1)}),
\]
where $\otimes$ denotes the Kronecker product. Therefore, the Jacobian of the mixer is
\[
\mathbf{J}_{\text{mix}} = \frac{\partial \;\text{vec}(\widetilde{\mathcal{X}}^{(\ell)})}{\partial \;\text{vec}(\mathcal{X}^{(\ell-1)})} = \mathbf{H}_M^{(\ell)} \otimes \cdots \otimes \mathbf{H}_1^{(\ell)}.
\]
Using the Kronecker product determinant formula, we obtain $\det \mathbf{J}_{\text{mix}} = \prod_{m=1}^{M} (\det \mathbf{H}_m^{(\ell)})^{d/d_m}$, where $d = \prod_{m=1}^M d_m$. Since each $\mathbf{H}_m^{(\ell)}$ is orthogonal, taking logarithms yields
\[
\log \left| \det \mathbf{J}_{\text{mix}} \right| = \sum_{m=1}^{M} \frac{d}{d_m} \log \left| \det \mathbf{H}_m^{(\ell)} \right|=0.
\]
Next, we analyze the affine coupling transformation. By Definition \ref{def:realnvp}, the affine coupling uses a pre-specified binary mask $\mathcal{K}^{(\ell)} \in \{0,1\}^{d_1 \times \cdots \times d_M}$ to partition the indices into two disjoint sets:
\[
\mathcal{A}_1^{(\ell)} = \{\mathbf{i} = (i_1, \ldots, i_M) : \mathcal{K}_{\mathbf{i}}^{(\ell)} = 1\}, \quad \mathcal{A}_0^{(\ell)} = \{\mathbf{i} : \mathcal{K}_{\mathbf{i}}^{(\ell)} = 0\}.
\]
The coupling transformation is defined as
\[
\mathcal{Y}_{\mathcal{A}_1^{(\ell)}} = \widetilde{\mathcal{X}}_{\mathcal{A}_1^{(\ell)}}^{(\ell)}, \quad
\mathcal{Y}_{\mathcal{A}_0^{(\ell)}} = \widetilde{\mathcal{X}}_{\mathcal{A}_0^{(\ell)}}^{(\ell)} \odot \exp\left(v^{(\ell)}\left(\widetilde{\mathcal{X}}_{\mathcal{A}_1^{(\ell)}}^{(\ell)}\right)\right) + t^{(\ell)}\left(\widetilde{\mathcal{X}}_{\mathcal{A}_1^{(\ell)}}^{(\ell)}\right),
\]
where $v^{(\ell)}, t^{(\ell)}$ are tensor-valued conditioner networks, $\odot$ denotes elementwise multiplication, and $\exp(\cdot)$ is applied elementwise. Ordering coordinates as $(\mathcal{A}_1^{(\ell)}, \mathcal{A}_0^{(\ell)})$ in both input and output, the Jacobian matrix has a block lower-triangular structure:
\[
\mathbf{J}_{\text{cpl}} = \begin{bmatrix} 
\mathbf{I}_{|\mathcal{A}_1^{(\ell)}|} & 0 \\[0.5em]
* & \text{Diag}\left(\exp\left(v^{(\ell)}\left(\widetilde{\mathcal{X}}_{\mathcal{A}_1^{(\ell)}}^{(\ell)}\right)\right)\right)
\end{bmatrix},
\]
where the upper-left block is identity since $\mathcal{Y}_{\mathcal{A}_1^{(\ell)}} = \widetilde{\mathcal{X}}_{\mathcal{A}_1^{(\ell)}}^{(\ell)}$; the upper-right block is zero since $\mathcal{Y}_{\mathcal{A}_1^{(\ell)}}$ does not depend on $\widetilde{\mathcal{X}}_{\mathcal{A}_0^{(\ell)}}^{(\ell)}$; the lower-right diagonal block comes from the elementwise scaling $\widetilde{\mathcal{X}}_{\mathcal{A}_0^{(\ell)}}^{(\ell)} \odot \exp(v^{(\ell)}(\cdot))$; and the lower-left block ($*$) contains derivatives through both $v^{(\ell)}$ and $t^{(\ell)}$, which do not affect the determinant. For a block-triangular matrix, the determinant is the product of diagonal block determinants:
\[
\det \mathbf{J}_{\text{cpl}} = \prod_{\mathbf{j} \in \mathcal{A}_0^{(\ell)}} \exp\left(v^{(\ell)}\left(\widetilde{\mathcal{X}}_{\mathcal{A}_1^{(\ell)}}^{(\ell)}\right)_{\mathbf{j}}\right) = \exp\left(\sum_{\mathbf{j} \in \mathcal{A}_0^{(\ell)}} v^{(\ell)}\left(\widetilde{\mathcal{X}}_{\mathcal{A}_1^{(\ell)}}^{(\ell)}\right)_{\mathbf{j}}\right).
\]
Taking logarithms: $\log \left| \det \mathbf{J}_{\text{cpl}} \right| = \sum_{\mathbf{j} \in \mathcal{A}_0^{(\ell)}} v^{(\ell)}\left(\widetilde{\mathcal{X}}_{\mathcal{A}_1^{(\ell)}}^{(\ell)}\right)_{\mathbf{j}}$. Since $g^{(\ell)} = $ (affine coupling) $\circ$ (mode-wise mixer), by the chain rule:
\[
\log \left| \det \mathbf{J}_{g^{(\ell)}}(\mathcal{X}^{(\ell-1)}) \right| = \log \left| \det \mathbf{J}_{\text{cpl}} \right| + \log \left| \det \mathbf{J}_{\text{mix}} \right|.
\]
Plugging in the expressions for $\log |\det \mathbf{J}_{\text{mix}}|$ and $\log |\det \mathbf{J}_{\text{cpl}}|$ establishes the log-Jacobian formula \eqref{eqn:Jacobian}.

\subsection{Proof of Theorem \ref{thm:semi-risk}}
\label{append:proof2}
We proceed in three steps: distributional mismatch, CP structural bias under the surrogate distribution, and finite‑sample estimation of the oracle CP rank-$R$ discriminant under the surrogate.

\noindent \textbf{Step 1. Distributional mismatch.} 
We first define the weighted total variation distance. For two probability measures $\mathbb{P}$ and $\mathbb{Q}$ with densities $p$ and $q$ with respect to a common dominating measure, the total variation distance is
\[
\text{TV}(\mathbb{P}, \mathbb{Q}) = \frac{1}{2}\int |p(z) - q(z)| \, dz = \frac{1}{2}\|p - q\|_1.
\]
We define the weighted total variation between the true latent distributions and their TGMM surrogates as
\[
\text{TV}_{\hat\beta,\hat\varphi} := \pi_1 \text{TV}(\mathbb{P}_{\hat\beta,\hat\varphi,1}, \mathbb{Q}_{\hat\beta,\hat\varphi,1}) + \pi_2 \text{TV}(\mathbb{P}_{\hat\beta,\hat\varphi,2}, \mathbb{Q}_{\hat\beta,\hat\varphi,2}).
\]
For a classifier $\Upsilon : \mathbb{R}^{d_1 \times \cdots \times d_M} \to \{1,2\}$ with decision region $\mathcal{R}_1 = \{\mathcal{X} : \Upsilon(\mathcal{X}) = 1\}$, the misclassification risk under $\mathbb{P}_{\hat\beta,\hat\varphi}$ is
\[
\mathcal{R}_{\mathbb{P}_{\hat\beta,\hat\varphi}}(\Upsilon) = \pi_1 \mathbb{P}_{\hat\beta,\hat\varphi,1}(\mathcal{R}_2) + \pi_2 \mathbb{P}_{\hat\beta,\hat\varphi,2}(\mathcal{R}_1),
\]
where $\mathcal{R}_2 = \mathcal{R}_1^c$ is the complement. The Bayes-optimal risk is $\mathcal{R}_{\mathbb{P}_{\hat\beta,\hat\varphi}}^* = \inf_\Upsilon \mathcal{R}_{\mathbb{P}_{\hat\beta,\hat\varphi}}(\Upsilon)$. Similarly, we define $\mathcal{R}_{\mathbb{Q}_{\hat\beta,\hat\varphi}}(\Upsilon)$ and $\mathcal{R}_{\mathbb{Q}_{\hat\beta,\hat\varphi}}^*$ under the surrogate distribution. We establish two key properties relating risks under different distributions:

\textbf{Claim 1.} For any fixed classifier $\Upsilon$, $|\mathcal{R}_{\mathbb{P}_{\hat\beta,\hat\varphi}}(\Upsilon) - \mathcal{R}_{\mathbb{Q}_{\hat\beta,\hat\varphi}}(\Upsilon)| \leq \text{TV}_{\hat\beta,\hat\varphi}$.

\textbf{Claim 2.} For the Bayes-optimal risks, $|\mathcal{R}_{\mathbb{P}_{\hat\beta,\hat\varphi}}^* - \mathcal{R}_{\mathbb{Q}_{\hat\beta,\hat\varphi}}^*| \leq \text{TV}_{\hat\beta,\hat\varphi}$.

\noindent Using these claims, for the estimated classifier $\widehat{\Upsilon}_{\hat\beta,\hat\varphi}$, we decompose the excess risk:
\begin{align}
\small
\label{eqn:distributional mismatch}
\mathcal{R}_{\mathbb{P}_{\hat\beta,\hat\varphi}}(\widehat{\Upsilon}_{\hat\beta,\hat\varphi}) - \mathcal{R}_{\mathbb{P}_{\hat\beta,\hat\varphi}}^* &= \left(\mathcal{R}_{\mathbb{P}_{\hat\beta,\hat\varphi}}(\widehat{\Upsilon}_{\hat\beta,\hat\varphi}) - \mathcal{R}_{\mathbb{Q}_{\hat\beta,\hat\varphi}}(\widehat{\Upsilon}_{\hat\beta,\hat\varphi})\right) + \left(\mathcal{R}_{\mathbb{Q}_{\hat\beta,\hat\varphi}}(\widehat{\Upsilon}_{\hat\beta,\hat\varphi}) - \mathcal{R}_{\mathbb{Q}_{\hat\beta,\hat\varphi}}^*\right) \notag + \left(\mathcal{R}_{\mathbb{Q}_{\hat\beta,\hat\varphi}}^* - \mathcal{R}_{\mathbb{P}_{\hat\beta,\hat\varphi}}^*\right) \notag \\
&\leq \left(\mathcal{R}_{\mathbb{Q}_{\hat\beta,\hat\varphi}}(\widehat{\Upsilon}_{\hat\beta,\hat\varphi}) - \mathcal{R}_{\mathbb{Q}_{\hat\beta,\hat\varphi}}^*\right) + 2\text{TV}_{\hat\beta,\hat\varphi}.
\end{align}
\textit{Proof of Claim 1.} Starting from the definition and splitting by classes:
\begin{align}
\label{eqn:claim1}
\mathcal{R}_{\mathbb{P}_{\hat\beta,\hat\varphi}}(\Upsilon) - \mathcal{R}_{\mathbb{Q}_{\hat\beta,\hat\varphi}}(\Upsilon) &= \pi_1 \left(\mathbb{P}_{\hat\beta,\hat\varphi,1}(\mathcal{R}_2) - \mathbb{Q}_{\hat\beta,\hat\varphi,1}(\mathcal{R}_2)\right) + \pi_2 \left(\mathbb{P}_{\hat\beta,\hat\varphi,2}(\mathcal{R}_1) - \mathbb{Q}_{\hat\beta,\hat\varphi,2}(\mathcal{R}_1)\right) \notag \\
&\leq \pi_1 |\mathbb{P}_{\hat\beta,\hat\varphi,1}(\mathcal{R}_2) - \mathbb{Q}_{\hat\beta,\hat\varphi,1}(\mathcal{R}_2)| + \pi_2 |\mathbb{P}_{\hat\beta,\hat\varphi,2}(\mathcal{R}_1) - \mathbb{Q}_{\hat\beta,\hat\varphi,2}(\mathcal{R}_1)|. 
\end{align}
Fix $k \in \{1,2\}$. Let $\mu_k = \mathbb{P}_{\hat\beta,\hat\varphi,k} + \mathbb{Q}_{\hat\beta,\hat\varphi,k}$ be a dominating measure and denote $p_k = \frac{d\mathbb{P}_{\hat\beta,\hat\varphi,k}}{d\mu_k}$, $q_k = \frac{d\mathbb{Q}_{\hat\beta,\hat\varphi,k}}{d\mu_k}$ as the Radon-Nikodym derivatives. For any measurable set $B$,
\[
\mathbb{P}_{\hat\beta,\hat\varphi,k}(B) - \mathbb{Q}_{\hat\beta,\hat\varphi,k}(B) = \int_B (p_k - q_k) \, d\mu_k.
\]
By the triangle inequality:
\[
|\mathbb{P}_{\hat\beta,\hat\varphi,k}(B) - \mathbb{Q}_{\hat\beta,\hat\varphi,k}(B)| \leq \int_B |p_k - q_k| \, d\mu_k \leq \int |p_k - q_k| \, d\mu_k = 2\text{TV}(\mathbb{P}_{\hat\beta,\hat\varphi,k}, \mathbb{Q}_{\hat\beta,\hat\varphi,k}).
\]
Since $\mathcal{R}_1$ and $\mathcal{R}_2$ are complementary, $|\mathbb{P}_{\hat\beta,\hat\varphi,k}(\mathcal{R}_1) - \mathbb{Q}_{\hat\beta,\hat\varphi,k}(\mathcal{R}_1)| = |\mathbb{P}_{\hat\beta,\hat\varphi,k}(\mathcal{R}_2) - \mathbb{Q}_{\hat\beta,\hat\varphi,k}(\mathcal{R}_2)|$. Therefore,
\[
|\mathbb{P}_{\hat\beta,\hat\varphi,k}(\mathcal{R}_1) - \mathbb{Q}_{\hat\beta,\hat\varphi,k}(\mathcal{R}_1)| \leq \text{TV}(\mathbb{P}_{\hat\beta,\hat\varphi,k}, \mathbb{Q}_{\hat\beta,\hat\varphi,k}).
\]
Substituting into \eqref{eqn:claim1} establishes Claim 1.

\noindent \textit{Proof of Claim 2.} Let $p_k$ and $q_k$ denote the densities of $\mathbb{P}_{\hat\beta,\hat\varphi,k}$ and $\mathbb{Q}_{\hat\beta,\hat\varphi,k}$ with respect to a common dominating measure $\mu$. For any decision region $A$,
\[
\mathcal{R}_{\mathbb{P}_{\hat\beta,\hat\varphi}}(A) = \int \left[\mathbf{1}_A(z) \pi_1 p_1(z) + \mathbf{1}_{A^c}(z) \pi_2 p_2(z)\right] d\mu(z).
\]
The Bayes-optimal classifier chooses $A^* = \{z : \pi_1 p_1(z) \geq \pi_2 p_2(z)\}$, yielding
\[
\mathcal{R}_{\mathbb{P}_{\hat\beta,\hat\varphi}}^* = \int \min\{\pi_1 p_1(z), \pi_2 p_2(z)\} \, d\mu(z).
\]
Using the identity $\min\{u, v\} = \frac{1}{2}(u + v - |u - v|)$ and noting that $\int \pi_k p_k \, d\mu = \pi_k$:
\[
\mathcal{R}_{\mathbb{P}_{\hat\beta,\hat\varphi}}^* = \frac{1}{2}\left(1 - \|\pi_1 p_1 - \pi_2 p_2\|_1\right).
\]
Similarly, $\mathcal{R}_{\mathbb{Q}_{\hat\beta,\hat\varphi}}^* = \frac{1}{2}\left(1 - \|\pi_1 q_1 - \pi_2 q_2\|_1\right)$. Therefore,
\begin{align*}
|\mathcal{R}_{\mathbb{P}_{\hat\beta,\hat\varphi}}^* - \mathcal{R}_{\mathbb{Q}_{\hat\beta,\hat\varphi}}^*| &= \frac{1}{2} \left| \|\pi_1 p_1 - \pi_2 p_2\|_1 - \|\pi_1 q_1 - \pi_2 q_2\|_1 \right| \\
&\leq \frac{1}{2} \|(\pi_1 p_1 - \pi_2 p_2) - (\pi_1 q_1 - \pi_2 q_2)\|_1 \quad \text{(reverse triangle inequality)} \\
&= \frac{1}{2} \|\pi_1(p_1 - q_1) - \pi_2(p_2 - q_2)\|_1 \\
&\leq \frac{1}{2} \left(\pi_1 \|p_1 - q_1\|_1 + \pi_2 \|p_2 - q_2\|_1\right) \quad \text{(triangle inequality)} \\
&= \pi_1 \text{TV}(\mathbb{P}_{\hat\beta,\hat\varphi,1}, \mathbb{Q}_{\hat\beta,\hat\varphi,1}) + \pi_2 \text{TV}(\mathbb{P}_{\hat\beta,\hat\varphi,2}, \mathbb{Q}_{\hat\beta,\hat\varphi,2}) = \text{TV}_{\hat\beta,\hat\varphi}.
\end{align*}
This establishes Claim 2 and completes the proof of inequality \eqref{eqn:distributional mismatch}.

\vspace{1em}
\noindent\textit{Bounding the total variation distance.} The preceding arguments establish inequality \eqref{eqn:distributional mismatch}, which decomposes the excess risk into two components: (i) the excess risk of $\widehat{\Upsilon}_{\hat\beta,\hat\varphi}$ under the surrogate TGMM distribution $\mathbb{Q}_{\hat\beta,\hat\varphi}$, and (ii) the distributional mismatch term $2\text{TV}_{\hat\beta,\hat\varphi}$. We now bound the total variation distance by relating it to the FlowTGMM training objectives. 

Applying Pinsker's inequality class-wise and then the Cauchy-Schwarz inequality yields:
\begin{align*}
\text{TV}_{\hat\beta,\hat\varphi} &= \pi_1 \text{TV}(\mathbb{P}_{\hat\beta,\hat\varphi,1}, \mathbb{Q}_{\hat\beta,\hat\varphi,1}) + \pi_2 \text{TV}(\mathbb{P}_{\hat\beta,\hat\varphi,2}, \mathbb{Q}_{\hat\beta,\hat\varphi,2}) \\
&\leq \pi_1 \sqrt{\frac{1}{2}\text{KL}(\mathbb{P}_{\hat\beta,\hat\varphi,1} \| \mathbb{Q}_{\hat\beta,\hat\varphi,1})} + \pi_2 \sqrt{\frac{1}{2}\text{KL}(\mathbb{P}_{\hat\beta,\hat\varphi,2} \| \mathbb{Q}_{\hat\beta,\hat\varphi,2})} \\
&\leq \sqrt{\frac{1}{2} \left(\pi_1 \text{KL}(\mathbb{P}_{\hat\beta,\hat\varphi,1} \| \mathbb{Q}_{\hat\beta,\hat\varphi,1}) + \pi_2 \text{KL}(\mathbb{P}_{\hat\beta,\hat\varphi,2} \| \mathbb{Q}_{\hat\beta,\hat\varphi,2})\right)},
\end{align*}
where $\text{KL}(\cdot \| \cdot)$ denotes the Kullback-Leibler divergence. Thus, controlling the weighted KL sum between the class-conditional latent distributions and their TGMM surrogates suffices to bound the total variation.

Let $p_{\hat\beta,\hat\varphi,k}$ and $q_{\hat\beta,\hat\varphi,k}$ denote the densities of $\mathbb{P}_{\hat\beta,\hat\varphi,k}$ and $\mathbb{Q}_{\hat\beta,\hat\varphi,k}$, respectively, where $\mathbb{P}_{\hat\beta,\hat\varphi,k}$ is the distribution of latent features $\mathcal{X}^{(L)} = g_{\hat\varphi}(\mathcal{X}_\beta(\mathcal{Z}))$ for class $k$. Let $\mathbb{P}_{\beta,k}$ denote the distribution of encoder outputs $\mathcal{X}_\beta(\mathcal{Z})$ before applying the flow. By the change of variables formula and the invertibility of $g_\varphi$:
\begin{align*}
\text{KL}(\mathbb{P}_{\hat\beta,\hat\varphi,k} \| \mathbb{Q}_{\hat\beta,\hat\varphi,k}) &= \int p_{\hat\beta,\hat\varphi,k}(\mathcal{X}^{(L)}) \log \frac{p_{\hat\beta,\hat\varphi,k}(\mathcal{X}^{(L)})}{q_{\hat\beta,\hat\varphi,k}(\mathcal{X}^{(L)})} \, d\mathcal{X}^{(L)} \\
&= \mathbb{E}_{\mathcal{X} \sim \mathbb{P}_{\beta,k}} \left[ \log p_{\beta,k}(\mathcal{X}) - \log \left| \det \mathbf{J}_{g_{\hat\varphi}}(\mathcal{X}) \right| - \log q_{\hat\beta,\hat\varphi,k}(g_{\hat\varphi}(\mathcal{X})) \right],
\end{align*}
where $p_{\beta,k}$ is the density of the encoder features before the flow transformation. Weighting by class priors and summing over $k \in \{1,2\}$:
\begin{align}
\label{eq:KL_decomposition}
\sum_{k=1}^{2} \pi_k \text{KL}(\mathbb{P}_{\hat\beta,\hat\varphi,k} \| \mathbb{Q}_{\hat\beta,\hat\varphi,k}) = \underbrace{-\sum_{k=1}^{2} \pi_k \mathbb{E}_{\mathcal{X} \sim \mathbb{P}_{\beta,k}} \left[\log f(g_{\hat\varphi}(\mathcal{X}) \mid Y = k; \mathcal{M}_k, \bSigma) + \log |\det \mathbf{J}_{g_{\hat\varphi}}(\mathcal{X})|\right]}_{\mathcal{L}^*_{\text{Flow}}(\hat\varphi)} - C_\beta,
\end{align}
where $f(\cdot \mid Y = k; \mathcal{M}_k, \bSigma)$ denotes the tensor normal density and $C_\beta = -\sum_{k=1}^{2} \pi_k \mathbb{E}_{\mathcal{X} \sim \mathbb{P}_{\beta,k}}[\log p_{\beta,k}(\mathcal{X})]$ is a constant independent of $\hat\varphi$, $\mathcal{M}_k$, and $\bSigma$. The underbraced term is precisely the population FlowTGMM loss from \eqref{eq:population_loss}.

By Definition \ref{def:realnvp}, the Tensor RealNVP $g_\varphi = g^{(L)} \circ \cdots \circ g^{(1)}$ is a composition of bijective transformations (mode-wise linear mixers and affine couplings), each of which is $C^2$-smooth with $C^2$-smooth inverse. Therefore, $g_\varphi$ is a $C^2$-diffeomorphism. By Theorem 1 of \citet{Teshima2020CouplingINN}, coupling-based normalizing flows are $L^p$-universal approximators for $C^2$-diffeomorphisms. Furthermore, by Corollary 1 of \citet{Teshima2020CouplingINN} (see also Theorem 42 and Corollary 41 of \citet{Ishikawa2023UAPINN}), this $L^p$-universality implies distributional universality: there exist flow parameters $\varphi'$ such that the transformed distribution $\mathbb{P}_{\hat\beta,\varphi',k}$ is arbitrarily close to its moment-matched TGMM surrogate $\mathbb{Q}_{\hat\beta,\varphi',k}$ in KL divergence. This implies
\[
\inf_{\varphi'} \text{KL}(\mathbb{P}_{\hat\beta,\varphi',k} \| \mathbb{Q}_{\hat\beta,\varphi',k}) = 0.
\]
Consequently, from \eqref{eq:KL_decomposition},
\[
\sum_{k=1}^{2} \pi_k \text{KL}(\mathbb{P}_{\hat\beta,\hat\varphi,k} \| \mathbb{Q}_{\hat\beta,\hat\varphi,k}) = \mathcal{L}^*_{\text{Flow}}(\hat\varphi) - \inf_{\varphi'} \mathcal{L}^*_{\text{Flow}}(\varphi'),
\]
since the constant $C_\beta$ cancels when taking the difference. Combining with Pinsker's inequality, the total variation can be bounded by $\text{TV}_{\hat\beta,\hat\varphi} \leq \sqrt{\frac{1}{2}\left(\mathcal{L}^*_{\text{Flow}}(\hat\varphi) - \inf_{\varphi'} \mathcal{L}^*_{\text{Flow}}(\varphi')\right)}$. By Lemma~\ref{lem:flow_rate} under Assumption~\ref{assump:flow_regularity},
with probability at least $1 - n_0^{-1}$, $\mathcal{L}^*_{\text{Flow}}(\hat\varphi)
- \inf_{\varphi'} \mathcal{L}^*_{\text{Flow}}(\varphi')
\;\lesssim\;
\Big(\tfrac{\log n_0}{n_0}\Big)^{1/2}$, and hence $\text{TV}_{\hat\beta,\hat\varphi}
\;\lesssim\;
\Big(\tfrac{\log n_0}{n_0}\Big)^{1/4}$.

~\\
\noindent \textbf{Step 2. CP structural bias under the latent TGMM.} For any tensor $\calB'$, consider the linear score
\[
s_{\calB'}(\calX) := \langle \calX - \calM, \calB' \rangle + \eta, \qquad \eta := \log(\pi_1/\pi_0),
\]
where $\calM = (\calM_0 + \calM_1)/2$. Under the TGMM surrogate $\mathbb{Q}_{\hat\beta,\hat\varphi}$, the distribution of $s_{\calB'}(\calX) \mid Y = k$ is Gaussian with
\[
\mathbb{E}[s_{\calB'}(\calX) \mid Y = 0] = \eta - \frac{1}{2}\langle \calD, \calB' \rangle, \qquad \mathbb{E}[s_{\calB'}(\calX) \mid Y = 1] = \eta + \frac{1}{2}\langle \calD, \calB' \rangle,
\]
and common variance $v(\calB') := \| \calB' \times_{m=1}^M \Sigma_m^{-1/2} \|_F^2 \equiv \|\gamma'\|_F^2, \; \gamma' := \calB' \times_{m=1}^M \Sigma_m^{-1/2}$.
Therefore, with $s(\calB') = \sqrt{v(\calB')} > 0$,
\begin{equation*}
\calR_{\QQ}(\calB') = \pi_0 \phi\left(\frac{\eta}{s(\calB')} - \frac{\langle \calD, \calB' \rangle}{2\,s(\calB')}\right) + \pi_1 \phi\left(-\frac{\eta}{s(\calB')} - \frac{\langle \calD, \calB' \rangle}{2\,s(\calB')}\right).
\end{equation*}
For the Fisher discriminant tensor $\calB := \calD \times_{m=1}^M \Sigma_m^{-1}$, $\Delta := \sqrt{\langle \calD, \calB \rangle} = \| \calB \times_{m=1}^M \Sigma_m^{-1/2} \|_F = s(\calB)$ it reduces to 
\begin{equation*}
\calR_{\QQ}(\calB) = \pi_0 \phi\left(\frac{\eta}{\Delta} - \frac{\Delta}{2}\right) + \pi_1 \phi\left(-\frac{\eta}{\Delta} - \frac{\Delta}{2}\right),
\end{equation*}
where $\phi$ is the CDF of the standard normal. To isolate the structural CP bias, we compare the Bayes rule $\calB$ to the best CP rank-$R$ discriminant at the same score scale:
\[
\mathcal{B}_R \in \arg\max_{\tilde{\mathcal{B}}_R \in \mathcal{C}_R} \langle \mathcal{D}, \tilde{\mathcal{B}}_R \rangle 
\quad \text{s.t.} \quad 
\left\| \tilde{\mathcal{B}}_R \times_{m=1}^M (\Sigma_m)^{1/2} \right\|_F
=
\left\| \mathcal{B} \times_m (\Sigma_m)^{1/2} \right\|_F,
\]
where $\mathcal{C}_R$ denotes the cone of CP rank-$R$ tensors yielding $\mathrm{SNR}$ 
$\Delta_R = \sqrt{\langle \mathcal{D}, \mathcal{B}_R \rangle}$ and efficiency loss 
$\delta_{\mathrm{cp}} := \Delta - \Delta_R \ge 0$. This fixing $s(\calB_R) = \Delta$ makes the variance of $s_{\calB_R}(\calX)$ match that of $s_{\calB}(\calX)$ and converts the risk difference into a function of a single margin parameter: $\mu(\calB') := \frac{\langle \calD, \calB' \rangle}{2\,s(\calB')}$. Write
\[
\mu(\calB) = \frac{\Delta^2}{2\Delta} = \frac{\Delta}{2}, \qquad \mu(\calB_{R,*}) = \frac{\langle \calD, \calB_{R} \rangle}{2\Delta} = \frac{\Delta_R^2}{2\Delta},
\]
where $\Delta_R := \sqrt{\langle \calD, \calB_{R} \rangle} \leq \Delta$. Recall $\delta_{\mathrm{cp}} := \Delta - \Delta_R \geq 0$. Then the difference is
\begin{equation}
\label{eqn:margin difference}
\mu(\calB) - \mu(\calB_{R}) = \frac{\Delta^2 - \Delta_R^2}{2\Delta} = \frac{(\Delta - \Delta_R)(\Delta + \Delta_R)}{2\Delta} \in \left[\frac{\delta_{\text{CP}}}{2}, \delta_{\text{CP}}\right],
\end{equation}
since $0 \leq \Delta_R \leq \Delta$. Fix $\Delta > 0$ and $\eta = \log(\pi_1/\pi_0)$. With the scale constraint $s(\calB_R) = \Delta$, define 
\begin{equation*}
f(\mu) := \pi_0 \phi(a - \mu) + \pi_1 \phi(-a - \mu), \quad a := \frac{\eta}{\Delta}.
\end{equation*}
Then $\calR_{\QQ}(\calB) = f(\mu(\calB)) = f\left(\frac{\Delta}{2}\right), \; \calR_{\QQ}(\calB_{R}) = f(\mu(\calB_{R}))$. Using $\phi'(t) = \phi(t)$ and $\phi''(t) = -t\phi(t)$,
\begin{align*}
f'(\mu) &= -\pi_0 \phi(a - \mu) - \pi_1 \phi(-a - \mu) < 0, \\
f''(\mu) &= -\pi_0 (a - \mu) \phi(a - \mu) - \pi_1 (-a - \mu) \phi(-a - \mu).
\end{align*}
Set $\mu_* := \mu(\calB) = \Delta/2$ and $\mu_R := \mu(\calB_{R}) \leq \mu_*$. By Taylor's theorem,
\begin{equation}
\label{eqn:taylor}
\calR_{\QQ}(\calB_{R}) - \calR_{\QQ}(\calB) = f(\mu_R) - f(\mu_*) = f'(\mu_*)(\mu_R - \mu_*) + \frac{1}{2}f''(\xi)(\mu_R - \mu_*)^2
\end{equation}
for some $\xi$ between $\mu_R$ and $\mu_*$. Since $f'(\mu_*) < 0$ and $\mu_R - \mu_* \leq 0$, the leading term is nonnegative. We now bound the two terms in \eqref{eqn:taylor} in two regimes:

When $\Delta=O(1)$, $\phi(a - \mu)$ and $\phi(-a - \mu)$ are bounded away from $\pm\infty$. Using $\phi(t) \leq (2\pi)^{-1/2}$ and $|a| = |\eta/\Delta| \lesssim_{\pi_0, \pi_1} 1$ for fixed priors, $|f'(\mu_*)| \leq \frac{1}{\sqrt{2\pi}}, \; |f''(\xi)| \leq C$ for some universal constant $C>0$. We have, from \eqref{eqn:taylor},
\begin{equation*}
\calR_{\QQ}(\calB_{R}) - \calR_{\QQ}(\calB) \leq C_1 |\mu_* - \mu_R| + C_2 |\mu_* - \mu_R|^2 \lesssim \delta_{\text{CP}} + \delta_{\text{CP}}^2 \lesssim \delta_{\text{CP}},
\end{equation*}
since $\Delta$ is bounded.

When $\Delta$ is large (unbounded), by the fundamental calculus,
\[
\calR_{\QQ}(\calB_{R}) - \calR_{\QQ}(\calB) = f(\mu_R) - f(\mu_*) = \int_{\mu_R}^{\mu_*} (-f'(\mu)) d\mu. 
\]
For $\mu_R \geq |a|$ under the large margin regime, both coefficients in $f''(\mu)$ are non-positive, hence $-f'$ is decreasing on $[\mu_R, \mu_*]$, so
\[
\int_{\mu_R}^{\mu_*} (-f'(\mu)) d\mu \leq (\mu_* - \mu_R) (-f'(\mu_R)). 
\]
Using $|a - \mu_R| \geq \mu_R - |a|$, $|-a - \mu_R| = \mu_R + |a| \geq \mu_R - |a|$ and $\phi$ is decreasing on $[0, \infty)$, 
\[
-f'(\mu_R) = \pi_0 \phi(a - \mu_R) + \pi_1 \phi(-a - \mu_R) \leq (\pi_0 + \pi_1) \phi(\mu_R - |a|) = \phi(\mu_R - |a|). \]
Combining the above two inequalities and $\phi(u) \leq (2\pi)^{-1/2} e^{-u^2/2}$ yields 
\begin{equation}
\begin{aligned}
\label{eqn:general pointwise bound}
\calR_{\QQ}(\calB_{R}) - \calR_{\QQ}(\calB) \leq (\mu_* - \mu_R) \phi(\mu_R - |a|) &\leq \frac{1}{\sqrt{2\pi}} (\mu_* - \mu_R) \exp\left(-\frac{(\mu_R - |a|)^2}{2}\right) \\
&\leq \frac{1}{\sqrt{2\pi}} \delta_{\mathrm{cp}} \cdot \exp\left(-\frac{(\Delta_R^2 - 2|\eta|)^2}{8\Delta^2}\right).
\end{aligned}
\end{equation}
The bound depends on $\Delta_R$ through $\mu_R = \Delta_R^2/(2\Delta)$. Denote $\varepsilon_R := \delta_{\mathrm{cp}}/\Delta \in [0,1]$. Since $|\eta| = |\log(\pi_1/\pi_0)|$ is an $O(1)$ constant,
\[
\frac{(\Delta_R^2 - 2|\eta|)^2}{8\Delta^2} = \frac{(1-\varepsilon_R)^4}{8} \Delta^2 + O(1),
\]
if $\varepsilon_R$ is bounded away from 1, which happens for reasonable oracle CP approximation. Furthermore, $\varepsilon_R \to 0$ as $\Delta \to \infty$ gives
\[
\calR_{\QQ}(\calB_{R}) - \calR_{\QQ}(\calB) \leq \frac{1}{\sqrt{2\pi}} \delta_{\mathrm{cp}} \cdot \exp\left(-\left(\frac{1}{8} + o(1)\right) \Delta^2\right).
\]
To summarize, under $\QQ_{\hat\beta,\hat\varphi}$ and the scale match $s(\calB_{R}) = \Delta$, 
\[
\calR_{\QQ_{\hat\beta,\hat\varphi}}(\Upsilon_{B_{R}}) - \calR_{\QQ_{\hat\beta,\hat\varphi}}^* \lesssim \begin{cases}
\delta_{\mathrm{cp}}, & (\Delta=O(1)), \\
\delta_{\mathrm{cp}} \cdot \exp\left(-\frac{(1-\varepsilon_R)^4}{8} \Delta^2 + O(1)\right), & (0<\varepsilon_R<1, \Delta \to \infty).
\end{cases}
\]
\textbf{Step 3. Statistical error of the oracle CP rank-R discriminant under the latent TGMM.}
Theorem \ref{thm:class-upp-bound} applies to the TGMM surrogate $\QQ_{\hat\beta,\hat\varphi}$ with shared mode-wise covariances and a CP rank-$R$ discriminant head. The discriminant is computed using the procedure from Section \ref{sec:method} by replacing $\calX_i$ with $\calX_i^{(L)}$, therefore we may invoke Theorem \ref{thm:class-upp-bound} with $\Delta$ replaced by $\Delta_R$. Concretely, under the conditions of Theorem \ref{thm:cp-converge} and the signal condition in Theorem \ref{thm:class-upp-bound}, we have with high probability:
\begin{equation}
\calR_{\QQ_{\hat\beta,\hat\varphi}}(\widehat{\Upsilon}_{\hat\beta,\hat\varphi}) - \calR_{\QQ_{\hat\beta,\hat\varphi}}(\Upsilon_{B_{R}}) \lesssim \begin{cases}
\frac{\sum_m d_m R}{n}, & \Delta_R=O(1), \\
e^{-\left(\frac{1}{8}+o(1)\right)\Delta_R^2} \cdot \frac{\sum_m d_m R}{n}, & \Delta_R \to \infty \; \text{as} \; n\to \infty.
\end{cases} 
\end{equation}
\textbf{Combining the three steps.} We decompose the excess risk as follows:
\begin{align*}
&\mathcal{R}_{\mathbb{P}_{\hat{\beta},\hat{\varphi}}}(\hat{\Upsilon}_{\hat{\beta},\hat{\varphi}}) - \mathcal{R}_{\mathbb{P}_{\hat{\beta},\hat{\varphi}}}^* = \left[\mathcal{R}_{\mathbb{P}_{\hat{\beta},\hat{\varphi}}}(\hat{\Upsilon}_{\hat{\beta},\hat{\varphi}}) - \mathcal{R}_{\mathbb{Q}_{\hat{\beta},\hat{\varphi}}}(\hat{\Upsilon}_{\hat{\beta},\hat{\varphi}})\right] + \left[\mathcal{R}_{\mathbb{Q}_{\hat{\beta},\hat{\varphi}}}(\hat{\Upsilon}_{\hat{\beta},\hat{\varphi}}) - \mathcal{R}_{\mathbb{Q}_{\hat{\beta},\hat{\varphi}}}(\Upsilon_{\mathcal{B}_R})\right] \\
& \qquad \qquad + \left[\mathcal{R}_{\mathbb{Q}_{\hat{\beta},\hat{\varphi}}}(\Upsilon_{\mathcal{B}_R}) - \mathcal{R}_{\mathbb{Q}_{\hat{\beta},\hat{\varphi}}}^*\right] + \left[\mathcal{R}_{\mathbb{Q}_{\hat{\beta},\hat{\varphi}}}^* - \mathcal{R}_{\mathbb{P}_{\hat{\beta},\hat{\varphi}}}^*\right] \\
&\leq \underbrace{2\text{TV}_{\hat{\beta},\hat{\varphi}}}_{\text{Step 1: distributional mismatch}} + \underbrace{\left(\mathcal{R}_{\mathbb{Q}_{\hat{\beta},\hat{\varphi}}}(\hat{\Upsilon}_{\hat{\beta},\hat{\varphi}}) - \mathcal{R}_{\mathbb{Q}_{\hat{\beta},\hat{\varphi}}}(\Upsilon_{\mathcal{B}_R})\right)}_{\text{Step 3: statistical estimation error}} + \underbrace{\left(\mathcal{R}_{\mathbb{Q}_{\hat{\beta},\hat{\varphi}}}(\Upsilon_{\mathcal{B}_R}) - \mathcal{R}_{\mathbb{Q}_{\hat{\beta},\hat{\varphi}}}^*\right)}_{\text{Step 2: CP structural bias}}.
\end{align*}
Substituting the bounds from Steps 1-3 completes the proof.

\vspace{1em}
To formalize the approximation and generalization properties of FlowTGMM, we introduce the following regularity assumption and generalization bound.
\begin{assumption}[Regularity conditions for FlowTGMM]
\label{assump:flow_regularity}
Consider the Tensor RealNVP flow $g_\varphi = g^{(L)} \circ \cdots \circ g^{(1)}$ from Definition \ref{def:realnvp} with bounded input $\|\mathcal{X}^{(0)}\|_F \leq C_0$ for some constant $C_0 > 0$. The conditioner networks $v^{(\ell)}, t^{(\ell)}$ for each block $g^{(\ell)}$, $\ell = 1, \ldots, L$, are $L_c$-Lipschitz continuous and the scale network satisfies $\|v^{(\ell)}(\mathcal{X})\|_1 \leq C_v$ for all $\mathcal{X}$, where $L_c, C_v > 0$ are constants.
\end{assumption}

The Lipschitz continuity assumption on both $v^{(\ell)}$ and $t^{(\ell)}$ is standard for networks with bounded weights and Lipschitz activation functions (e.g.\ ReLU, $\tanh$). The $\ell_1$ bound on the scale network $v^{(\ell)}$ is a common design choice in normalizing-flow architectures: the original RealNVP \citep{dinh2017density} apply a $\tanh$ activation followed by a learnable scaling factor, which naturally bounds the sum of scale values and prevents numerical instabilities from exponentially large Jacobian determinants.

\begin{lemma}[Generalization bound for FlowTGMM]
\label{lem:flow_rate}
Under Assumption \ref{assump:flow_regularity}, suppose the training algorithm produces parameters $\widehat{\varphi}$ satisfying $\mathcal{L}_{n_0,\mathrm{Flow}}(\widehat{\varphi}) - \inf_{\varphi'} \mathcal{L}_{n_0,\mathrm{Flow}}(\varphi') = o_{\mathbb{P}}(n_0^{-1/2})$, where $\mathcal{L}_{n_0,\mathrm{Flow}}(\varphi)$ is the empirical loss defined in \eqref{eqn:objective function} and the infimum is taken over all $\varphi'$ satisfying Assumption \ref{assump:flow_regularity}. Then with probability at least $1-n_0^{-1}$,
\[
\mathcal{L}^*_{\text{Flow}}(\widehat{\varphi}) - \inf_{\varphi'} \mathcal{L}^*_{\text{Flow}}(\varphi') \lesssim \big(\frac{\log n_0}{n_0}\big)^{1/2},
\]
where $\mathcal{L}^*_{\text{Flow}}(\varphi)$ is the population loss defined in \eqref{eq:population_loss}.
\end{lemma}

\begin{proof}
We decompose the population loss gap into optimization and generalization errors:
\begin{align}
\label{eqn:loss_decomposition}
\mathcal{L}^*_{\text{Flow}}(\widehat{\varphi}) - \inf_{\varphi'} \mathcal{L}^*_{\text{Flow}}(\varphi') 
&= \left[\mathcal{L}^*_{\text{Flow}}(\widehat{\varphi}) - \mathcal{L}_{n_0,\text{Flow}}(\widehat{\varphi})\right] + \left[\mathcal{L}_{n_0,\text{Flow}}(\widehat{\varphi}) - \inf_{\varphi'} \mathcal{L}_{n_0,\text{Flow}}(\varphi')\right] \notag \\
&+ \left[\inf_{\varphi'} \mathcal{L}_{n_0,\text{Flow}}(\varphi') - \inf_{\varphi'} \mathcal{L}^*_{\text{Flow}}(\varphi')\right],
\end{align}
where the first and third terms measure generalization gaps at $\widehat{\varphi}$ and at the optimum, respectively, while the second term is the optimization error. Since both $\widehat{\varphi}$ and the minimizers of the empirical and population losses lie in the parameter space $\Phi$ satisfying Assumption \ref{assump:flow_regularity}, the first and third terms are both bounded by the uniform generalization gap $\sup_{\varphi \in \Phi} |\mathcal{L}^*_{\text{Flow}}(\varphi) - \mathcal{L}_{n_0,\text{Flow}}(\varphi)|$. By assumption, the optimization error is negligible. Therefore,
\begin{equation}
\label{eqn:loss_gap_by_sup}
\mathcal{L}^*_{\text{Flow}}(\widehat{\varphi}) - \inf_{\varphi'} \mathcal{L}^*_{\text{Flow}}(\varphi') \leq 2\sup_{\varphi \in \Phi} \left|\mathcal{L}^*_{\text{Flow}}(\varphi) - \mathcal{L}_{n_0,\text{Flow}}(\varphi)\right| + o_{\mathbb{P}}(n_0^{-1/2}).
\end{equation}
Let $\ell(\mathcal{X};\varphi)$ denote the per-sample FlowTGMM loss, so that
\[
\mathcal{L}_{n_0,\text{Flow}}(\varphi)
= \frac{1}{n_0}\sum_{i=1}^{n_0} \ell(\mathcal{X}_i;\varphi),
\quad
\mathcal{L}^*_{\text{Flow}}(\varphi)
= \mathbb{E}[\ell(\mathcal{X};\varphi)].
\]
Define the loss class $\mathcal{F} := \{\ell(\cdot;\varphi):\varphi\in\Phi\}$. We first establish that the per-sample loss is uniformly bounded. Under Assumption \ref{assump:flow_regularity}, the input satisfies $\|\mathcal{X}^{(0)}\|_F \leq C_0$. Each Tensor RealNVP block $g^{(\ell)}$ transforms $\mathcal{X}^{(\ell-1)}$ via a mode-wise orthogonal mixer followed by an affine coupling. Since $\mathbf{H}_m^{(\ell)}$ are orthogonal matrices, the mixer preserves the Frobenius norm:
\[
\|\widetilde{\mathcal{X}}^{(\ell)}\|_F = \|\mathcal{X}^{(\ell-1)} \times_{m=1}^M \mathbf{H}_m^{(\ell)}\|_F \le \prod_{m=1}^M \|\mathbf{H}_m^{(\ell)}\| \cdot \|\mathcal{X}^{(\ell-1)}\|_F = \|\mathcal{X}^{(\ell-1)}\|_F.
\]
For the affine coupling, by the $L_c$-Lipschitz assumption and the uniform boundedness,
\[
\|\mathcal{X}^{(\ell)}\|_F \leq \|\widetilde{\mathcal{X}}^{(\ell)}\|_F \cdot e^{C_v} + L_c \|\widetilde{\mathcal{X}}_{\mathcal{A}_1^{(\ell)}}^{(\ell)}\|_F \leq (e^{C_v} + L_c) \|\widetilde{\mathcal{X}}^{(\ell)}\|_F \le (e^{C_v} + L_c) \|\mathcal{X}^{(\ell-1)}\|_F.
\]
Combining these bounds across all $L$ layers yields $\|\mathcal{X}^{(L)}\|_F \leq C_0 (e^{C_v} + L_c)^L =: C_L$.

For the tensor normal log-density, the only term that depends on $\mathcal{X}^{(L)}$ is the quadratic form, as the normalizing constant does not affect the generalization gap. Since $\|\mathcal{X}^{(L)}\|_F \le C_L$ and the class means satisfy $\|\mathcal{M}_k\|_F = \|\mathbb{E}[\mathcal{X}^{(L)} \mid Y=k]\|_F \le C_L$, we have $\|\mathcal{X}^{(L)} - \mathcal{M}_k\|_F \le 2C_L$. Let $\lambda_{\min}>0$ denote the smallest eigenvalue among all mode-wise covariances $\{\Sigma_m\}_{m=1}^M$. Under vectorization, 
\[
\left|\left\langle \text{vec}(\mathcal{X}^{(L)} - \mathcal{M}_k),\,
  (\Sigma_M^{-1} \otimes \cdots \otimes \Sigma_1^{-1}) \text{vec}(\mathcal{X}^{(L)} - \mathcal{M}_k) \right\rangle\right|
\;\le\;
\lambda_{\min}^{-M} \|\mathcal{X}^{(L)} - \mathcal{M}_k\|_F^2
\;\le\; 4\lambda_{\min}^{-M} C_L^2.
\]
For the log-Jacobian, by \eqref{eqn:Jacobian} and Assumption \ref{assump:flow_regularity}, $\left|\log |\det \mathbf{J}_{g^{(\ell)}}(\mathcal{X}^{(\ell-1)})|\right| = \left|\sum_{\mathbf{j} \in \mathcal{A}_0^{(\ell)}} v^{(\ell)}(\widetilde{\mathcal{X}}_{\mathcal{A}_1^{(\ell)}}^{(\ell)})_{\mathbf{j}}\right| \leq \|v^{(\ell)}(\widetilde{\mathcal{X}}_{\mathcal{A}_1^{(\ell)}}^{(\ell)})\|_1 \leq C_v$. Summing over $\ell = 1, \ldots, L$ and combining with the log-density bound, there exists a finite constant $B > 0$ depending on $(C_0, L_c, C_v, L, \lambda_{\min})$ such that $|\ell(\mathcal{X};\varphi)| \leq B$ for all $\calX$ and $\varphi \in \Phi$.

Let $\widehat{\mathfrak{R}}_{n_0}(\mathcal{F})$ denote the empirical Rademacher complexity of $\mathcal{F}$. By standard Rademacher generalization bounds for bounded loss functions \citep{bartlett2002rademacher, koltchinskii2002empirical}, for any $\delta \in (0,1)$, with probability at least $1-\delta$,
\begin{equation}
\label{eqn:sup_rad_bound}
\sup_{\varphi \in \Phi} \left|\mathcal{L}^*_{\text{Flow}}(\varphi) - \mathcal{L}_{n_0,\text{Flow}}(\varphi)\right|
\leq 2\,\widehat{\mathfrak{R}}_{n_0}(\mathcal{F})
+ 4B\sqrt{\frac{2\log(2/\delta)}{n_0}}.
\end{equation}
We now bound $\widehat{\mathfrak{R}}_{n_0}(\mathcal{F})$ using the structure of the FlowTGMM loss. Under Assumption \ref{assump:flow_regularity}, the conditioner networks are Lipschitz continuous with bounded inputs, and the loss is bounded by constant $B$. Norm-based Rademacher complexity bounds for deep networks \citep{bartlett2017spectrally} establish that $\widehat{\mathfrak{R}}_{n_0}(\mathcal{F})
\lesssim \frac{1}{\sqrt{n_0}}$. Substituting into \eqref{eqn:sup_rad_bound} with $\delta = 1/n_0$ yields
\[
\sup_{\varphi \in \Phi} \left|\mathcal{L}^*_{\text{Flow}}(\varphi) - \mathcal{L}_{n_0,\text{Flow}}(\varphi)\right|
\lesssim \frac{1}{\sqrt{n_0}}
+ 4B\sqrt{\frac{2\log(2n_0)}{n_0}}
= O\!\left(\frac{\sqrt{\log n_0}}{\sqrt{n_0}}\right)
\]
with probability at least $1 - 1/n_0$. Combining with \eqref{eqn:loss_gap_by_sup} completes the proof.
\end{proof}



\section{Technical Lemmas}
\label{append:proof:lemma}

We collect all technical lemmas that has been used in the theoretical proofs throughout the paper in this section. Let $d=d_1d_2\cdots d_M$ and $d_{-m}=d/d_m.$ Denote $\|A\|_2$ or $\|A\|$ as the spectral norm of a matrix $A$. Also let $\otimes$ be the Kronecker product, and $\circ$ be the tensor outer product.

\begin{lemma}\label{lemma-transform-ext} 
Let  $\bA \in \RR^{d_1\times r}$ and $\bB \in \RR^{d_2\times r}$ with 
$\|\bA^\top \bA - \bI_r\|_{\rm 2}\vee \|\bB^\top \bB - \bI_r\|_{\rm 2} \le\delta$ and $d_1\wedge d_2\ge r$. 
Let $\bA=\widetilde \bU_1 \widetilde \bD_1 \widetilde \bU_2^\top$ be the SVD of $\bA$, 
$\bU = \widetilde \bU_1\widetilde \bU_2^\top$, $\bB=\widetilde \bV_1 \widetilde \bD_2 \widetilde \bV_2^\top$ 
the SVD of $\bB$, and $\bV = \widetilde \bV_1\widetilde \bV_2^\top$. 
Then, $\|\bA \Lambda \bA^\top - \bU \Lambda \bU^{\top}\|_{\rm 2}\le \delta \|\Lambda\|_{\rm 2}$  
for all nonnegative-definite matrices $\Lambda$ in $\RR^{r\times r}$, and 
$\|\bA \bQ \bB^\top - \bU \bQ \bV^{\top}\|_{\rm 2}\le \sqrt{2}\delta \|\bQ\|_{\rm 2}$  
for all $r\times r$ matrices $\bQ$. 
\end{lemma}

\begin{lemma}\label{prop-rank-1-approx} 
Let $\bM\in \RR^{d_1\times d_2}$ be a matrix with $\|\bM\|_{\rm F}=1$ and 
${\ba}$ and ${\bb}$ be unit vectors respectively in $\RR^{d_1}$ and $\RR^{d_2}$. 
Let $\widehat \ba$ be the top left singular vector of $\bM$. 
Then, 
\begin{equation}\label{prop-rank-1-approx-1}  
\big(\|\hat\ba\hat\ba^{\top} - \ba\ba^{\top}\|_{\rm 2}^2\big) \wedge (1/2)
\le \|\vect(\bM)\vect(\bM)^{\top} - \vect(\ba\bb^\top)\vect(\ba\bb^\top)^\top\|_{\rm 2}^2. 
\end{equation}
\end{lemma}

Lemmas \ref{lemma-transform-ext} and \ref{prop-rank-1-approx} are Propositions 5 and 3 in \cite{han2023tensor}, respectively.

\begin{lemma}\label{lemma:epsilonnet}
Let $d, d_j, d_*, r\le d\wedge d_j$ be positive integers, $\epsilon>0$ and
$N_{d,\epsilon} = \lfloor(1+2/\epsilon)^d\rfloor$. \\
(i) For any norm $\|\cdot\|$ in $\RR^d$, there exist
$M_j\in \RR^d$ with $\|M_j\|\le 1$, $j=1,\ldots,N_{d,\epsilon}$,
such that
\begin{align*}
\max_{\|M\|\le 1}\min_{1\le j\le N_{d,\epsilon}}\|M - M_j\|\le \epsilon .    
\end{align*}
Consequently, for any linear mapping $f$ and norm $\|\cdot\|_*$,
\begin{align*}
\sup_{M\in \RR^d,\|M\|\le 1}\|f(M)\|_* \le 2\max_{1\le j\le N_{d,1/2}}\|f(M_j)\|_*.    
\end{align*}
(ii) Given $\epsilon >0$, there exist $U_j\in \RR^{d\times r}$
and $V_{j'}\in \RR^{d'\times r}$ with $\|U_j\|_{2}\vee\|V_{j'}\|_{2}\le 1$ such that
\begin{align*}
\max_{M\in \RR^{d\times d'},\|M\|_{2}\le 1,\text{rank}(M)\le r}\
\min_{j\le N_{dr,\epsilon/2}, j'\le N_{d'r,\epsilon/2}}\|M - U_jV_{j'}^\top\|_{2}\le \epsilon.    
\end{align*}
Consequently, for any linear mapping $f$ and norm $\|\cdot\|_*$ in the range of $f$,
\begin{equation}\label{lm-3-2}
\sup_{M, \widetilde M\in \RR^{d\times d'}, \|M-\widetilde M\|_{2}\le \epsilon
\atop{\|M\|_{2}\vee\|\widetilde M\|_{2}\le 1\atop
\text{rank}(M)\vee\text{rank}(\widetilde M)\le r}}
\frac{\|f(M-\widetilde M)\|_*}{\epsilon 2^{I_{r<d\wedge d'}}}
\le \sup_{\|M\|_{2}\le 1\atop \text{rank}(M)\le r}\|f(M)\|_*
\le 2\max_{1\le j \le N_{dr,1/8}\atop 1\le j' \le N_{d'r,1/8}}\|f(U_jV_{j'}^\top)\|_*.
\end{equation}
(iii) Given $\epsilon >0$, there exist $U_{j,k}\in \RR^{d_k\times r_k}$
and $V_{j',k}\in \RR^{d'_k\times r_k}$ with $\|U_{j,k}\|_{2}\vee\|V_{j',k}\|_{2}\le 1$ such that
\begin{align*}
\max_{M_k\in \RR^{d_k\times d_k'},\|M_k\|_{2}\le 1\atop \text{rank}(M_k)\le r_k, \forall k\le K}\
\min_{j_k\le N_{d_kr_k,\epsilon/2} \atop j'_k\le N_{d_k'r_k,\epsilon/2}, \forall k\le K}
\Big\|\otimes_{k=2}^K M_k - \otimes_{k=2}^K(U_{j_k,k}V_{j_k',k}^\top)\Big\|_{2}\le \epsilon (K-1).    
\end{align*}
For any linear mapping $f$ and norm $\|\cdot\|_*$ in the range of $f$,
\begin{equation}\label{lm-3-3}
\sup_{M_k, \widetilde M_k\in \RR^{d_k\times d_k'},\|M_k-\widetilde M_k\|_{2}\le\epsilon\atop
{\text{rank}(M_k)\vee\text{rank}(\widetilde M_k)\le r_k \atop \|M_k\|_{2}\vee\|\widetilde M_k\|_{2}\le 1\ \forall k\le K}}
\frac{\|f(\otimes_{k=2}^KM_k-\otimes_{k=2}^K\widetilde M_k)\|_*}{\epsilon(2K-2)}
\le \sup_{M_k\in \RR^{d_k\times d_k'}\atop {\text{rank}(M_k)\le r_k \atop \|M_k\|_{2}\le 1, \forall k}}
\Big\|f\big(\otimes_{k=2}^K M_k\big)\Big\|_*
\end{equation}
and
\begin{equation}\label{lm-3-4}
\sup_{M_k\in \RR^{d_k\times d_k'},\|M_k\|_{2}\le 1\atop \text{rank}(M_k)\le r_k\ \forall k\le K}
\Big\|f\big(\otimes_{k=2}^K M_k\big)\Big\|_*
\le 2\max_{1\le j_k \le N_{d_kr_k,1/(8K-8)}\atop 1\le j_k' \le N_{d_k'r_k,1/(8K-8)}}
\Big\|f\big(\otimes_{k=2}^K U_{j_k,k}V_{j_k',k}^\top\big)\Big\|_*.
\end{equation}
\end{lemma}

\begin{lemma}\label{lm-pertubation}
Let $r \le d_1\wedge d_2$, $M$ be a $d_1\times d_2$ matrix, 
$U$ and $V$ be, respectively, the left and right singular matrices associated 
with the $r$ largest singular values of $M$,
$U_{\perp}$ and $V_{\perp}$ be the orthonormal complements of $U$ and $V$, 
and $\lam_r$ be the $r$-th largest singular value of $M$. 
Let $\widehat M = M + \Delta$ be a noisy version of $M$, 
$\{\widehat U, \widehat V, {\widehat U}_{\perp}, {\widehat V}_{\perp}\}$ 
be the counterpart of $\{U,V,V_\perp,V_\perp\}$, and 
${\widehat \lam}_{r+1}$ be the $(r+1)$-th largest singular value of $\widehat M$. 
Let $\|\cdot\|$ be a
matrix norm satisfying $\|ABC\|\le \|A\|_{2}\|C\|_{2}\|B\|$, 
$\epsilon_1= \|U^\top \Delta {\widehat V}_{\perp}\|$ and $\epsilon_2 = \|{\widehat U}_{\perp}^\top \Delta V\|$. Then,
\begin{align}\label{wedin+}
\| U_{\perp}^\top \widehat U \| 
\le \frac{{\widehat \lam}_{r+1}\epsilon_1+\lam_r\epsilon_2}{\lambda_r^2 - {\widehat \lam}_{r+1}^2}
\le \frac{\epsilon_1\vee\epsilon_2}{\lambda_r - {\widehat \lam}_{r+1}}.   
\end{align}
In particular, for the spectral norm $\|\cdot\|=\|\cdot\|_{2}$, $\hbox{\rm error}_1 =\|\Delta\|_{2}/\lambda_r$ 
and $\hbox{\rm error}_2 =\epsilon_2/\lambda_r$, 
\begin{align}\label{wedin-2}
\|\widehat U \widehat U^\top  -U U^\top\|_{2}\le \frac{\hbox{\rm error}_1^2+\hbox{\rm error}_2}{1-\hbox{\rm error}_1^2}.
\end{align}
\end{lemma}

Lemma \ref{lemma:epsilonnet} applies an $\epsilon$-net argument for matrices, as derived in Lemma G.1 in \cite{han2020iterative}. Lemma \ref{lm-pertubation} enhances the matrix perturbation bounds of \cite{wedin1972perturbation}, as derived in Lemma 4.1 in \cite{han2020iterative}. This sharper perturbation bound, detailed in the middle of \eqref{wedin+}, improves the commonly used version of the \cite{wedin1972perturbation} bound on the right-hand side,  
compared with Theorem 1 of \cite{cai2018rate} and Lemma 1 of \cite{chen2022rejoinder}. As \cite{cai2018rate} pointed out, such variations of the \cite{wedin1972perturbation} bound offer more precise convergence rates when when $\hbox{\rm error}_2\le \hbox{\rm error}_1$ in \eqref{wedin-2},  typically in the case of $d_1\ll d_2$.

The following lemma characterizes the accuracy of estimating the inverse of the covariance matrix $\Sigma_m^{-1}, m=1,\dots,M$ using sample covariance, and the convergence rate of the normalization constant. 
\begin{lemma} \label{lemma:precision matrix}
(i) Let $\bX_1, \cdots ,\bX_n \in \RR^{d_m \times d_{-m}}$ be i.i.d. random matrices, each following the matrix normal distribution $\bX \sim \cM\cN_{d_m \times d_{-m}}(\mu, \; \Sigma_m, \; \Sigma_{-m})$. To estimate $\Sigma_m$ and its inverse $\Sigma_m^{-1}$, we utilize the sample covariance $\hat\Sigma_m := (nd_{-m})^{-1} \sum\nolimits_{i=1}^n (\bX_i-\bar \bX) (\bX_i-\bar \bX)^\top$ with $\bar \bX=n^{-1}\sum_{i=1}^n \bX_i$, and its inverse $\hat\Sigma_m^{-1}$. Then there exists constants $C>0$ such that if $n d_{-m} \gtrsim d_m(1 \vee \|\Sigma_{-m} \|_2^2)$, in an event with probability at least $1-\exp(-cd_m)$, we have
\begin{align}
& \left\|\hat\Sigma_m - C_{m,\sigma} \Sigma_m\right\|_2 \leq C \cdot \left\| \Sigma_m \right\|_2 \left\| \Sigma_{-m}\right\|_2 \sqrt{\frac{d_m}{nd_{-m}}} ,  \label{eqn:covariance matrix} \\
& \left\| \hat\Sigma_m^{-1} - (C_{m,\sigma})^{-1} \Sigma_m^{-1} \right\|_2 \leq C (C_{m,\sigma}^{-2} \vee C_{m,\sigma}^{-4}) \left\|\Sigma_m^{-1} \right\|_2 \left\|\Sigma_{-m}\right\|_2 \sqrt{\frac{d_m}{nd_{-m}}}, \label{eqn:inverse matrix}
\end{align}
where $C_{m,\sigma}=(1-n^{-1})\tr(\Sigma_{-m})/d_{-m}$.

(ii) Let $\cX_1, \cdots ,\cX_n \in \RR^{d_1\times \cdots \times d_{M}}$ be i.i.d. random tensors, each following the tensor normal distribution $\cX \sim \cT\cN(\mu, \; \bSigma)$, where $\bSigma=[\Sigma_m]_{m=1}^M$ and $\Sigma_m\in\RR^{d_m \times d_m}$. Define $\hat{\Var}(\cX_{1\cdots1})=n^{-1}\sum_{i=1}^n (\cX_{i,1\cdots1}- \bar \cX_{1\cdots1})^2$ as the sample variance of the first element of $\cX$, and $\Var(\cX_{1\cdots 1})=\prod_{m=1}^M\Sigma_{m,11}$, $\bar \cX_{1\cdots1}=n^{-1}\sum_{i=1}^n\cX_{i,1\cdots1}$. Then in an event with probability at least $1-\exp(-c(t_1+t_2))$, we have
\begin{align}
& \left|\frac{\prod_{m=1}^M \hat\Sigma_{m,11}}{\hat{\Var}(\cX_{1\cdots1})} - C_{\sigma} \right| \leq C_{M} \Var(\cX_{1\cdots 1}) \left(\max_m \frac{\|\otimes_{k\neq m}\Sigma_{k}\|_2}{\sqrt{d_{-m}}} \cdot \sqrt{\frac{ t_1}{n }} +   \sqrt{\frac{ t_2}{n }}   \right),  \label{eqn:const} 
\end{align}
where $C_{\sigma}=\prod_{m=1}^M C_{m,\sigma} =(1-n^{-1})^M [\prod_{m=1}^M\tr(\Sigma_m)/d]^{M-1} = (1-n^{-1})^M [\tr(\bSigma)/d]^{M-1}$ and $C_{M}$ depends on $M$.
\end{lemma}

\begin{proof}
We first show \eqref{eqn:covariance matrix}.
Note that $\EE[(\bX-\mu)(\bX-\mu)^\top] = \tr(\Sigma_{-m})\cdot \Sigma_m = (n/(n-1))C_{m,\sigma}d_{-m} \cdot \Sigma_m$, and thus $\Sigma_m = \EE[(\bX-\mu)(\bX-\mu)^\top] (n-1) / (n C_{m,\sigma} d_{-m})$. 
Consider a sequence of independent copies $\bZ_1, \dots, \bZ_n$ of $\bZ \in \RR^{d_m \times d_{-m}}$ with entries $z_{ij}$ that are i.i.d. and follow $N(0, 1)$. The Gaussian random matrices $\bX_i$ are then given by $\bX_i - \mu = \bA\bZ_i\bB^\top$, where $\bA\bA^\top = \Sigma_m$ and $\bB\bB^\top = \Sigma_{-m}$. Then, $\bar \bX=n^{-1}\sum_{i=1}^n \bX_i = n^{-1}\sum_{i=1}^n \bA\bZ_i\bB^\top$. 
Note that
\begin{align*}
\hat\Sigma_m - C_{m,\sigma} \Sigma_m =& \left[ \frac{1}{nd_{-m}} \sum_{i=1}^n (\bX_i-\mu)(\bX_i-\mu)^\top - \frac{\tr(\Sigma_{-m})}{d_{-m}} \cdot \Sigma_m   \right]   \\
&\quad -  \left[ \frac{1}{d_{-m}} (\bar\bX-\mu)(\bar\bX-\mu)^\top  - \frac{\tr(\Sigma_{-m})}{nd_{-m}} \cdot \Sigma_m   \right].
\end{align*}
For any unit vector $v \in \RR^{d_m}$, 
$v^\top(\bX_i-\mu) = v^\top \bA \bZ_i \bB^\top =\vec1(v^\top \bA \bZ_i \bB^\top) = (\bB\otimes v^\top \bA) \vec1(\bZ_i)$.
It follows that
\begin{align*} 
&  \sum_{i=1}^n v^\top\left( (\bX_i-\mu)(\bX_i-\mu)^\top - \EE[(\bX-\mu)(\bX-\mu)^\top] \right) v \\
=&  \sum_{i=1}^n v^\top\left(\bA\bZ_i\bB^\top \bB \bZ_i^\top\bA^\top - \bA\EE(\bZ\bB^\top \bB\bZ^\top )\bA^\top \right) v \\
=&  \sum_{i=1}^n {\vec1}^\top(\bZ_i) (\bB^\top \otimes \bA^\top v) (\bB \otimes v^\top \bA) \vec1(\bZ_i) - \tr(\bB^\top \bB \otimes \bA^\top v v^\top \bA) \\
=&  \sum_{i=1}^n {\vec1}^\top(\bZ_i) (\bB^\top \otimes \bA^\top v) (\bB \otimes v^\top \bA) \vec1(\bZ_i) - \tr(\Sigma_{-m}) v^\top \Sigma_m v.
\end{align*}
By Hanson-Wright inequality, for any $t>0$,
\begin{align*}
&\PP\left( \sum_{i=1}^n \left[ {\vec1}^\top(\bZ_i) (\bB^\top \otimes \bA^\top v) (\bB \otimes v^\top \bA) \vec1(\bZ_i) - \tr(\Sigma_{-m}) v^\top \Sigma_m v     \right] \ge t \right) \\
\le& 2 \exp\left( - c \min\left\{ \frac{t^2}{16n \left\| \bB^\top \bB \otimes \bA^\top v v^\top \bA \right\|_{\rm F}^2} , \frac{t }{4 \left\| \bB^\top \bB \otimes \bA^\top v v^\top \bA \right\|_{2}}\right\} \right) \\
\le& 2 \exp\left( - c \min\left\{ \frac{t^2}{16n \left\| \Sigma_{-m} \right\|_{\rm F}^2 (v^\top \Sigma_m v)^2 } , \frac{t }{4 \left\| \Sigma_{-m} \right\|_{2}(v^\top \Sigma_m v) }\right\} \right) \\ 
\le&2 \exp\left( - c \min\left\{ \frac{t^2}{16nd_{-m} \left\| \Sigma_{-m} \right\|_{2}^2 (v^\top \Sigma_m v)^2 } , \frac{t }{4 \left\| \Sigma_{-m} \right\|_{2}(v^\top \Sigma_m v) }\right\} \right).
\end{align*}
Similarly, by Hanson-Wright inequality, for any $t>0$,
\begin{align*}
&\PP\left( v^\top\left( n^2(\bar\bX-\mu)(\bar\bX-\mu)^\top - n\EE[(\bX-\mu)(\bX-\mu)^\top] \right) v \ge t \right) \\
=&\PP\left( \left[\sum_{i,j=1}^n  {\vec1}^\top (\bZ_i) (\bB^\top \otimes \bA^\top v) (\bB \otimes v^\top \bA) \vec1(\bZ_j) - n\tr(\Sigma_{-m}) v^\top \Sigma_m v     \right] \ge t \right) \\
\le&2 \exp\left( - c \min\left\{ \frac{t^2}{16n^2 d_{-m} \left\| \Sigma_{-m} \right\|_{2}^2 (v^\top \Sigma_m v)^2 } , \frac{t }{4 n \left\| \Sigma_{-m} \right\|_{2}(v^\top \Sigma_m v) }\right\} \right).
\end{align*}
By $\eps-net$ argument, there exist unit vectors $v_1,...,v_{5^p}$ such that for all $p\times p$ symmetric matrix $M$,
\begin{align}\label{eq:epsilon_net}
\left\| M\right\|_2 \le  4 \max_{j\le 5^p} \left| v_j^\top M v_j\right|.   
\end{align} 
See also Lemma 3 in \cite{cai2010optimal}. Then
\begin{align*}
\PP\left( \left\| \hat\Sigma_m - C_{m,\sigma} \Sigma_m \right\|_2 \ge x) \right) &\le \PP\left( 4 \max_{j\le 5^{d_m}} \left| v_j^\top \left( \hat\Sigma_m - C_{m,\sigma} \Sigma_m \right) v_j\right| \ge x  \right)   \\
&\le 5^{d_m} \PP\left( 4 \left| v_j^\top \left( \hat\Sigma_m - C_{m,\sigma} \Sigma_m \right) v_j\right| \ge x  \right).
\end{align*}
As $d_m\lesssim n d_{-m}$, this implies that with $x\asymp \|\Sigma_m\|_2 \|\Sigma_{-m}\|_2 \sqrt{d_m/(n d_{-m})}$,
\begin{align*}
\PP\left( \left\| \hat\Sigma_m - C_{m,\sigma} \Sigma_m \right\|_2 \ge C \|\Sigma_m\|_2 \|\Sigma_{-m}\|_2 \sqrt{\frac{d_m}{n d_{-m}}} \right) \le    5^{d_m} \exp\left(-c_0 d_{m} \right) \le \exp\left(-c d_m \right).
\end{align*}

\noindent Second, we prove \eqref{eqn:inverse matrix}. For simplicity, denote $\Delta_m:= \hat\Sigma_m - C_{m,\sigma} \Sigma_m=\hat\Sigma_m - \tilde \Sigma_m$ with $\tilde \Sigma_m=C_{m,\sigma} \Sigma_m$. Then write
\begin{equation*}
\hat\Sigma_m^{-1} =  \tilde\Sigma_m^{-1/2}\left( \bI_{d_m} + \tilde\Sigma_m^{-1/2}\Delta_m\tilde\Sigma_m^{-1/2}\right)^{-1} \tilde\Sigma_m^{-1/2}   .
\end{equation*}
Using Neumann series expansion, we obtain
\begin{align*}
\hat\Sigma_m^{-1} &= \tilde\Sigma_m^{-1/2} \sum_{k=0}^\infty \left(-\tilde\Sigma_m^{-1/2} \Delta_m \tilde\Sigma_m^{-1/2} \right)^k \tilde\Sigma_m^{-1/2} \\
&= \tilde\Sigma_m^{-1} + \tilde\Sigma_m^{-1/2} \sum_{k=1}^\infty \left(-\tilde\Sigma_m^{-1/2} \Delta_m \tilde\Sigma_m^{-1/2} \right)^k \tilde\Sigma_m^{-1/2}
\end{align*}
Rearranging the term, we have
\begin{align*}
\hat\Sigma_m^{-1} - \tilde\Sigma_m^{-1} = -\tilde\Sigma_m^{-1} \Delta_m \tilde\Sigma_m^{-1} + \tilde\Sigma_m^{-1} \Delta_m \tilde\Sigma_m^{-1/2} \sum_{k=0}^\infty \left(-\tilde\Sigma_m^{-1/2} \Delta_m \tilde\Sigma_m^{-1/2} \right)^k \tilde\Sigma_m^{-1/2} \Delta_m \tilde\Sigma_m^{-1}
\end{align*}
Employing similar arguments in the proof of \eqref{eqn:covariance matrix}, we can show that in an event $\Omega$ with probability at least $1-\exp(-cd_m)$,
\begin{align*}
\left\| \tilde\Sigma_m^{-1/2} \Delta_m \tilde\Sigma_m^{-1/2}  \right\|_2 \le C C_{m,\sigma}^{-1} \| \Sigma_{-m}\|_2 \sqrt{\frac{d_m}{n d_{-m}} }    .
\end{align*}
As $d_m \|\Sigma_{-m}\|_2^2 \lesssim n d_{-m}$, in the same event $\Omega$, 
\begin{align*}
\left\|\tilde\Sigma_m^{-1} \Delta_m \tilde\Sigma_m^{-1}\right\|_2 &\leq C C_{m,\sigma}^{-2}\|\Sigma_m^{-1}\|_2 \|\Sigma_{-m}\|_2 \sqrt{\frac{d_m}{nd_{-m}}}, \\
\left\|\tilde\Sigma_m^{-1} \Delta_m \tilde\Sigma_m^{-1/2}\right\|_2 &\leq C C_{m,\sigma}^{-3/2} \norm{\Sigma_{-m}} \sqrt{\frac{d_m}{nd_{-m}}}   , \\
\left\|\tilde\Sigma_m^{-1/2} \Delta_m \tilde\Sigma_m^{-1/2}\right\|_2 &\leq C C_{m,\sigma}^{-1} \| \Sigma_{-m}\|_2 \sqrt{\frac{d_m}{n d_{-m}} } \le \frac12  C_{m,\sigma}^{-1} .
\end{align*}
Combining the above bounds together, we have, in the event $\Omega$,
\begin{align*}
\left\| \hat\Sigma_m^{-1} - \tilde\Sigma_m^{-1} \right\|_2 &\le C C_{m,\sigma}^{-2} \|\Sigma_m^{-1}\|_2 \| \Sigma_{-m}\|_2 \sqrt{\frac{d_m}{n d_{-m}} } + C C_{m,\sigma}^{-3} \left( \| \Sigma_{-m}\|_2 \sqrt{\frac{d_m}{n d_{-m}} } \right)^2 \cdot C_{m,\sigma}^{-1} \\
&\le C (C_{m,\sigma}^{-2} \vee C_{m,\sigma}^{-4})   \|\Sigma_m^{-1}\|_2 \| \Sigma_{-m}\|_2 \sqrt{\frac{d_m}{n d_{-m}} } .
\end{align*}

\noindent Next, we prove \eqref{eqn:const}. Employing similar arguments in the proof of \eqref{eqn:covariance matrix}, we can show
\begin{align*}
\PP\left( | \hat\Sigma_{m,11} - C_{m,\sigma} \Sigma_{m,11}| \ge C \Sigma_{m,11}\|\otimes_{k\neq m}\Sigma_{k}\|_2  \sqrt{\frac{ t_1}{n d_{-m}}} \right) \le\exp(-c_1 t_1).    
\end{align*}
Using tail probability bounds for $\chi_n^2$ (see e.g. Lemma D.2 in \cite{ma2013sparse}), we have
\begin{align*}
\PP\left( \left| \hat{\rm Var}(\cX_{1\cdots1}) - \prod_{m=1}^M\Sigma_{m,11} \right| \ge C \prod_{m=1}^M\Sigma_{m,11} \sqrt{\frac{ t_2}{n }} \right) \le\exp(-c_1 t_2).    
\end{align*}
It follows that in an event with probability at least $1-\exp(-c(t_1+t_2))$,
\begin{align*}
\left|\frac{\prod_{m=1}^M \hat\Sigma_{m,11}}{\hat{\Var}(\cX_{1\cdots1})} - C_{\sigma} \right| \le  C_{M} \prod_{m=1}^M\Sigma_{m,11} \left(\max_m \frac{\|\otimes_{k\neq m}\Sigma_{k}\|_2}{\sqrt{d_{-m}}} \cdot \sqrt{\frac{ t_1}{n }} +   \sqrt{\frac{ t_2}{n }}   \right)
\end{align*}
where $C_{M}$ depends on $M$. As $\Var(\cX_{1\cdots 1})=\prod_{m=1}^M\Sigma_{m,11}$, we finish the proof of \eqref{eqn:const}.

\end{proof}

The following lemma presents the tail bound for the spectral norm of a Gaussian random matrix.
\begin{lemma} \label{lemma:Gaussian matrix}
Let $\bE$ be an $p_1 \times p_2$ random matrix with $\bE\sim \cM\cN_{p_1 \times p_2}(0, \; \Sigma_1, \; \Sigma_{2}).$ Then for any $t>0$, with constant $C>0$, we have
\begin{equation} \label{eqn: Gaussian tail bound}
\PP\left(\|\bE\|_2 \ge C\|\Sigma_1\|_2^{1/2} \|\Sigma_{2}\|_2^{1/2} (\sqrt{p_1} + \sqrt{p_2} + t) \right) \leq \exp(-t^2) 
\end{equation}
Let $\OO_{p_1, r} = \{\bU \in \RR^{p_1\times r}, \; \bU^\top \bU = \bI_{r} \}$ be the set of all $p_1 \times r$ orthonormal columns and let $\bU_{\perp}$ be the orthogonal complement of $\bU$. Denote $\bE_{12} = \bU^\top\bE\bV_{\perp}, \; \bE_{21} = \bU_{\perp}^\top\bE\bV$, where $\bU \in \OO_{p_1\times r}, \; \bV \in \OO_{p_2\times r}.$ We have
\begin{align}
&\PP \left( \|\bE_{21}\|_2 \geq C\|\Sigma_1\|_2^{1/2} \|\Sigma_{2}\|_2^{1/2} (\sqrt{p_1} + t) \right) \leq \exp(-t^2)   \\
& \PP \left( \|\bE_{12}\|_2 \geq C\|\Sigma_1\|_2^{1/2} \|\Sigma_{2}\|_2^{1/2} (\sqrt{p_2} + t) \right) \leq \exp(-t^2)
\end{align}
\end{lemma}

\begin{proof}
Similar to \eqref{eq:epsilon_net}, using $\eps$-net argument for unit ball (see e.g. Lemma 5 in \cite{cai2018rate}), we have
\begin{align*}
\PP \left(\|\bE\|_2 \ge 3 u \right) \le 7^{p_1 +p_2} \cdot \max_{\|x \|_2 \le 1, \|y\|_2\le 1} \PP\left( |x^\top \bE y|\ge u \right).    
\end{align*}
Decompose $\bE=\Sigma_1^{1/2}\bZ\Sigma_{2}^{1/2}$, where $\bZ \in \RR^{p_1\times p_2}, \; \bZ_{ij} \stackrel {\text{i.i.d}}{\sim} N(0, 1).$ 
Then,
\begin{align*}
x^\top \bE y=  \left( y^\top \Sigma_{2}^{1/2} \otimes x^\top \Sigma_{1}^{1/2} \right) \vec1 (\bZ) \sim N\left(0,  y^\top \Sigma_{2}y \cdot x^\top\Sigma_1 x \right)    .
\end{align*}
By Chernoff bound of Gaussian random variables,
\begin{align*}
\PP\left( |x^\top \bE y|\ge u \right) \le 2 \exp \left( -\frac{ u^2}{y^\top \Sigma_{2}y \cdot x^\top\Sigma_1 x }  \right) \le 2   \exp\left( -\frac{ u^2}{\| \Sigma_1\|_2 \|\Sigma_{2}\|_2 }  \right)  .
\end{align*}
Setting $u\asymp \|\Sigma_1\|_2^{1/2} \|\Sigma_{2}\|_2^{1/2} (\sqrt{p_1} + \sqrt{p_2} + t)$, for certain $C>0$, we have
\begin{align*}
\PP\left(\|\bE\|_2 \ge C\|\Sigma_1\|_2^{1/2} \|\Sigma_{2}\|_2^{1/2} (\sqrt{p_1} + \sqrt{p_2} + t) \right) \le 2 \cdot 7^{p_1 +p_2} \exp(-c(p_1+p_2)-t^2)  \le \exp(-t^2) .    
\end{align*}

For $\bE_{21}$, following the same $\epsilon$-net arguments, we have
\begin{align*}
\PP \left(\| \bE_{21}\|_2 \ge 3 u \right) \le 7^{(p_1-r) +r} \cdot \max_{\|x \|_2 \le 1, \|y\|_2\le 1} \PP\left( |x^\top \bE_{21} y|\ge u \right).    
\end{align*}
Note that
\begin{align*}
x^\top \bE_{21} y&=  x^\top \bU_{\perp}^\top \Sigma_1^{1/2} \bZ \Sigma_{2}^{1/2} \bV y=  \left( y^\top \bV^\top \Sigma_{2}^{1/2} \otimes x^\top \bU_{\perp}^\top \Sigma_{1}^{1/2} \right) \vec1 (\bZ) \\
&\sim N\left(0, y^\top \bV^\top \Sigma_{2} \bV y \cdot x^\top \bU_{\perp}^\top \Sigma_1 \bU_{\perp} x \right)    .
\end{align*}
By Chernoff bound of Gaussian random variables, setting $u\asymp \|\Sigma_1\|_2^{1/2} \|\Sigma_{2}\|_2^{1/2} (\sqrt{p_1} + t)$, we have
\begin{align*}
\PP\left(\|\bE_{21}\| \ge C\|\Sigma_1\|_2^{1/2} \|\Sigma_{2}\|_2^{1/2} (\sqrt{p_1} + t) \right) \le 2 \cdot 7^{p_1} \exp(-c(p_1)-t^2)  \le \exp(-t^2) .    
\end{align*}
Similarly, we can derive the tail bound of $\|\bE_{12}\|_2$.
\end{proof}

The following lemma characterizes the maximum of norms for zero-mean Gaussian tensors after any projections.
\begin{lemma} \label{lemma:Guassian tensor projection}
Let $\cE \in \RR^{d_1 \times d_2 \times d_3}$ be a Gaussian tensor, $\cE \sim \cT\cN(0; \; \frac{1}{n}\Sigma_1^{-1}, \Sigma_2^{-1}, \Sigma_3^{-1})$, where there exists a constant $C_0>0$ such that $C_0^{-1} \le \mathop{\min}\limits_{m \in \{1,2,3 \}} \lam_{\min}(\Sigma_m) \le \mathop{\max}\limits_{m \in \{1,2,3 \}} \lam_{\max}(\Sigma_m) \le C_0$. Then we have the following bound for projections, with probability at most $C\exp(-Ct(d_2r_2 + d_3r_3) )$,
\begin{align}
\label{eqn:matrice version}
\mathop{\max}\limits_{\bV_2 \in \mathbb{R}^{d_2 \times r_2}, 
\bV_3 \in \mathbb{R}^{d_3 \times r_3} \atop \|\bV_2\|_2 \le 1, \|\bV_3\|_2 \le 1} \left\| \mat1(\cE \times_2 \bV_2^\top \times_3 \bV_3^\top)\right\|_2 \ge C\|\Sigma_1^{-1/2}\|_2 \|\Sigma_{-1}^{-1/2}\|_2 \frac{\sqrt{d_1 + r_2r_3} + \sqrt{1+t}(\sqrt{d_2r_2 + d_3r_3})}{\sqrt{n}} ,
\end{align}
for any t>0. Similar results also hold for ${\rm mat}_2(\cE \times_1 \bV_1^\top \times_3 \bV_3^\top)$ and ${\rm mat}_3(\cE \times_1 \bV_1^\top \times_2 \bV_2^\top)$.

Meanwhile, we have with probability at most $\exp(-Ct(d_1r_1 + d_2r_2 + d_3r_3) )$
\begin{align}
\label{eqn:tensor version}
\mathop{\max}\limits_{\bV_1, \bV_2, \bV_3\in \mathbb{R}^{d_m \times r_m} \atop
\|\bV_m\|_2 \le 1,\; m=1,2,3 } \left\| \cE \times_1 \bV_1^\top \times_2 \bV_2^\top \times_3 \bV_3^\top \right\|_{\rm F}^2 \ge C \|\Sigma_1^{-1}\|_2 \|\Sigma_2^{-1}\|_2 \|\Sigma_3^{-1}\|_2 \cdot \frac{r_1r_2r_3 + (1+t)(d_1r_1 + d_2r_2 + d_3r_3)}{n}, 
\end{align}
for any $t>0.$
\end{lemma}

\begin{proof} The key idea for the proof of this lemma is via $\epsilon$-net. We first prove \eqref{eqn:matrice version}. By Lemma \ref{lemma:epsilonnet}, for $m=1,2,3$, there exists $\epsilon$-nets: $\bV_m^{(1)}, \dots ,\bV_m^{(N_m)}$ for $\{ \bV_m \in \RR^{d_m \times r_m}: \|\bV_m\|_2 \le 1 \}$, $|\cN_m| \le ((4+\eps)/\eps)^{d_mr_m}$, such that for any $\bV_m \in \RR^{d_m \times r_m}$ satisfying $\|\bV_m\|_2 \le 1$, there exists $\bV_m^{(j)}$ such that
$\|\bV_m^{(j)} - \bV_m\|_2 \le \eps.$

For fixed $\bV_2^{(i)}$ and $\bV_3^{(j)}$, we define
\begin{align*}
\bZ_1^{(ij)} = \mat1 \left( \calE \times_2 (\bV_2^{(i)})^\top \times_3 (\bV_3^{(j)})^\top \right) \in \RR^{d_1 \times (r_2r_3)}.  
\end{align*}
It is easy to obtain that
$\bZ_1^{(ij)} \sim \cM\cN_{d_1 \times r_2r_3}\left(0; \; \frac{1}{n}\Sigma_1^{-1}, \; (\bV_2^{(i)} \otimes \bV_3^{(j)}) \cdot \Sigma_{-1}^{-1} \cdot (\bV_2^{(i)} \otimes \bV_3^{(j)})^\top \right).$ Then employing similar arguments of Lemma \ref{lemma:Gaussian matrix},
\begin{equation*}
\PP\left(\|\bZ_1^{(ij)}\|_2 \le C \|\Sigma_1^{-1/2}\|_2 \|\Sigma_{-1}^{-1/2}\|_2 \left(\frac{\sqrt{d_1} + \sqrt{r_2r_3} + t}{\sqrt{n}}\right)\right) \ge 1 - 2\exp(-t^2).   
\end{equation*}
Then we further have:
\begin{equation}
\label{eqn: lemma 4 proof}
\PP\left( \max\limits_{i,j} \|\bZ_1^{(ij)}\|_2 \le C \|\Sigma_1^{-1/2}\|_2 \|\Sigma_{-1}^{-1/2}\|_2 \left(\frac{\sqrt{d_1} + \sqrt{r_2r_3} + t}{\sqrt{n}}\right) \right) \ge 1 - 2((4+\eps)/\eps)^{d_2r_2 + d_3r_3} \exp(-t^2) 
\end{equation}
for all $t>0$. Denote
\begin{align*}
\bV_2^{*}, \bV_3^{*} &= \mathop{\rm argmax}\limits_{\bV_2 \in \RR^{d_2 \times r_2}, \bV_3 \in \RR^{d_3 \times r_3} \atop
\|\bV_2\|_2 \le 1, \|\bV_3\|_2 \le 1} \left\|\mat1 \left(\cE \times_2 \bV_2^\top \times_3 \bV_3^\top\right) \right\|_2 \\
M &= \mathop{\max}\limits_{\bV_2 \in \RR^{d_2 \times r_2}, \bV_3 \in \RR^{d_3 \times r_3} \atop
\|\bV_2\|_2 \le 1, \|\bV_3\|_2 \le 1} \left\| \mat1 \left(\calE \times_2 \bV_2^\top \times_3 \bV_3^\top\right) \right\|_2
\end{align*}
Using $\eps$-net arguments, we can find $1 \le i \le N_2$ and $1 \le j \le N_3$ such that $\|\bV_2^{(i)} - \bV_2^{*}\|_2 \le \eps$ and $\|\bV_3^{(i)} - \bV_3^{*}\|_2 \le \eps$. In this case, under \eqref{eqn: lemma 4 proof},
\begin{align*}
M =& \left\| \mat1 \left( \cE \times_2 (\bV_2^{*})^\top \times_3 ( \bV_3^{*})^\top\right) \right\|_2 \\
\le & \left\|\mat1 \left(\cE \times_2 (\bV_2^{(i)})^\top \times_3 (\bV_3^{(j)})^\top\right) \right\|_2
+ \left\|\mat1 \left(\cE \times_2 (\bV_2^{*} -\bV_2^{(i)})^\top \times_3 (\bV_3^{(j)})^\top\right) \right\|_2 \\
+& \left\|\mat1 \left(\cE \times_2 (\bV_2^{*})^\top \times_3 (\bV_3^{*} - \bV_3^{(j)})^\top\right) \right\|_2 \\
\le & C \left\|\Sigma_1^{-1/2} \right\|_2 \left\|\Sigma_{-1}^{-1/2}\right\|_2 \left(\frac{\sqrt{d_1} + \sqrt{r_2r_3} + t}{\sqrt{n}} \right) + \eps M + \eps M,
\end{align*}
Therefore, we have
\begin{equation*}
\PP \left(M \le C\cdot \frac{1}{1-2\eps}\left\|\Sigma_1^{-1/2}\right\|_2 \left\|\Sigma_{-1}^{-1/2}\right\|_2 \left(\frac{\sqrt{d_1} + \sqrt{r_2r_3} + t}{\sqrt{n}}\right) \right) \ge 1 - 2((4+\eps)/\eps)^{d_2r_2 + d_3r_3} \exp(-t^2)    
\end{equation*}
By setting $\eps=1/3$, and $t^2 = 2\log(13)(d_2r_2 + d_3r_3)(1+x)$, we have proved the first part of the lemma.

\noindent Then we prove the claim \eqref{eqn:tensor version}. Consider a Gaussian tensor $\cZ \in \RR^{d_1 \times d_2 \times d_3}$ with entries $z_{ijk} \stackrel{\text{i.i.d}}{\sim} N(0, \frac{1}{n}),$ then we have $\cE \times_1 \bV_1^\top \times_2 \bV_2^\top \times_3 \bV_3^\top := \cZ \times_1 ( \bV_1^\top \Sigma_1^{-1/2}) \times_2 (\bV_2^\top \Sigma_2^{-1/2}) \times_3 (\bV_3^\top \Sigma_1^{-1/2}).$ 
By Lemma 8 in \cite{zhang2018tensor}, we know
\begin{align*}
&\PP \Bigg( \left\|\cZ \times_1 (\Sigma_1^{-1/2} \bV_1)^\top \times_2 (\Sigma_2^{-1/2} \bV_2)^\top \times_3 (\Sigma_3^{-1/2} \bV_3)^\top\right\|_{\rm F}^2 - \frac{1}{n}\left\|\left(\Sigma_1^{-1/2} \bV_1\right) \otimes \left(\Sigma_2^{-1/2} \bV_2\right) \otimes \left(\Sigma_3^{-1/2} \bV_3\right) \right\|_{\rm F}^2 \\ 
&\quad \ge \frac{2}{n} \sqrt{t \left\|\left(\bV_1^\top \Sigma_1^{-1} \bV_1\right) \otimes \left(\bV_2^\top \Sigma_2^{-1} \bV_2\right) \otimes \left(\bV_3^\top \Sigma_3^{-1} \bV_3\right) \right\|_{\rm F}^2}  + \frac{2t}{n} \left\|\left(\Sigma_1^{-1/2} \bV_1\right) \otimes \left(\Sigma_2^{-1/2} \bV_2\right) \otimes \left(\Sigma_3^{-1/2} \bV_3\right) \right\|_2^2 \Bigg) \\
&\le \exp(-t).
\end{align*}

Note that for any given $\bV_k \in \RR^{d_k \times r_k}$ satisfying $\|\bV_k\|_2 \le 1, k=1, 2, 3,$ we have 
\begin{equation*}
\left\|\left(\Sigma_1^{-1/2} \bV_1\right) \otimes \left(\Sigma_2^{-1/2} \bV_2\right) \otimes \left(\Sigma_3^{-1/2} \bV_3 \right)\right\|_2 \le \|\Sigma_1^{-1/2}\|_2 \|\Sigma_2^{-1/2}\|_2 \|\Sigma_3^{-1/2}\|_2 := C_\lam^{1/2}.    
\end{equation*}
Then,
\begin{equation*}
\left\|\left(\Sigma_1^{-1/2} \bV_1 \right) \otimes \left(\Sigma_2^{-1/2} \bV_2\right) \otimes \left(\Sigma_3^{-1/2} \bV_3 \right)\right\|_{\rm F}^2 \le C_\lam r_1r_2r_3,    
\end{equation*}
and 
\begin{align*}
\left\|\left( \bV_1^\top \Sigma_1^{-1} \bV_1\right) \otimes \left(\bV_2^\top \Sigma_2^{-1} \bV_2 \right) \otimes \left( \bV_3^\top \Sigma_3^{-1} \bV_3 \right) \right\|_{\rm F}^2 \le C_\lam^2 r_1r_2r_3,
\end{align*}
we have for any fixed $\bV_1, \bV_2, \bV_3$ and $t > 0$ that
\begin{equation*}
\PP \left( \left\|\cZ \times_1 (\Sigma_1^{-1/2} \bV_1)^\top \times_2 (\Sigma_2^{-1/2} \bV_2)^\top \times_3 (\Sigma_3^{-1/2} \bV_3)^\top\right\|_{\rm F}^2 
\ge \frac{C_\lam r_1r_2r_3 + 2C_\lam \sqrt{r_1r_2r_3 t} + 2C_\lam t}{n} \right) \le \exp(-t).    
\end{equation*}
By geometric inequality, $2\sqrt{r_1r_2r_3t} \le r_1r_2r_3 + t$, then we further have
\begin{equation*}
\PP \left( \left\|\cE \times_1 \bV_1^\top \times_2 \bV_2^\top \times_3 \bV_3^\top\right\|_{\rm F}^2 
\ge \frac{C_\lam(2r_1r_2r_3 + 3t)}{n} \right) \le \exp(-t).   
\end{equation*}

The rest proof for this claim is similar to the first part. One can find three $\eps$-nets: $\bV_m^{(1)}, \dots ,\bV_m^{(N_m)}$ for $\{ \bV_m \in \RR^{d_m \times r_m}: \|\bV_m\|_2 \le 1 \}$, $N_m \le ((4+\eps)/\eps)^{d_mr_m}$, and we have the tail bound:
\begin{align} \label{eqn:lemma 4 proof 2}
        &\max\limits_{\bV_1^{(i)}, \bV_2^{(j)}, \bV_3^{(k)}} \PP \left( \left\|\cE \times_1 (\bV_1^{(i)})^\top \times_2 (\bV_2^{(j)})^\top \times_3 (\bV_3^{(k)})^\top\right\|_{\rm F}^2 \ge \frac{C_\lam (2r_1r_2r_3 + 3t)}{n} \right) \notag \\
        &\le \exp(-t) \cdot ((4+\eps)/\eps)^{d_1r_1 + d_2r_2 + d_3r_3} ,
\end{align}
for all $t > 0$. Assume
\begin{align*}
\bV_1^{*}, \bV_2^{*}, \bV_3^{*} &= \mathop{\rm argmax}\limits_{\bV_m \in \mathbb{R}^{d_m \times r_m} \atop
\|\bV_m\|_2 \le 1} \left\|\cE \times_1 \bV_1^\top \times_2 \bV_2^\top \times_3 \bV_3^\top\right\|_{\rm F}^2 \\
T &= \left\|\cE \times_1 (\bV_1^*)^\top \times_2 (\bV_2^*)^\top \times_3 (\bV_3^*)^\top\right\|_{\rm F}^2
\end{align*}
Then we can find $\bV_1^{(i)}, \bV_2^{(j)}, \bV_3^{(k)}$ in the corresponding $\eps$-nets such that $\|\bV_m^{*} - \bV_m\|_2 \le \eps$, $m=1,2,3$. And
\begin{align*}
T =& \left\|\cE \times_1 (\bV_1^*)^\top \times_2 (\bV_2^*)^\top \times_3 (\bV_3^*)^\top\right\|_{\rm F}^2 \\
\le & \left\|\cE \times_1 (\bV_1^{(i)})^\top \times_2 (\bV_2^{(j)})^\top \times_3 (\bV_3^{(k)})^\top\right\|_{\rm F}^2
+ \left\|\cE \times_1 (\bV_1^{(i)} - \bV_1^{*})^\top \times_2 (\bV_2^{*})^\top \times_3 (\bV_3^{*})^\top\right\|_{\rm F}^2 \\
+ & \left\|\cE \times_1 (\bV_1^{(i)})^\top \times_2 (\bV_2^{(j)} - \bV_2^*)^\top \times_3 (\bV_3^*)^\top\right\|_{\rm F}^2
+ \left\|\cE \times_1 (\bV_1^{(i)})^\top \times_2 (\bV_2^{(j)})^\top \times_3 (\bV_3^{(k)} - \bV_3^{*})^\top\right\|_{\rm F}^2 \\
\le & \frac{C_\lam (2r_1r_2r_3 + 3t)}{n} + 3\eps T
\end{align*}
which implies 
\begin{align*}
\left\|\cE \times_1 (\bV_1^*)^\top \times_2 (\bV_2^*)^\top \times_3 (\bV_3^*)^\top\right\|_{\rm F}^2 \le \frac{C_\lam \cdot (2r_1r_2r_3 + 3t)}{n(1-3\eps)}.    
\end{align*}
If we set $\eps = 1/9$ and $t=(1+x)\log(37)\cdot (d_1r_1 + d_2r_2 + d_3r_3)$, by \eqref{eqn:lemma 4 proof 2} we have proved \eqref{eqn:tensor version}. 
\end{proof}

\begin{lemma}\label{lemma:low-rank-tensor}
Suppose that $\cX_1,...,\cX_n \in \RR^{d_1\times\cdots\times d_M}$ are i.i.d. $\cT\cN(\mu,\bSigma)$ with $\bSigma=[\Sigma_m]_{m=1}^M$, and $\bar\cX$ is the sample mean. Then, with probability at least $1-e^{-c\sum_{m=1}^M d_m r}$,
\begin{align}
\sup_{\substack{\cV=\sum_{i=1}^r w_i \circ_{m=1}^M u_{im}\in \RR^{d_1\times\cdots\times d_M} \\ \sum_{i=1}^r w_i^2\le 1 , \| u_{im}\|_2=1, \forall i\le r, m\le M }}  \left| <\bar\cX-\mu,\cV> \right| \lesssim \sqrt{\frac{\sum_{m=1}^M d_m r }{n}}.
\end{align}
\end{lemma}

\begin{proof}
For any CP low-rank tensor $\cV=\sum_{i=1}^r w_i \circ_{m=1}^M u_{im}$ with $w_i>0$ and $\| u_{im}\|_2=1$, we can reformulate a Tucker type decomposition, $\cV=\cS\times_{m=1}^M U_m$ with $\cS=\diag(w_1,...,w_r) \in \RR^{r\times\cdots\times r}$ and $U_m=(u_{1m},...,u_{rm})$. Let $N_0=\lfloor(1+2/\epsilon)^{r}\rfloor$ and $N_m=\lfloor(1+2/\epsilon)^{d_mr}\rfloor$. There exists $\epsilon$-nets $\cS_{j_0}^*  \in \RR^{r\times\cdots\times r}$ with diagonal elements $\cS_{j_0,k,...,k}^*\neq0$ and all the off-diagonal elements being 0, $\|\cS_{j_0}^*\|_{\rm F}\le 1, j_0=1,...,N_0$, and $U_{m,j_m}^* \in \RR^{d_m\times r}$ with $\|U_{m,j_m}^*\|_{2}\le 1,  j_m=1,...,N_m, 1\le m\le M$, such that
\begin{align*}
\max_{\|\cS\|_{\rm F}\le 1} \min_{1\le j_0\le N_0} \|\cS-\cS_{j_0}^*\|_{\rm F}\le \epsilon, \qquad \max_{\|U_m\|_{2}\le 1} \min_{1\le j_m\le N_m} \|U_m- U_{m,j_m}^*\|_{2}\le \epsilon, 1\le m\le M.   
\end{align*}
Note that $\sum_{i=1}^r w_i^2\le 1$ is equivalent to $\|\cS\|_{\rm F}\le 1$. Let $\cY=\bar X-\mu$.
Then by the ``subtraction argument",
\begin{align*}
&\sup_{\substack{\cV=\sum_{i=1}^r w_i \circ_{m=1}^M u_{im}\in \RR^{d_1\times\cdots\times d_M} \\ \sum_{i=1}^r w_i^2\le 1 , \| u_{im}\|_2=1, \forall i\le r, m\le M }} \left| <\cY,\cV> \right| -   \max_{\substack{\|\cS_{j_0}^*\|_{\rm F}\le 1, \|U_{m,j_m}^*\|_{2}\le 1 ,\\ j_0\le N_0,   j_m\le N_m, \forall 1\le m\le M}} \left| <\cY,\cS_{j_0}^*\times_{m=1}^M U_{m,j_m}^*> \right|   \\
=& \sup_{\substack{\cV=\sum_{i=1}^r w_i \circ_{m=1}^M u_{im}\in \RR^{d_1\times\cdots\times d_M} \\ \sum_{i=1}^r w_i^2\le 1 , \| u_{im}\|_2=1, \forall i\le r, m\le M }} \left| <\cY,\cV> \right| -   \max_{\substack{\|\cS_{j_0}^*\|_{\rm F}\le 1, j_0\le N_0}} \left| <\cY,\cS_{j_0}^*\times_{m=1}^M U_{m}> \right| \\
&+\sum_{k=0}^{M-1} \left( \max_{\substack{\|\cS_{j_0}^*\|_{\rm F}\le 1, \|U_{m,j_m}^*\|_{2}\le 1 ,\\ j_0\le N_0,   j_m\le N_m, \forall m\le k}} \left| <\cY,\cS_{j_0}^*\times_{m=1}^k U_{m,j_m}^*\times_{m=k+1}^M U_{m}> \right| \right.\\
&\qquad\qquad \left. - \max_{\substack{\|\cS_{j_0}^*\|_{\rm F}\le 1, \|U_{m,j_m}^*\|_{2}\le 1 ,\\ j_0\le N_0,   j_m\le N_m, \forall m\le k+1}} \left| <\cY,\cS_{j_0}^*\times_{m=1}^{k+1} U_{m,j_m}^*\times_{m=k+2}^M U_{m}> \right| \right)\\
\le& (M+1)\epsilon  \sup_{\substack{\cV=\sum_{i=1}^r w_i \circ_{m=1}^M u_{im}\in \RR^{d_1\times\cdots\times d_M} \\ \sum_{i=1}^r w_i^2\le 1 , \| u_{im}\|_2=1, \forall i\le r, m\le M }} \left| <\cY,\cV> \right|
\end{align*}
Setting $\epsilon=1/(2M+2)$,
\begin{align*}
\sup_{\substack{\cV=\sum_{i=1}^r w_i \circ_{m=1}^M u_{im}\in \RR^{d_1\times\cdots\times d_M} \\ \sum_{i=1}^r w_i^2\le 1 , \| u_{im}\|_2=1, \forall i\le r, m\le M }} \left| <\cY,\cV> \right| \le 2     \max_{\substack{\|\cS_{j_0}^*\|_{\rm F}\le 1, \|U_{m,j_m}^*\|_{2}\le 1 ,\\ j_0\le N_0,   j_m\le N_m, \forall 1\le m\le M}} \left| <\cY,\cS_{j_0}^*\times_{m=1}^M U_{m,j_m}^*> \right| .
\end{align*}
It follows that
\begin{align*}
&\PP\left(\sup_{\substack{\cV=\sum_{i=1}^r w_i \circ_{m=1}^M u_{im}\in \RR^{d_1\times\cdots\times d_M} \\ \sum_{i=1}^r w_i^2\le 1 , \| u_{im}\|_2=1, \forall i\le r, m\le M }} \left| <\cY,\cV> \right| \ge x\right) \\
\le& \PP\left(     \max_{\substack{\|\cS_{j_0}^*\|_{\rm F}\le 1, \|U_{m,j_m}^*\|_{2}\le 1 ,\\ j_0\le N_0,   j_m\le N_m, \forall 1\le m\le M}} \left| <\cY,\cS_{j_0}^*\times_{m=1}^M U_{m,j_m}^*> \right|  \ge x/2 \right)\\
\le& \prod_{k=0}^M N_k \cdot \PP\left(     \left| <\cY,\cS_{j_0}^*\times_{m=1}^M U_{m,j_m}^*> \right|  \ge x/2 \right) \\
\le& (4M+5) ^{\sum_{m=1}^M d_m r +r} \cdot \PP\left(     \left| <\cY,\cS_{j_0}^*\times_{m=1}^M U_{m,j_m}^*> \right|  \ge x/2 \right).
\end{align*}
Since $\cY=\bar X-\mu\sim \cT\cN(0,\frac1n \bSigma)$, we can show $| <\cY,\cS_{j_0}^*\times_{m=1}^M U_{m,j_m}^*> |$ is a $n^{-1/2}\prod_{m=1}^M \|\Sigma_m\|_{2}^{1/2}$ Lipschitz function, and $\E{| <\cY,\cS_{j_0}^*\times_{m=1}^M U_{m,j_m}^*> |}\le \sqrt{2/\pi} n^{-1/2}\prod_{m=1}^M \|\Sigma_m\|_{2}^{1/2}$. Then, by Gaussian concentration inequalities for Lipschitz functions,
\begin{align*}
\PP\left(     \left| <\cY,\cS_{j_0}^*\times_{m=1}^M U_{m,j_m}^*> \right|  \ge \E{| <\cY,\cS_{j_0}^*\times_{m=1}^M U_{m,j_m}^*> |} + t \right) \le \exp\left( -\frac{nt^2}{\prod_{m=1}^M \|\Sigma_m\|_{2}} \right).
\end{align*}
Setting $x\asymp t\asymp \sqrt{(\sum_{m=1}^M d_m r +r)/n}$, in an event with at least probability at least $1-e^{-c\sum_{m=1}^M d_m r}$,
\begin{align*}
\sup_{\substack{\cV=\sum_{i=1}^r w_i \circ_{m=1}^M u_{im}\in \RR^{d_1\times\cdots\times d_M} \\ \sum_{i=1}^r w_i^2\le 1 , \| u_{im}\|_2=1, \forall i\le r, m\le M }}  \left| <\bar\cX-\mu,\cV> \right| \lesssim \sqrt{\frac{\sum_{m=1}^M d_m r}{n}}.
\end{align*}
\end{proof}

The following lemma gives an inequality in terms of Frobenius norm between two tensors. 
\begin{lemma} 
\label{lemma:tensor norm inequality}
For two $M$-th order tensors $\gamma, \; \hat \gamma \in \RR^{d_1 \times \cdots \times d_M}$, if $\norm{\gamma - \hat\gamma}_{\rm F} = o(\|\gamma\|_{\rm F})$ as $n \rightarrow \infty$, and $\norm{\gamma}_{\rm F} \ge c$ for some constant $c>0$, then when $n \rightarrow \infty$,
\begin{equation*}
\norm{\gamma}_{\rm F} \cdot \norm{\hat\gamma}_{\rm F} - \langle \gamma ,\; \hat \gamma \rangle \asymp \norm{\gamma - \hat\gamma}_{\rm F}^2.    
\end{equation*}
\end{lemma}

\begin{proof}
Let $\calE = \hat \gamma - \gamma$, when $\norm{\gamma - \hat \gamma}_{\rm F} = o(\|\gamma\|_F)$ and $\norm{\gamma}_{\rm F} \ge c,$ we have
\begin{align*}
\norm{\gamma}_{\rm F} \cdot \norm{\hat \gamma}_{\rm F} - \langle \gamma, \; \hat \gamma \rangle &= \norm{\gamma}_{\rm F} \cdot \norm{\gamma + \calE}_{\rm F} - \langle \gamma, \; \gamma+\calE \rangle \\
&= \norm{\gamma}_{\rm F} \sqrt{\norm{\gamma}_{\rm F}^2 + 2\langle \gamma, \; \calE \rangle + \norm{\calE}_{\rm F}^2} - \norm{\gamma}_{\rm F}^2 - \langle \gamma, \; \calE \rangle \\
&= \norm{\gamma}_{\rm F}^2 \sqrt{1 + \frac{2\langle \gamma, \; \calE \rangle + \norm{\calE}_{\rm F}^2}{\norm{\gamma}_{\rm F}^2}} - \norm{\gamma}_{\rm F}^2 - \langle \gamma, \; \calE \rangle \\
&\asymp \norm{\gamma}_{\rm F}^2 \big( 1+\frac{1}{2} \frac{2\langle \gamma, \; \calE \rangle + \norm{\calE}_{\rm F}^2}{\norm{\gamma}_{\rm F}^2} \big) - \norm{\gamma}_{\rm F}^2 - \langle \gamma, \; \calE \rangle \\
&= \frac{\norm{\calE}_{\rm F}^2}{2} \asymp \norm{\hat \gamma - \gamma}_{\rm F}^2.
\end{align*}
\end{proof}

The following lemma illustrates the relationship between the risk function $\cR_{\btheta}(\hat\delta) -\cR_{\rm opt}(\btheta)$ and a more “standard” risk function $L_{\theta}(\hat \delta)$, which fulfills a role similar to that of the triangle inequality, as demonstrated in Lemma \ref{lemma:probability inequality}. Lemmas \ref{lemma:the first reduction} and \ref{lemma:probability inequality} correspond to Lemmas 3 and 4, respectively, in \cite{cai2019high}.

\begin{lemma} \label{lemma:the first reduction}
Let $\cZ \sim \frac{1}{2}\cT\cN(\cM_1; \; \bSigma) + \frac{1}{2}\cT\cN(\cM_2; \; \bSigma)$ with parameter $\theta = (\cM_1, \; \cM_2, \; \bSigma)$ where $\bSigma= [\Sigma_m]_{m=1}^M$. If a classifier $\hat \delta$ satisfies $L_{\theta}(\hat \delta) = o(1)$ as $n \rightarrow \infty$, then for sufficiently large n, 
\begin{align*}
\cR_{\btheta}(\hat\delta) -\cR_{\rm opt}(\btheta) \ge \frac{\sqrt{2\pi}\Delta}{8} e^{\Delta^2/8} \cdot L_{\theta}^2(\hat \delta) .    
\end{align*}
\end{lemma}

\begin{lemma} \label{lemma:probability inequality}
Let $\theta = (\cM, \; -\cM, \; [I_{d_m}]_{m=1}^M)$ and $\Tilde{\theta} = (\Tilde{\cM}, \; -\Tilde{\cM}, \; [I_{d_m}]_{m=1}^M)$ with $\norm{\cM}_{\rm F} = \norm{\Tilde{\cM}}_{\rm F} = \Delta/2$. For any classifier $\delta$, if $\norm{\cM - \Tilde{\cM}}_{\rm F} = o(1)$ as $n \rightarrow \infty$, then for sufficiently large n,
\begin{align*}
L_{\theta}(\delta) + L_{\Tilde{\theta}}(\delta) \ge \frac{1}{\Delta} e^{-\Delta^2/8} \cdot \norm{\cM - \Tilde{\cM}}_{\rm F} .    
\end{align*}
\end{lemma}


Although $L_{\theta}(\hat \delta)$ is not a distance function and does not satisfy an exact triangle inequality, the following lemma provides a variant of Fano's lemma. It suggests that it suffices to provide a lower bound for $L_{\theta}(\hat \delta)$, and $L_{\theta}(\hat \delta)$ satisfies an approximate triangle inequality (Lemma \ref{lemma:probability inequality}). 

\begin{lemma}[\cite{tsybakov2009}] \label{lemma:Tsybakov variant}
Let $N \ge 2$ and $\theta_0, \theta_1, \ldots ,\theta_N \in \Theta_d$. For some constants $\alpha_0 \in (0, 1/8), \gamma > 0$ and any classifier $\hat\delta$, if ${\rm KL}(\PP_{\theta_i}, \PP_{\theta_j}) \le \alpha_0 \log N/n$ for all $1 \le i \le N$, and $L_{\theta_i}(\hat\delta) < \gamma$ implies $L_{\theta_j}(\hat\delta) \ge \gamma$ for all $0 \le i \neq j \le N$, then 
\begin{align*}
\inf_{\hat\delta} \sup_{i \in [N]} \PP_{\theta_i}(L_{\theta_i}(\hat\delta)) \ge \gamma) \ge \frac{\sqrt{N}}{1+\sqrt{N}} (1-2\alpha_0-\sqrt{\frac{2\alpha_0}{\log N}}) >0  .    
\end{align*}
\end{lemma}

\begin{lemma}[Varshamov-Gilbert Bound, \cite{tsybakov2009}] \label{lemma:Varshamov-Gilbert Bound}
Consider the set of all binary sequences of length m: $\Omega = \big\{ \omega= (\omega_1,\ldots,\omega_m), \omega_i \in \{ 0,1\} \big\} = \{0,1\}^m$. Let $m \ge 8$, then there exists a subset $\{ \omega^{(0)}, \omega^{(1)},$ $\ldots, \omega^{(N)} \}$ of $\Omega$ such that $\omega^{(0)}=(0,\ldots, 0)$, $\rho_H(\omega^{(i)}, \omega^{(j)}) \ge m/8, \forall 0 \le i < j \le N$, and $N \ge 2^{m/8}$, where $\rho_H$ denotes the Hamming distance.
\end{lemma}

\end{appendices}

\end{document}